\documentclass{article}
\usepackage{ijcai23}

\usepackage{times}
\usepackage{soul}
\usepackage{url}
\usepackage[hidelinks]{hyperref}
\usepackage[utf8]{inputenc}
\usepackage[small]{caption}
\usepackage{microtype}
\usepackage{graphicx}
\usepackage{subfigure}
\usepackage{booktabs} 
\usepackage{amsmath,amssymb,amsfonts}
\usepackage{amsthm}
\usepackage{booktabs}
\usepackage[ruled,vlined]{algorithm2e}
\usepackage{algpseudocode}
\usepackage{graphicx}
\usepackage{textcomp}
\usepackage{xcolor}
\usepackage{dsfont}
\usepackage{pifont}
\usepackage{multirow}
\usepackage{colortbl}
\usepackage[switch]{lineno}

\newcommand{\para}[1]{\noindent {\bf #1}}
\newcommand{\textt}[1]{\noindent {\tt #1}}

\newcommand{\ie}{\emph{i.e.}}
\newcommand{\eg}{\emph{e.g.}}

\newcommand{\framework}{UMP}
\newcommand{\algone}{D-SUM}
\newcommand{\algtwo}{GT-DSUM}
\newtheorem{theorem}{Theorem}
\newtheorem{assumption}{Assumption}
\newtheorem{lemma}{Lemma}
\newtheorem{definition}{Definition}
\newtheorem{proposition}{Proposition}
\newtheorem{remark}{Remark}

\title{A Unified Momentum-based Paradigm of Decentralized SGD for Non-Convex Models and Heterogeneous Data}

\author{
Haizhou Du$^1$
\and
Chendong Ni$^2$\and
\affiliations
$^1$Shanghai University of Electric Power\\
}

\begin{document}

\maketitle

\begin{abstract}
Emerging distributed applications recently boost the development of decentralized machine learning, especially in IoT and edge computing fields.
In real-world scenarios, the common problems of non-convexity and data heterogeneity result in inefficiency, performance degradation, and development stagnation.
The bulk of studies concentrates on one of the issues mentioned above without having a more general framework that has been proven optimal. 
To this end, we propose a unified paradigm called \framework{}, which comprises two algorithms \textt{\algone} and \textt{\algtwo} 
based on the momentum technique with decentralized stochastic gradient descent (SGD).
The former provides a convergence guarantee for general non-convex objectives, 
 while the latter is extended by introducing gradient tracking, which estimates the global optimization direction to mitigate data heterogeneity (\ie, distribution drift). 
We can cover most momentum-based variants based on the classical heavy ball or Nesterov's acceleration with different parameters in \framework{}. In theory, we rigorously provide the convergence analysis of these two approaches for non-convex 
objectives and conduct extensive experiments, demonstrating a significant improvement in model accuracy up to $57.6\%$ 
compared to other methods in practice.    
\end{abstract}

\section{Introduction}
\label{sec:introduction}

Distributed machine learning (DML) has emerged as an important paradigm in large-scale machine learning~\cite{wan2022shielding,zhang2022fedduap,qu2022generalized}. In terms of how to aggregate 
the model parameters/gradients among workers,  researchers classify the system architecture into two main classes: parameter server (PS) and decentralized. 
The former is generally considered
as the centralized paradigm 
where the central server acts as a coordinator for convenience,
while the latter allows communication in a peer-to-peer fashion over an underlying topology, which could guarantee the model consistency across all workers with better scalability. 

Meanwhile, multiple complementary studies~\cite{fang2018spider,yu2019linear,hsieh2020non} have focused on the issues of DML
mainly based on the following two key aspects.
$\bullet$ \textit{The property of non-convex objectives is quite complicated in deep learning, in particular in distributed scenarios~\cite{karimireddy2020scaffold,lian2017can}.} 
Although some standard theoretical results have been obtained for convex models~\cite{tao2022private,deng2021minibatch,TaoLWT21}, much less is applicable in non-convex settings since they may be lossy and cause serious obstacles (\eg, high computation complexity and poor generalization)~\cite{ghadimi2015global,mai2020convergence}. 
$\bullet$ \textit{It is well known that heterogeneity in the data is one of 
key challenges in distributed training, 
 resulting in a slow and unstable convergence as well as poor model generalization.} There still exists a gap between the disappointing empirical performance and the degree of data heterogeneity~\cite{shang2022federated,lin2021quasi,esfandiari2021cross}.
Unfortunately, there are currently no existing works attempting to improve real-world decentralized training from a comprehensive perspective by taking both non-convexity and data heterogeneity into account.
Thus, it is non-trivial to handle these challenges, which significantly hinder the development of 
real-life applications.
Motivated by the momentum's effects on optimal convergence complexity and empirical evaluation successes ~\cite{koloskova2019decentralized,yu2019linear,han2020riemannian,lin2021quasi},
we propose \textbf{UMP}, a \textbf{U}nified, \textbf{M}omentum-based \textbf{P}aradigm
in the decentralized learning without considering the communication overhead throughout the paper.
It consists of two algorithms named \textt{\algone} and \textt{\algtwo}. 
The former one \textt{\algone} explores the potential of momentum by maintaining and scaling the momentum buffer to sharpen the loss landscape significantly and overcomes the restrictions of non-convexity, leading to better model performance and faster convergence rate in the non-convex settings. Our latter algorithm \textt{\algtwo} also aims to mitigate the impact of data heterogeneity on the discrepancy of local model parameters by introducing the gradient tracking (GT) technique~\cite{di2016next}. The core insight is that the variance between workers is decreasing while the local gradient asymptotically aligns with the global optimization direction independent on the heterogeneity of the data. \textt{\algtwo} accelerates decentralized learning achieving better generalization performance under both non-convex and different degrees of non-IID.

This paper makes the following \textbf{main contributions}:

\begin{itemize}
    \item 
    We 
    propose a unified momentum-based paradigm \framework{} with two algorithms 
    for dealing with non-convex and the degree of non-IID simultaneously. 
    Moreover, a variety of algorithms with the momentum technique could be obtained by specifying the parameters of our base algorithms.
    \item 
    We design the first algorithm \textt{\algone}, which achieves good model performance, demonstrating its applicability in terms of efficacy and efficiency. We also provide its convergence result under the non-convex cases.
    \item Our second one \textt{\algtwo}, which is robust to the distribution drift problem by applying the GT technique, is being further developed. We rigorously prove its convergence bound in smooth, non-convex settings.
    \item 
    We additionally conduct extensive experiments to evaluate the performance of \framework{} on common models, datasets, and dynamic real-world settings. Experimental results demonstrate that \textt{\algone} and \textt{\algtwo} improve the model accuracy by up to $35.8\%$ and $57.6\%$ respectively under different non-IID degrees compared with the well-known decentralized baselines. \textt{\algtwo} performs better than \textt{\algone} on model generalization across training tasks suffering from data skewness. 

\end{itemize}



\section{The Unified Paradigm: \framework{}}
\label{sec:design}
In this section, we first begin with the notation and revisit two momentum approaches: the heavy ball (HB) method~\cite{polyak1964some} and Nesterov's momentum~\cite{nesterov1983method}.
Inspired by them, we generalize a unified momentum-based paradigm  
with two algorithms \textt{\algone} and \textt{\algtwo}, which could cover the above two classical methods and other momentum-based variants, aiming to address issues on non-convexity and data heterogeneity in real-world decentralized learning applications. Finally, we provide the convergence result that they could converge almost to a stationary point 
for general smooth, non-convex objectives.

\subsection{Notation and Preliminary}
\label{subsec:pre_not}
To better demonstrate the applicable effect in real-world complex scenarios, we consider a decentralized setting with a network topology where $n$ workers jointly deal with an optimization problem. Assume that for every worker $i$, it holds its own datasets drawn from $\mathcal{D}_i$ distribution, which corresponds to data heterogeneity. Let $f_i: \mathbb{R}^d \to \mathbb{R}$ be the training datasets loss function of worker $i$ and can be given in a stochastic form $\mathbb{E}_{\xi_i \sim \mathcal{D}_i} \left [ \nabla F_i(\mathbf{x}, \xi_i) \right ] = \nabla f_i(\mathbf{x})$, where $F_i(\mathbf{x}, \xi_i)$ is the per-data loss function related with the mini-batch sample $\xi_i \sim \mathcal{D}_i$. Then, we formulate the empirical risk minimization with sum-structure objectives:
\begin{equation}
\label{eq:obj_func}
f^{\ast} = \min_{\mathbf{x} \in \mathbb{R}^d } \left [ f(\mathbf{x} ) := \frac{1}{n} \sum_{i=1}^n \left [ f_i(\mathbf{x} ) 
= \mathbb{E}_{\xi_i \sim \mathcal{D}_i } F_i(\mathbf{x}, \xi_i ) \right ]   \right ].
\end{equation}

Among workers, there is an underlying topology graph $\mathbf{W} \in \mathbb{R}^{n \times n}$, which is convenient to encode the communication between arbitrary two workers, 
\ie, we let $w_{ij}=0$ if and only if worker $i$ and $j$ are not connected.
\begin{definition}[Consensus Matrix ~\cite{koloskova2021improved}]
\label{def:consensus_matrix}
A matrix with non-negative entries $\mathbf{W} \in \left [ 0, 1 \right ]^{n \times n}$ that is symmetric ($\mathbf{W} = \mathbf{W}^{\top}$), and doubly stochastic ($\mathbf{W1} = \mathbf{1}, \mathbf{1}^{\top}\mathbf{W} = \mathbf{1}$), where $\mathbf{1}$ denotes the all-one vector in $\mathbb{R}^{n}$.
\end{definition}
Throughout the paper, we use the notation $\mathbf{x}_i^{(t),\tau}$ to denote the sequence of model parameters on worker $i$ at the $\tau$-th local update in epoch $t$. For any vector $\mathbf{a}_i \in \mathbb{R}^d$, 
we denote its model averaging $\bar{\mathbf{a}} = \frac{1}{n}\sum_{i=1}^n \mathbf{a}_i$. Let $\left \| \cdot \right \|$, $\left \| \cdot \right \|_F$ denote the $l_2$ vector norm and Frobenius matrix norm, respectively.

For ease of presentation, we apply both vector and matrix notation whenever it is more convenient. We denote by a capital letter for the matrix form combining by $\mathbf{a}_i$ as follows,
\begin{equation}
\begin{matrix}
\mathbf{A} = \left [ \mathbf{a}_1, \cdots, \mathbf{a}_n   \right ] \in \mathbb{R}^{d \times n},  
& \bar{\mathbf{A}} = \left [ \bar{\mathbf{a}}, \cdots, \bar{\mathbf{a}}  \right ] = \mathbf{A} \frac{1}{n} \mathbf{11}^{\top}.  
\end{matrix}   
\end{equation}

The introduction of a \textit{momentum} term is one of the most common modifications, which is viewed as a critical component for training the state-of-the-art deep neural networks~\cite{qu2022generalized,lin2021quasi}. Corresponding to its empirical success, momentum attempts to enhance the convergence rate on non-convex objectives by setting the optimized searching direction as the combination of stochastic gradient and historical directions.

The HB method (\ie, also known as Polyak's momentum) is first proposed for the smooth and convex settings, written as
\begin{equation}
\label{eq:momentum_hb}
\left\{\begin{matrix}
\begin{aligned}
& \mathbf{u}_i^{(t+1)} = \beta \mathbf{u}_i^{(t)} + \mathbf{g}_i^{(t)}   \\
& \mathbf{x}_i^{(t+1)} = \mathbf{x}_i^{(t)} - \eta \mathbf{u}_i^{(t+1)},
\end{aligned}
\end{matrix}\right.
\end{equation}
where $\mathbf{u}_i^{(t)}$, $\mathbf{g}_i^{(t)}$ are denoted as the momentum buffer, and the stochastic gradient of worker $i$ at epoch $t$, respectively. $\eta$ presents the learning rate. The momentum variable $\beta$ adjusts the magnitude of updating direction provided by the past information estimation with the stochastic gradient, indicating the direction of the steepest descent. Equivalently, (\ref{eq:momentum_hb}) can be also updated below
\begin{equation}
\label{eq:momentum_hb_one}
\mathbf{x}_i^{(t+1)} = \mathbf{x}_i^{(t)} - \eta \mathbf{g}_i^{(t)} + \beta \left ( \mathbf{x}_i^{(t)} - \mathbf{x}_i^{(t-1)}   \right ),
\end{equation}
when $t \ge 1$. Holding the past gradient values, this style of update can have better stability to some extent and enables improvement compared with some vanilla SGD methods~\cite{cutkosky2020momentum}.

Another kind of technique called Nesterov's shows that choosing with suitable parameters, the extrapolation step can be accelerated from $\mathcal{O}\left ( \frac{1}{t}  \right )$ to $\mathcal{O}\left ( \frac{1}{t^2}  \right )$, which is the optimal rate for the smooth convex problems. Concretely, its update step is described as follows
\begin{equation}
\label{eq:momentum_nesterov}
\left\{\begin{matrix}
\begin{aligned}
& \mathbf{u}_i^{(t+1)} = \beta \mathbf{u}_i^{(t)} + \mathbf{g}_i^{(t)}  \\
& \mathbf{v}_i^{(t+1)} = \beta\mathbf{u}_i^{(t+1)} + \mathbf{g}_i^{(t)}   \\
& \mathbf{x}_i^{(t+1)} = \mathbf{x}_i^{(t)} - \eta \mathbf{v}_i^{(t+1)}.
\end{aligned}
\end{matrix}\right.
\end{equation}
The model parameters are updated by introducing the momentum vector $\mathbf{u}_i$ and extra auxiliary $\mathbf{v}_i$ sequences. Compared with (\ref{eq:momentum_hb}), through decaying the momentum buffer $\mathbf{u}_i^{(t)}$, it effectively improves the rate of convergence without causing oscillations. Similarly, the above steps can be written as 
\begin{equation}
\label{eq:momentum_nesterov_one}
\mathbf{x}_i^{(t+1)} = \mathbf{x}_i^{(t)} - \eta \mathbf{g}_i^{(t)} 
+ \beta \left ( \mathbf{x}_i^{(t)} - \eta \mathbf{g}_i^{(t)} - \mathbf{x}_i^{(t-1)} + \eta \mathbf{g}_i^{(t-1)} \right ).
\end{equation}
Based on (\ref{eq:momentum_hb_one}) and (\ref{eq:momentum_nesterov_one}), it is not difficult to observe that the former could evaluate
the gradient and add momentum simultaneously, while the latter applies momentum after evaluating gradients, which intuitively causes more computation cost. Meanwhile, leveraging the idea of HB momentum, Nesterov's acceleration  brings us closer to the minimum (\ie, $\mathbf{x}^{\ast}$) by introducing an additional gradient descent rule by adding the subtracted gradients $\eta (\mathbf{g}_i^{(t-1)} - \mathbf{g}_i^{(t)})$ for general convex cases. 
The above two basic momentum-based approaches are firstly investigated in convex settings, showing their advantage compared with the vanilla SGD. However, there is still a shortage of a comprehensive analysis of momentum-based SGD under non-convex conditions in common real-world scenarios.

\subsection{\textt{\algone{}}  Algorithm}
\label{subsec:sum}
\begin{algorithm}[t]
\DontPrintSemicolon
\SetKwInput{Input}{Input}
\SetAlgoLined
\LinesNumbered
\caption{\colorbox{blue!30}{vanilla SGD} and \colorbox{red!30}{\textt{\algone{}}}; colors indicate the two alternative variants.}
\label{alg:sum_with_vanilla}
\Input{$\forall i$, initialize $\mathbf{x}_i^{(0),0} = \mathbf{v}_i^{(0), 0} = \mathbf{x}_0$; constant parameters $\eta$, $\alpha$, and $\beta$; $\forall i, j$, consensus matrix $\mathbf{W}$ with entries $w_{ij}$; the number of epochs $T$ and local steps $K$.}
\For{$t \in \left \{ 0, \cdots, T-1 \right \}$ {\it \bf at worker} $i$ {\it \bf in parallel}}{
Set $\mathbf{x}_i^{(t), 0} = \mathbf{x}_i^{(t)}, \mathbf{v}_i^{(t), 0} = \mathbf{v}_i^{(t)}$. \;
\For{$\tau \in \left \{ 0, \cdots, K-1 \right \}$ }{
Sample $\xi_i^{(t), \tau}$ and compute $\mathbf{g}_i^{(t), \tau} = \nabla F_i(\mathbf{x}_i^{(t), \tau}, \xi_i^{(t), \tau})$. \;
\colorbox{blue!30}{$\mathbf{x}^{(t), \tau+1} = \mathbf{x}^{(t),\tau} - \eta \mathbf{g}_i^{(t),\tau}$.} \;
\colorbox{red!30}{Compute local model $\mathbf{x}_i^{(t),\tau}$ from (\ref{eq:sum}).} \;
}
Perform gossip averaging via  (\ref{eq:gossip_avg_model}). \;
\colorbox{red!30}{$\mathbf{v}_i^{(t+1)} = \sum_{j=1}^n w_{ij} \mathbf{v}_j^{(t), K }$.} \;
}
\end{algorithm}

In this section, we present \framework{} and its first algorithm \textt{\algone}, 
which is employed in decentralized training under non-convex cases. 

Under each epoch, workers first perform $K$ local updates using different optimizers (\ie, SGD, Adam ~\cite{kingma2014adam}, etc.) with or without momentum. In this paper, we mainly focus on the momentum-based SGD variants, which are demonstrated in (\ref{eq:momentum_hb}), and (\ref{eq:momentum_nesterov}) for example. From a comprehensive view, we apply
the key update of the stochastic unified momentum (SUM) is according to

\begin{equation}
\label{eq:sum}
\left\{\begin{matrix}
\begin{aligned}
& \mathbf{u}_{i}^{(t), \tau+1} = \mathbf{x}_{i}^{(t),\tau} - \eta \mathbf{g}_i^{(t), \tau}  \\
& \mathbf{v}_{i,}^{(t),\tau+1} = \mathbf{x}_{i}^{(t), \tau} - \alpha \eta  \mathbf{g}_i^{(t), \tau} \\
 & \mathbf{x}_{i}^{(t), \tau+1} = \mathbf{u}_{i}^{(t),\tau+1} + \beta \left (  \mathbf{v}_{i}^{(t),\tau+1} -  \mathbf{v}_{i}^{(t),\tau} \right ),
\end{aligned}  
\end{matrix}\right.
\end{equation}
where $\alpha \ge 0$, and $\beta \in \left [ 0,1 \right )$. $\mathbf{a}_i^{(t),\tau}$ ($\mathbf{a}_i$ could be the instance for $\mathbf{x}_i$, $\mathbf{u}_i$, $\mathbf{v}_i$, and $\mathbf{g}_i$) is denoted as the related variables for worker $i$ after $\tau$ local updates in epoch $t$. After $K$ local steps, worker $i$ communicates with its neighbors according to the communication pattern $\mathbf{W}$ for exchanging their local model parameters. We call this synchronization operation as gossip averaging which can be compactly written as 
\begin{equation}
\label{eq:gossip_avg_model}
\mathbf{x}_i^{(t+1)} = \sum_{j=1}^{n} w_{ij} \mathbf{x}_j^{(t), K}.
\end{equation}
To present the difference between vanilla SGD and stochastic unified momentum in (\ref{eq:sum}), we summarize the training procedure
in Algorithm~\ref{alg:sum_with_vanilla}. The specific algorithm instance is obtained by tuning the hyperparameters $\alpha$, $\beta$, $\eta$, and $K$. We cover the basic Heavy Ball method (\ref{eq:momentum_hb_one}) and Nesterov's momentum (\ref{eq:momentum_nesterov_one}) when setting $\alpha = 0$, and $\alpha = 1$, respectively. Besides, when $K=1$, it reduces to the standard mini-batch SGD with momentum acceleration. Specially, we update the auxiliary variable sequences $\left \{ \mathbf{v}_i \right \}$ for any worker $i$ by using the same gossip synchronization as in (\ref{eq:gossip_avg_model}) interpreted as a restart in the next training epoch to simplify theoretical analysis.

However, there is no theoretical or empirical analysis to demonstrate that the momentum gets rid of heterogeneity which degrades the distributed deep training due to the discrepancies between local activation statistics~\cite{hsieh2020non}. Not only taking non-convex functions into account, but we also incorporate a technique that is agnostic to data heterogeneity, gradient tracking into \textt{\algone} to alleviate the impact of heterogeneous data in decentralized training for better model generalization in the following.

\subsection{\textt{\algtwo{}} Algorithm}
\label{sec:gt_sum}

\begin{algorithm}[t]
\DontPrintSemicolon
\SetKwInput{Input}{Input}
\SetAlgoLined
\LinesNumbered
\caption{\textt{\algtwo{}}}
\label{alg:gt_sum_gt}
\Input{$\forall i$, initialize $\mathbf{x}_i^{(0),0} = \mathbf{v}_i^{(0),0} = \mathbf{x}_0$, $\mathbf{y}_i^{(0)} = \mathbf{g}_i^{(0),0} = \nabla F_i(\mathbf{x}_i^{(0),0}, \xi_i^{(0),0})$, and $\mathbf{d}_i^{(-1)} = \mathbf{0}_p$; constant parameters $\alpha \ge 0, \beta \in \left [0,1 \right ), \eta, \lambda \in \left [0,1 \right ]$; $\forall i, j$, consensus matrix $\mathbf{W}$ with entries $w_{ij}$; the number of epochs $T$, and local steps $K$.}
\For{$t \in \left \{ 0, \cdots, T-1 \right \}$ {\it \bf at worker} $i$ {\it \bf in parallel}}{
\For{$\tau \in \left \{ 0, \cdots, K-1 \right \}$}{
Sample $\xi_i^{(t),\tau}$, compute $\mathbf{g}_i^{(t),\tau} = \nabla F_i(\mathbf{x}_i^{(t),\tau}, \xi_i^{(t),\tau})$. \;
$\mathbf{m}_i^{(t),\tau} = \lambda \mathbf{g}_i^{(t),\tau} + (1-\lambda) \mathbf{y}_i^{(t)}$. \;
Substitute $\mathbf{g}_i^{(t),\tau}$ with $\mathbf{m}_i^{(t),\tau}$ as the local gradient estimation, perform (\ref{eq:sum}). \;
}
Gossip averaging $\mathbf{x}_i^{(t+1)} = \sum_{j=1}^n w_{ij} \mathbf{x}_j^{(t),K}$. \;
$\mathbf{v}_i^{(t+1)} = \sum_{j=1}^n w_{ij} \mathbf{v}_i^{(t),K}$. \;
$\mathbf{d}_i^{(t)} = \frac{\mathbf{x}_i^{(t)} - \mathbf{x}_i^{(t+1)}}{K \eta}$. \;
Gradient tracking based on $\mathbf{y}_i^{(t+1)} = \sum_{j=1}^n w_{ij} \left ( \mathbf{y}_j^{(t)} +  \mathbf{d}_j^{(t)} - \mathbf{d}_j^{(t-1)} \right )$. \;
}
\end{algorithm}

In this subsection, we go further
the fact that heterogeneity hinders the local momentum acceleration ~\cite{lin2021quasi} and provides our second algorithm in \framework{}, termed \textt{\algtwo}, which aims to generalize the consensus model parameters better and alleviate the impact of heterogeneous data by applying the gradient tracking  technique.

Taking the discrepancies between workers' local data partition into account, GT introduces an extra worker-sided auxiliary variable $\mathbf{y}_i^{(t)}, \forall i$ aiming to asymptotically track the average of $\nabla f_i$ assuming the local accurate gradients are accessible at any epoch $t$. Intuitively, GT is agnostic to the heterogeneity, while $\mathbf{y}_i^{(t)}$ is approximately equivalent to the global gradient direction along with the epoch $t$ increases. Inspired by this, we
introduce GT into \textt{\algone}, yields \textt{\algtwo}. Concretely, we normalize the applied gradient $\mathbf{m}_i^{(t) ,\tau}$ using the mini-batch gradient $\mathbf{g}_i^{(t),\tau}$, and the $\mathbf{y}_i^{(t)}$ with the dampening factor $\lambda$ to highlight the necessity of local updates. The detailed algorithm is described in Algorithm~\ref{alg:gt_sum_gt}. Within local updates, the model parameters are updated on line 5 with \textt{\algone} but using a normalization term $\mathbf{m}_i^{(t),\tau}$. Line 7 and 8 are the same as the basic \textt{\algone} procedures in Algorithm~\ref{alg:sum_with_vanilla}. For \textt{\algtwo}, we apply the difference of two consecutive synchronized models shown in line 9 to update the gradient tracker variable in line 10 using the gossip-liked style~\cite{xin2021improved,xin2021hybrid}. 
Especially, when $K = 1$, $\lambda = 1$ and $\beta = 0$, the Algorithm~\ref{alg:gt_sum_gt} can be reduced to the original GT algorithm ~\cite{koloskova2021improved} instance.

Since Algorithm~\ref{alg:sum_with_vanilla} and \ref{alg:gt_sum_gt} employ multiple consensus steps from parameters exchanging which significantly increase communication cost, we 
apply the communication compression technique GRACE~\cite{xu2021grace} to trade off between model generalization and communication overhead in Section~\ref{sec:eval}.

\subsection{Theoretical Analysis}
\label{secsub:theoretical}
In what follows, we present the convergence analysis of two algorithms in the \framework{} for general non-convex settings. The detailed proof is in Appendix. Firstly, we state our assumptions throughout the paper.

\begin{assumption}[$L$-smooth]
\label{ass:l_smooth}
For each function $f_i: \mathbb{R}^d \to \mathbb{R}$ is differentiable, and there exists a constant $L > 0$ such that for each $\mathbf{x}, \mathbf{x}^{\prime} \in \mathbb{R}^d: \left \| \nabla f(\mathbf{x} ) - \nabla f(\mathbf{x}^{\prime} ) \right \| \le L\left \| \mathbf{x} - \mathbf{x}^{\prime}   \right \|$.
\end{assumption}

\begin{assumption}[Bounded variances]
\label{ass:bound_variance}
We assume that there exists $\sigma>0$ and $\zeta > 0$ for any $i, \mathbf{x} \in \mathbb{R}^d$ such that $\mathbb{E}_{\xi_i \sim \mathcal{D}_i } \left \| \nabla F_i(\mathbf{x}, \xi_i) - \nabla f_i(\mathbf{x} )  \right \|^2 \le \sigma^2 $, and $\frac{1}{n} \sum_{i=1}^n\left \| \nabla f_i(\mathbf{x}) - \nabla f(\mathbf{x} ) \right \|^2 \le \zeta^2 $.
\end{assumption}

Assumptions~\ref{ass:l_smooth} and \ref{ass:bound_variance} are standard in general non-convex objective literature~\cite{lin2021quasi,yu2019linear,koloskova2020unified} in order to ensure the basis of loss functions continuous and the limited influence of heterogeneity among distributed scenarios. Noted that when $\zeta = 0$, we have $\nabla f_i(\mathbf{x}) = \nabla f(\mathbf{x})$, \ie, it reduces to the case of IID data distribution across all participating workers. The third common assumption is to assume the stochastic gradients are uniformly bounded which is stated as follows.
\begin{assumption}[Bounded stochastic gradient]
We assume that the second moment of stochastic gradients is bounded for any $i, \mathbf{x} \in \mathbb{R}^d, \mathbb{E}_{\xi_i \sim \mathcal{D}_i} \left \| \nabla F_i(\mathbf{x}, \xi_i) \right \|^2 \le G^2$.
\label{ass:bound_sto_grad}
\end{assumption}

\begin{assumption}
\label{ass:mixing_matrix}
The mixing matrix is doubly stochastic by Definition~\ref{def:consensus_matrix}. Further, define $\bar{\mathbf{Z}} = \mathbf{Z}\frac{1}{n}\mathbf{11}^{\top}$ for any matrix $\mathbf{Z} \in \mathbb{R}^{d \times n}$. Then, the mixing matrix satisfies $\mathbb{E}_{\mathbf{W}} \left \| \mathbf{ZW} - \bar{\mathbf{Z}} \right \|^2_F \le (1-\rho) \left \| \mathbf{Z} - \bar{\mathbf{Z}} \right \|^2_F$.
\end{assumption}

In Assumption~\ref{ass:mixing_matrix}, we assume that $\rho := 1- \max \left \{ \left | \lambda_2(\mathbf{W} )  \right |,\left | \lambda_n(\mathbf{W} ) \right |  \right \}^2 > 0$, where let $\lambda_i(\mathbf{W})$ denote the $i$-th largest eigenvalue of the mixing matrix $\mathbf{W}$ with $-1 \le \lambda_n(\mathbf{W} ) \le \cdots \le \lambda_2(\mathbf{W} ) \le \lambda_1(\mathbf{W} ) \le 1$. For example, the value of $\rho$ is commonly used when $\rho = 1$ for the full-mesh (complete) communication topology.

\subsubsection{Convergence Analysis of \textt{\algone{}}}
\label{subsubsec:convergence_algone}
We now state our convergence result for \textt{\algone} (red highlight) in Algorithm~\ref{alg:sum_with_vanilla}. The detailed proof is presented in Appendix~\ref{appendix:proof_alg1}

\begin{theorem}
\label{thm:dsgd_sum}
Considering problem (\ref{eq:obj_func}) under the above mentioned assumptions, we denote $\beta_0 = \max \left \{ 1+\beta, 1+\alpha\beta \right \}$, for all $T \ge 1$ and $K \ge 1$ in Algorithm~\ref{alg:sum_with_vanilla} with learning rate $\eta \le \frac{\rho}{5L}$ and parameters satisfy $\frac{4-\rho}{2} < \frac{1}{\beta_0^2}$, we have
\begin{equation}
\begin{aligned}
& \frac{1}{KT} \sum_{t=0}^{T-1}\sum_{\tau=0}^{K-1} 
\mathbb{E} \left \| \nabla f(\bar{\mathbf{x}}^{(t), \tau}) \right \|^2 \\
& \quad \le \frac{2 \left ( f(\mathbf{x}_0) - f^{\ast} \right )  }{\tilde{\eta}KT}
+ \frac{2\beta^2\hat{\eta}^2L^2G^2}{n
(1-\beta)^4}
+ \frac{2L^2C_1}{n^2(1-Q_1)} \\
& \quad \quad  + \frac{L}{n} \left (\sigma^2 + 2\tilde{\eta}G^2
+ 3\sigma^2\tilde{\eta} \right ), \notag 
\end{aligned}
\end{equation}
where $\tilde{\eta} = \frac{\eta}{1-\beta}$, $\hat{\eta} = \left ( (1-\beta)\alpha - 1\right ) \eta$, $Q_1 = 2\beta_0^2 (1-\frac{\rho}{4})$, and $C_1 = 24\eta^2\beta_0^2\zeta^2/\rho + 4\left ( 1 - \rho \right ) \left ( 1 + 2\alpha\beta + 2\alpha^2\beta^2  \right )  \eta^2\sigma^2$. 
\end{theorem}

\begin{remark}
Theorem~\ref{thm:dsgd_sum} proposes a non-asymptotic convergence bound of \textt{\algone} for general neural network since the second term \ie, $\mathcal{O} \left ( \frac{L^2\hat{\eta}^2}{n}  \right )$ generates from the core SGD step in (\ref{eq:sum}). Intuitively, there exists an appropriate $\alpha$ for achieving the optimal training performance in practice, which has been observed in the single node case~\cite{yan2018unified}. In Section~\ref{sec:eval}, we perform related experiments to confirm this speculation.
\end{remark}

\subsubsection{Convergence Analysis of \textt{\algtwo{}}}
Next theorem is the convergence result of \textt{\algtwo{}} in Algorithm~\ref{alg:gt_sum_gt} when $K = 1$ with a fixed communication topology among workers for convenience, and the detailed proof is in Appendix~\ref{appendix:proof_alg2}. Based on the GT is addressed with the issue on how to apply the mini-batch gradient estimates to track the global optimization descent direction, we define the following proposition to clarify this illustration.

\begin{proposition}[Gradients averaging tracker~\cite{di2016next}]
\label{prop:appro_grad_tracking}
We assume a loose constraint that the auxiliary variables $\mathbf{y}_i^{(t)}$ are considered as the tracker of the average $\frac{1}{n} \sum_{j=1}^n \nabla f_j(\mathbf{x}_i^{(t)})$, which means for any epoch $t$, we have $\mathbb{E} \left \| \mathbf{y}_i^{(t)} - \frac{1}{n}\sum_{j=1}^n \nabla f_j(\mathbf{x}_i^{(t)})  \right \|^2 \le \epsilon ^2$.
\end{proposition}

\begin{theorem}
\label{thm:gt_dsum}
Consider problem (\ref{eq:obj_func}) under the listed specific
assumptions, we denote $\beta_0 = \max \left \{ 1+\alpha\beta, 1+\beta \right \}$, and
set $T \ge 1$ in Algorithm~\ref{alg:gt_sum_gt} without multiple local steps (\ie, $K=1$) with learning rate $\eta$ chosen as 
\begin{equation}
0 \le \eta \le \min \left \{ \frac{\rho}{12\lambda}, \frac{1-\beta}{2}, \frac{3+\beta}{2\sqrt{3}\lambda (1+\alpha\beta)}, \frac{1+\beta}{2\sqrt{3}\lambda\alpha\beta}   \right \}\frac{1}{L}  \notag
\end{equation}
and parameters satisfy
\begin{equation}
\left\{\begin{matrix}
\begin{aligned}
 & (1+\alpha\beta) (1-\lambda) \le \frac{1}{2\sqrt{2}},\\
 & \rho \le \frac{48 \lambda^2L^2}{(1-\lambda)^2},\\
 & 4\beta_0^2\left ( 1-\frac{\rho}{4}  \right ) < 1,  \\
 & 8(1-\rho) (1+\alpha\beta)^2 < 1 ,
\end{aligned}
\end{matrix}\right.\notag 
\end{equation}
we have 
\begin{equation}
\begin{aligned}
& \frac{1}{T} \sum_{t=0}^{T-1} \mathbb{E} \left \| \nabla f(\bar{\mathbf{x}}^{(t)}) \right \|^2 \\
& \le \frac{2\left ( f(\mathbf{x}_0) - f^{\ast} \right ) }{\tilde{\eta}T} 
+ \frac{8\beta^2\hat{\eta}^2L^2}{n(1-\beta)^4} \left ( \sigma^2 + 3\zeta^2 + 4G^2 \right ) \\ 
& \quad + \frac{12L^2}{n} \left ( \frac{Q_3}{T} + \frac{(2-Q_2)Q_4}{(1-Q_2)T} + \frac{V_{\max}}{1-Q_2} +\frac{C_2}{1-Q_2}  \right ) \\
& \quad + \frac{4\lambda^2\sigma^2+16\zeta^2}{n} + 8(1-\lambda^2) \epsilon^2 + \frac{12\beta^2\hat{\eta}^2L^2\epsilon^2}{(1-\beta)^4} \notag
\end{aligned}
\end{equation}
where $\tilde{\eta} = \frac{\eta}{1-\beta}$ and $\hat{\eta} = (\alpha - \alpha\beta -1)\eta$. In addition, $C_2 = \frac{\beta_0^2\left ( 1-\frac{\rho}{2}  \right )}{L^2} (\sigma^2+\zeta^2) (192\lambda^2L^2 + \rho )$, $Q_2 \triangleq \min \left \{ 4\beta_0^2\left ( 1-\frac{\rho}{4}  \right ), 8(1-\rho) (1+\alpha\beta)^2  \right \}$, $Q_3 = 6(\zeta^2+\sigma^2)$, $Q_4 = \beta_0^2 \left ( 1-\frac{\rho}{2}  \right ) (\sigma^2+\zeta^2) \left ( 48+ \frac{\rho}{L^2} + 192\lambda^2   \right )$. Furthermore, we define 
$V_{\max}
\triangleq \max_{0 \le t \le T-1} \left \{ \frac{1}{n} \left (\mathbb{E} \left \| \mathbf{X}^{(t)} - \bar{\mathbf{X}}^{(t)} \right \|_F^2
+  \mathbb{E} \left \| \mathbf{Y}^{(t)} - \bar{\mathbf{Y}}^{(t)} \right \|_F^2 \right )\right \} .$
\end{theorem}

\begin{remark}
\label{remark:gt}
The fourth term on the right-hand side of the Theorem~\ref{thm:gt_dsum}, \ie, $\mathcal{O} \left ( \frac{1}{nT} + \frac{\beta_0^2}{nT} + \frac{1}{n\beta_0^2} \right )$ comes from the the additional GT step for searching global optimal descent estimation in line $10$, Algorithm~\ref{alg:gt_sum_gt}. However, this term can be dominant when $\alpha$ scales due to its higher order. Clearly, $\beta_0 = 1 + \alpha\beta$ when $\alpha > 1$, and it performs the convergence rate $\mathcal{O} \left ( \frac{\beta_0^2}{nT} \right )$, leading 
to a significant deterioration from convergence perspective if $\alpha$ rises.
Hence, we will show the impact of $\alpha$ in Section~\ref{sec:eval}.
\end{remark}

\section{Evaluation}
\label{sec:eval}
Our main evaluation results demonstrate that \textt{\algone{}} outperforms other methods in terms of model accuracy, and \textt{\algtwo{}} achieves a higher performance under different levels of non-IID.
All experiments are executed in a CPU/GPU cluster, equipped with Inter(R) Xeon(R) Gold 6126, 4 GTX 2080Ti cards, and 12 Tesla T4 cards. We used Pytorch and Ray \cite{ray2018} to implement and train our models. 

\subsection{Experiment Methodology}
\label{subsec:exp_setting}
\para{Baselines.}
We consider the following three decentralized methods with momentum, which are described as follows:
$\bullet$ \textit{Local SGD}~\cite{stich2018local}  periodically averages model parameters among all worker nodes. Compared with the vanilla SGD, each node independently runs the single-node SGD with Heavy Ball momentum.
$\bullet$ \textit{QG-DSGDm}~\cite{lin2021quasi} mimics the global optimization direction and integrates the quasi-global momentum into local stochastic gradients without causing extra communication costs. It empirically mitigates the impact on data heterogeneity. 
$\bullet$ \textit{SlowMo}~\cite{slowmo} performs a slow, periodical momentum update through an All-Reduce pattern (model averaging) after multiple SGD steps. For simplicity, we use the common mini-batch gradient as the local update direction. 

\para{Datasets and models.}
We study the decentralized behaviors on both computer vision (CV) and natural language processing (NLP) tasks, including MNIST, EMNIST, CIFAR10, and AG NEWS. For all CV tasks, we train different CNN models. For NLP, we train an RNN, which includes an embedding layer, and a dropout layer, followed by a dense layer. The model description is shown in Appendix~\ref{sec:appendix:ex_setup}.

\begin{table*}[t]
\centering
\caption{
The testing accuracy with different algorithms on various training benchmarks and different degrees of non-IID.}
\label{tab:acc}
\scriptsize
{
\begin{tabular}{ccccc}
\toprule
\multirow{2}{*}{Datasets} & \multirow{2}{*}{Algorithms} & \multicolumn{3}{c}{Testing Accuracy ($\%$)} \\
\cmidrule(lr){3-5}
& & non-IID $=0.1$ & non-IID $=1$ & non-IID $=10$ \\
\midrule
\multirow{5}{*}{MNIST~\cite{lecun1998gradient}} & \texttt{Local SGD w/ momentum} & $95.66 \pm 0.21$ & $97.99 \pm 0.03$ & $98.39 \pm 0.03$ \\
& \texttt{QG-DSGDm} & $96.02 \pm 0.19$ & $97.46 \pm 1.36$ & $98.21 \pm 0.04$\\
& \texttt{SlowMo} & $97.32 \pm 0.02$ & $97.93 \pm 0.07$ & $98.34 \pm 0.06$ \\ \cmidrule(lr){2-5}
& \texttt{\algone} (ours) & $\textbf{97.89} \pm 0.21$ & $\textbf{98.77} \pm 0.04$ & $\textbf{98.94} \pm 0.01$ \\
& \texttt{\algtwo} (ours) & $97.51 \pm 0.61$ & $98.70 \pm 0.01$ & $98.82 \pm 0.03$ \\ \midrule
\multirow{5}{*}{EMNIST~\cite{cohen2017emnist}} & \texttt{Local SGD w/ momentum} & $45.90 \pm 1.21$ & $36.77 \pm 0.13$ & $38.29 \pm 0.03$ \\
& \texttt{QG-DSGDm} & $46.03 \pm 0.6$ & $46.02 \pm 0.12$ & $36.72 \pm 0.02$\\
& \texttt{SlowMo} & $45.52 \pm 0.03$ & $37.11 \pm 0.01$ & $37.50 \pm 0.0$ \\ \cmidrule(lr){2-5}
& \texttt{\algone} (ours) & $49.68 \pm 0.43$ & $49.75 \pm 0.05$ & $42.50 \pm 0.01$ \\
& \texttt{\algtwo} (ours) & $\textbf{50.49} \pm 0.82$ & $\textbf{50.25} \pm 0.07$ & $\textbf{51.87} \pm 0.02$ \\ \midrule
\multirow{5}{*}{CIFAR10~\cite{krizhevsky2009learning}} & \texttt{Local SGD w/ momentum} & $22.94 \pm 1.11$ & $42.93 \pm 0.85$ & $52.82 \pm 0.01$ \\
& \texttt{QG-DSGDm} & $26.34 \pm 1.42$ & $49.12 \pm 038$ & $54.03 \pm 0.24$\\
& \texttt{SlowMo} & $31.06 \pm 1.27$ & $50.46 \pm 0.04$ & $55.50 \pm 0.10$ \\ \cmidrule(lr){2-5}
& \texttt{DSUM} (ours) & $31.16 \pm 1.27$ & $54.34 \pm 0.11$ & $57.59 \pm 1.05$ \\
& \texttt{\algtwo} (ours) & $\textbf{36.16} \pm 0.74$ & $\textbf{56.95} \pm 1.56$ & $\textbf{59.34} \pm 1.55$ \\ \midrule
\multirow{5}{*}{AG NEWS~\cite{zhang2015character}} & \texttt{Local SGD w/ momentum} & $75.51 \pm 0.44$ & $77.98 \pm 0.39$ & $80.66 \pm 0.02$ \\
& \texttt{QG-DSGDm} & $78.82 \pm 0.31$ & $79.33 \pm 0.38$ & $82.24 \pm 0.02 $\\
& \texttt{SlowMo} & $82.57 \pm 0.03$ & $83.17 \pm 0.01$ & $83.79 \pm 0.01$ \\ \cmidrule(lr){2-5}
& \texttt{DSUM} (ours) & $84.13 \pm 0.55$ & $85.46 \pm 0.31$ & $87.52 \pm 0.04$ \\
& \texttt{\algtwo} (ours) & $\textbf{84.29} \pm 0.37$ & $\textbf{87.59} \pm 0.18$ & $\textbf{89.07} \pm 0.04$ \\ \bottomrule
\end{tabular}
}
\label{tbl:baseline_evalutaion}
\end{table*}

\para{Hyperparameters.}
For all algorithms with different benchmarks, the setting deploys $10$ workers by default. In our experiments, we set the local mini-batch size as $256$ for CIFAR10 and $128$ for the rest, and the number of local updates is set as $K=10$. To illustrate the challenge of data heterogeneity in decentralized deep training, we adopt the Dirichlet distribution value \cite{lin2021quasi} to control different levels of non-IID degree, for the case with non-IID $= 0.1, 1, 10$; the smaller the value is, the more likely the workers
hold samples from only one class of labels (\ie, non-IID $= 0.1$ can be viewed as an extreme data skewness case). Besides, we set the scalar $\alpha$, momentum $\beta$, normalized parameter $\lambda$ as $2$, $0.9$, and $0.8$ respectively by default. Among choices of $\mathbf{W}$ considered in practice, we pre-construct a dynamic topology changing sequence varying from full-mesh to ring by the popular Metropolis-Hastings rule~\cite{koloskova2021improved} \ie, $w_{ij}=w_{ji}=\min \left \{ \frac{1}{\text{deg}(i)+1 }, \frac{1}{\text{deg}(j)+1 }  \right \}$ for any $i, j$, $w_{ii}=1 - \sum_{j=1}^n w_{ij}$. The learning rate $\eta$ is fine-tuned via a grid search on the set $\left \{ 10^{-2}, 10^{-1.5}, 10^{-1}, 10^{-0.5} \right \}$ for each algorithm and dataset.

\para{Performance Metrics.}
We examine the effects of different momentum variants on decentralized deep learning, including
\begin{itemize}
    \item \textit{Model generalization} is measured by the proportion between the amount of the correct data by the model and that of all data in the test dataset. We report the averaged model performance of local models over test samples.
    \item \textit{Effect of different hyperparameters} is explored by tuning their values to study the properties of \textt{\algone} and \textt{\algtwo}. 
    \item \textit{Scalability} is a crucial property while handling tasks in a distributed situation.
\end{itemize}

\subsection{Evaluation results}
\label{subsec:eva_results}
\para{Performance with compared baselines.} 
In Table~\ref{tbl:baseline_evalutaion}, we can see that 
our proposed algorithms outperform all other baselines across different levels of data skewness.
For CIFAR10 and AG NEWS, the performance of our algorithms and benchmarks: \textt{\algtwo} $>$ \textt{\algone} $>$ \textt{SlowMo} $>$ \textt{QG-DSGDm} $>$ \textt{Local SGD w/ momentum}.
Our proposed algorithms outperform other benchmarks on model generalization and demonstrate that GT technique effectively mitigates the negative impact caused by data heterogeneity. 
As the non-IID level increases, 
\textt{GT-DSUM} achieves a
higher accuracy than \textt{Local SGD w/ momentum} 
up to $57.6\%$ on CIFAR10. 

\begin{table*}[!htbp]
\centering
\caption{The impact of $\alpha$ for \textt{\algone} and \textt{\algtwo} on the test accuracy with non-IID $= 1$. 
``$\star$'' indicates non-convergence.}
\label{tbl:impact_alpha}
\scriptsize
{
\begin{tabular}{cccccccccccc}
\toprule
\multirow{2}{*}{Datasets} & \multirow{2}{*}{Methods} &  \multicolumn{10}{c}{The test accuracy ($\%$) evaluated on different $\alpha$ under the non-IID $= 1$ case} \\ \cmidrule(lr){3-12}
& & $\alpha = 0$  & $\alpha = 0.5$ & $\alpha = 1$ & $\alpha = 2$ & $\alpha = 3$ & $\alpha = 4$ & $\alpha = 5$ & $\alpha = 8$ & $\alpha = 10$ & $\alpha = 15$ \\ \midrule
\multirow{2}{*}{MNIST}    & \textt{\algone}  & $98.05$ & $98.40$ & $98.57$  &  $98.76$ & $98.85$ & $98.76$ & $98.85 $ & $\textbf{99.09}$ & $98.87$  & $92.86$  \\
& \textt{\algtwo}                          & $98.18$ & $98.50$ & $98.70$  &  $98.70$ & $\textbf{98.80}$   &  $\star $ & $\star $ & $\star$ & $\star$  & $\star$  \\ \midrule
\multirow{2}{*}{EMNIST}   & \textt{\algone}  & $37.1$ & $43.58$ & $35.58$  &  $49.75$ & $\textbf{55.32}$   &  $49.80 $ & $43.00 $ & $48.47 $ & $53.04$  & $\star$  \\
& \textt{\algtwo}                          & $33.72$ & $46.06$ & $47.10$  &  $\textbf{50.25} $ & $39.84$   &  $\star $ & $\star $ & $\star $ & $\star$  & $\star $  \\ \midrule
\multirow{2}{*}{CIFAR10}  & \textt{\algone}  & $47.10 $ & $50.85$ & $51.68$  &  $50.32$ & $\textbf{54.83}$   &  $51.10 $ & $51.53  $ & $50.68$ & $\star$  & $\star$  \\
& \textt{\algtwo}                          & $45.98$ & $49.80$ & $50.55$  &  $54.77$ & $\textbf{57.58}$   &  $54.43 $ & $\star $ & $\star $ & $\star$  & $\star$  \\ \midrule
\multirow{2}{*}{AG NEWS}  & \textt{\algone}  & $79.16$ & $81.78 $ & $83.72$  &  $85.46$ & $86.38$   &  $86.36   $ & $88.20 $ & $87.07 $ & $\textbf{88.82}$  & $88.56$  \\
& \textt{\algtwo}                          & $78.90 $ & $83.12  $ & $85.34$  &  $\textbf{87.59} $ & $86.89$   &  $77.67  $ & $\star $ & $\star $ & $\star$  & $\star$  \\ \bottomrule    
\end{tabular}
}
\end{table*}

\begin{figure*}[t]
\centering
\subfigure[\small LeNet over MNIST]{
\includegraphics[width=0.23\textwidth]{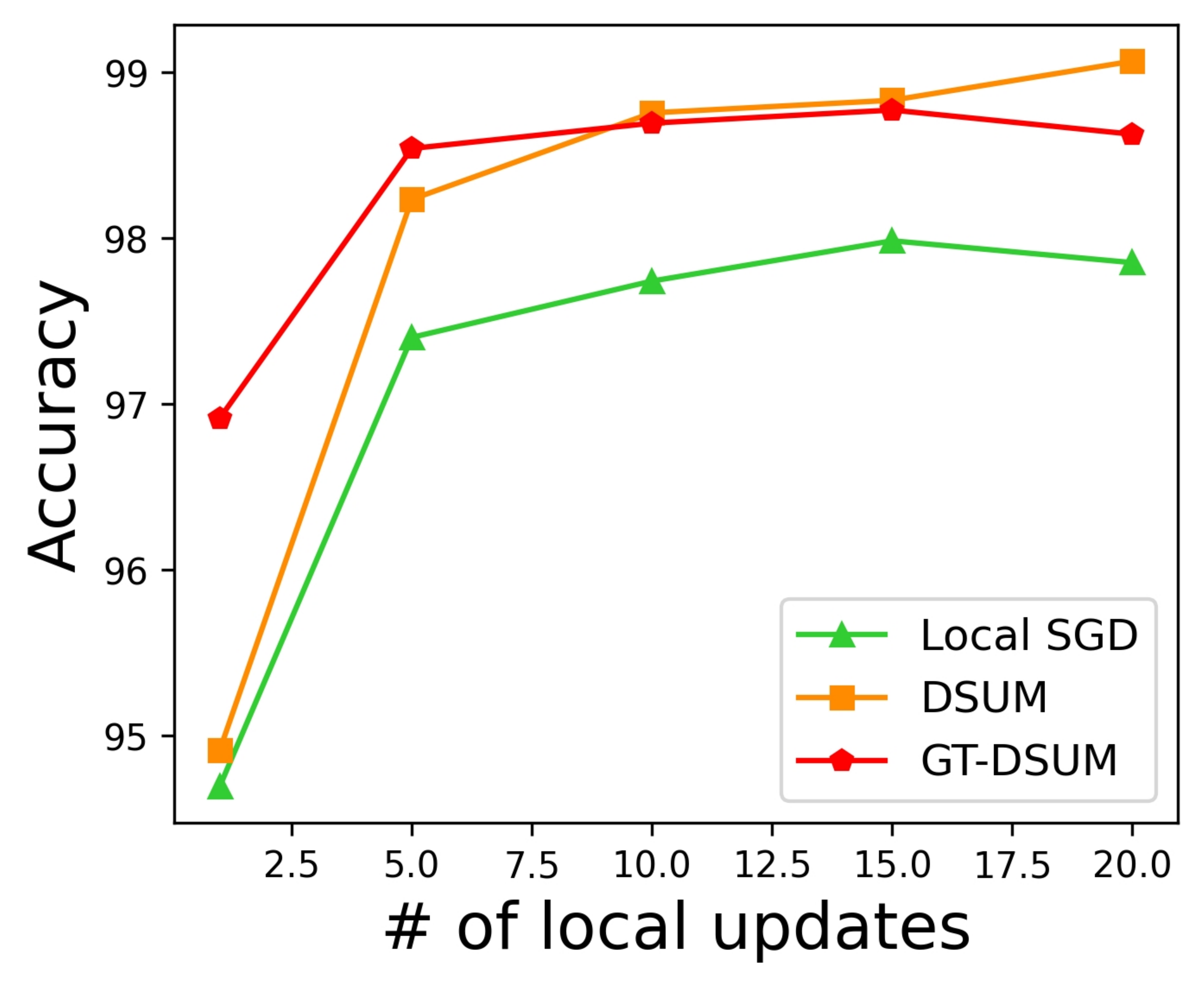}
\label{fig:impact_lu_mnist}}
\subfigure[\small CNN over EMNIST]{
\includegraphics[width=0.23\textwidth]{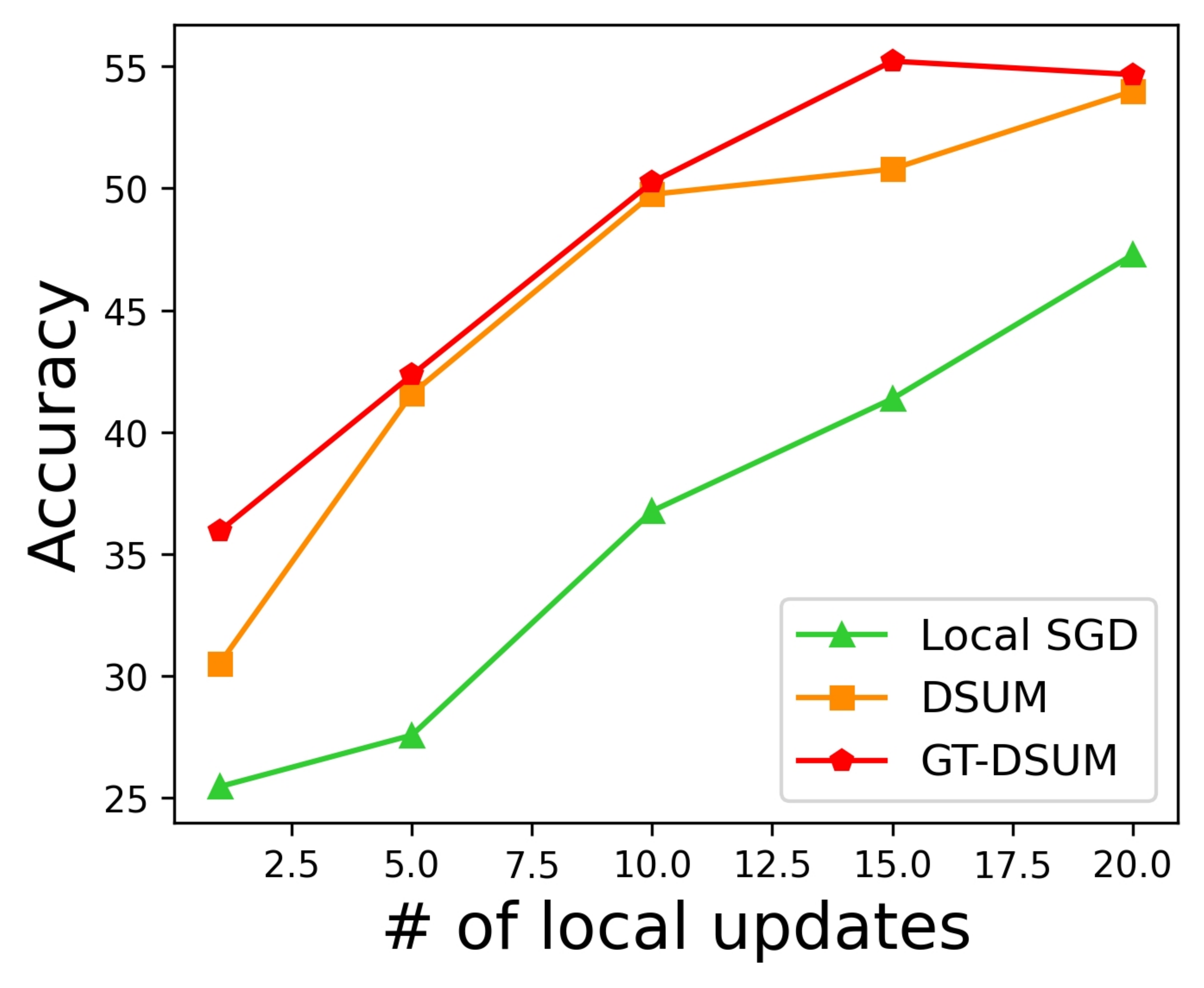}
\label{fig:impact_lu_emnist}}
\subfigure[\small LeNet over CIFAR10]{
\includegraphics[width=0.23\textwidth]{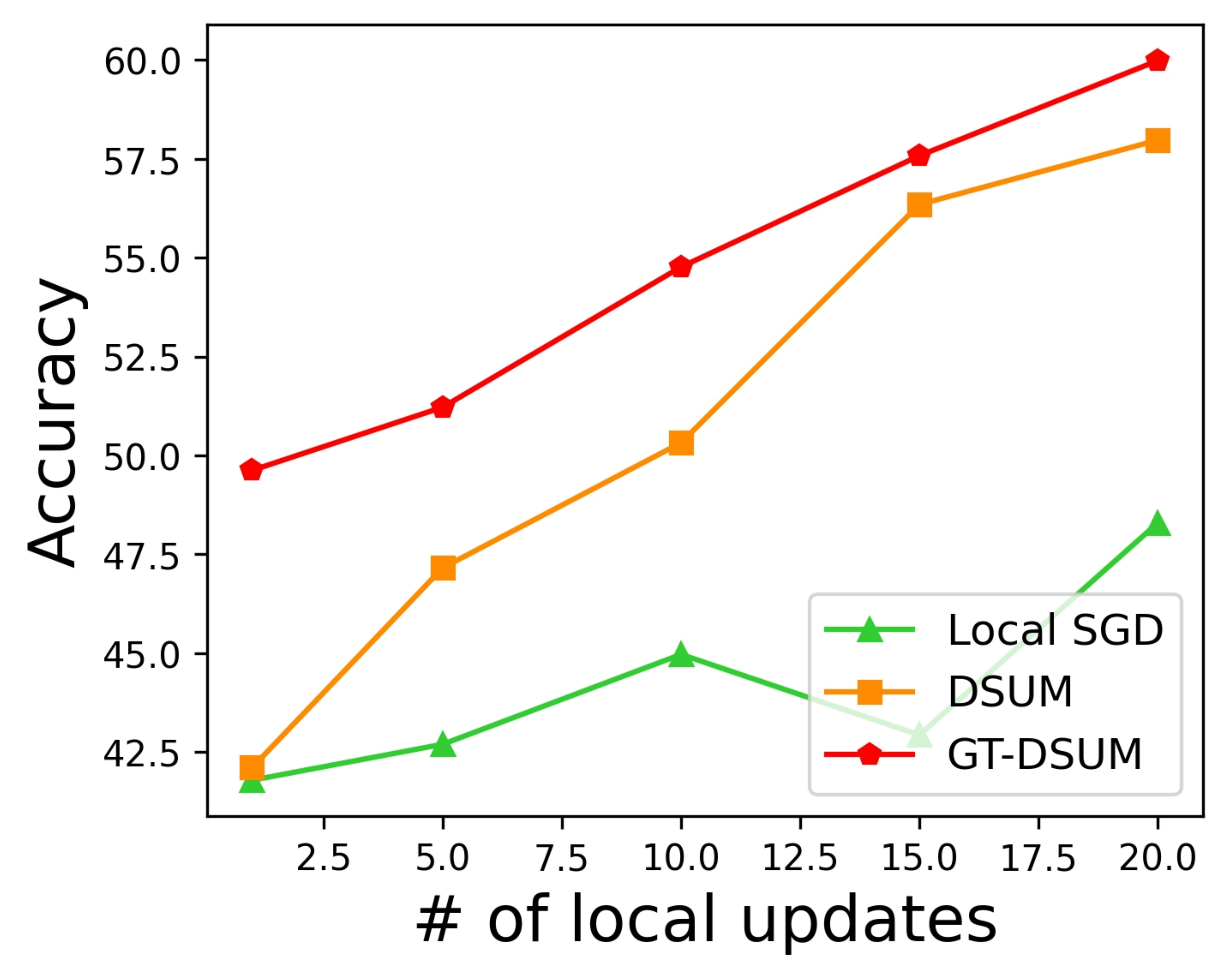}
\label{fig:impact_lu_cifar10}}
\subfigure[\small RNN over AG NEWS]{
\includegraphics[width=0.23\textwidth]{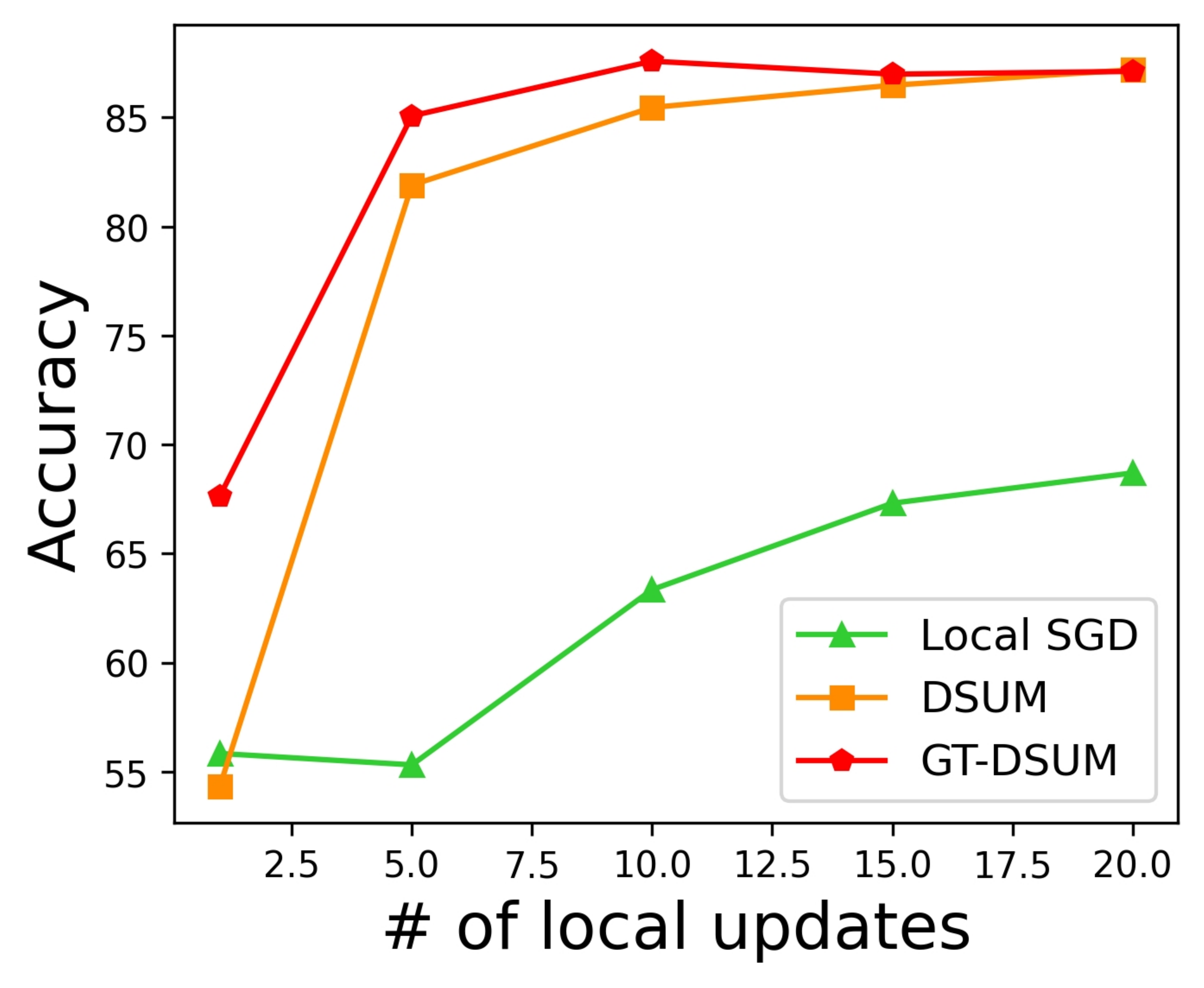}
\label{fig:impact_lu_agnews}}
\caption{Impact on the number of local updates $K$ on the convergence when momentum $\beta = 0.9$ under the non-IID $= 1$ case.}
\label{fig:impact_lu}
\end{figure*}

\begin{figure*}[!htbp]
\centering
\subfigure[\small LeNet over MNIST]{
\label{fig:impact_momentum_mnist}
\includegraphics[width=0.23\textwidth]{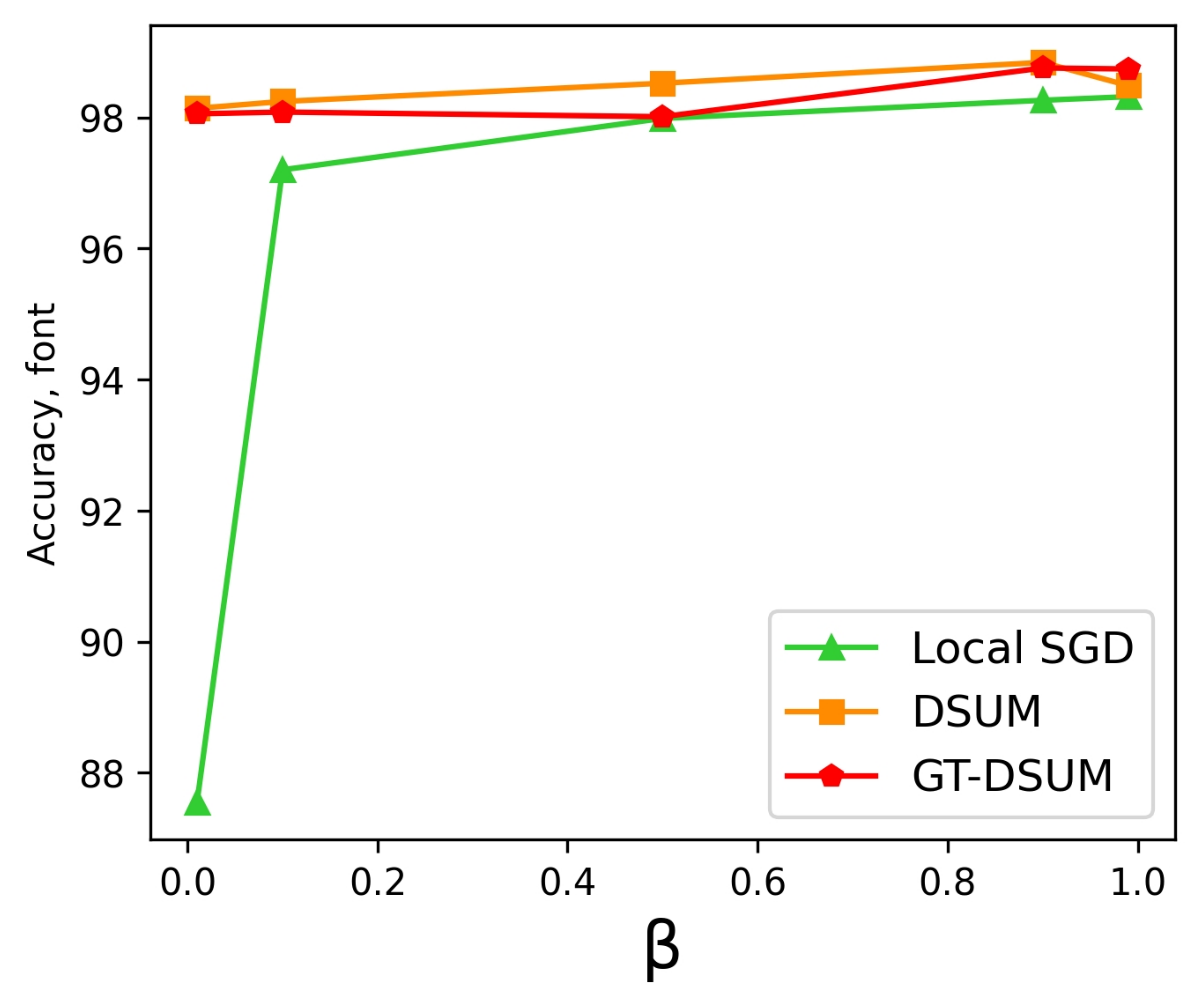}}
\subfigure[\small CNN over EMNIST]{
\label{fig:impact_momentum_emnist}
\includegraphics[width=0.23\textwidth]{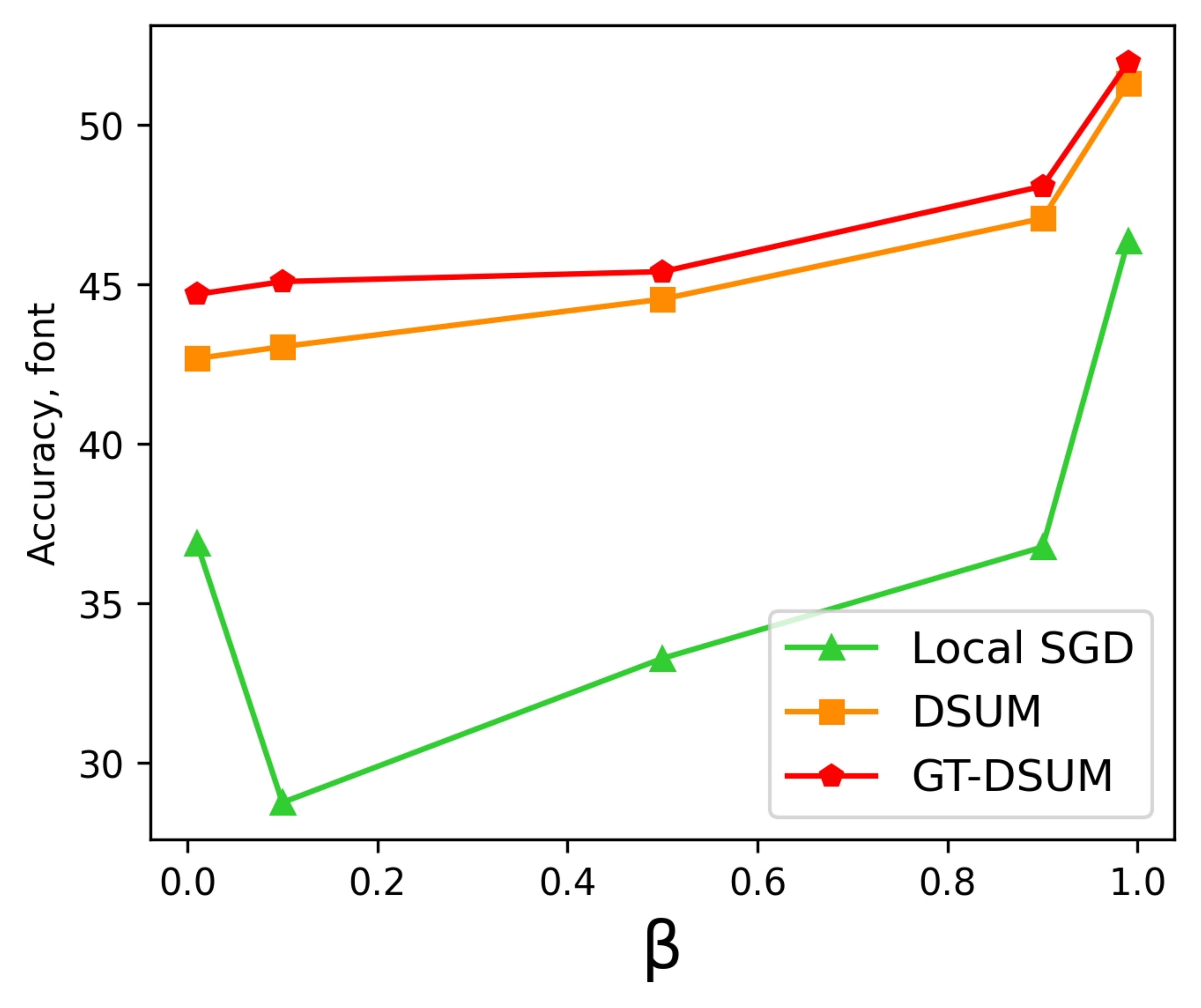}}
\subfigure[\small LeNet over CIFAR10]{
\label{fig:impact_momentum_cifar10}
\includegraphics[width=0.23\textwidth]{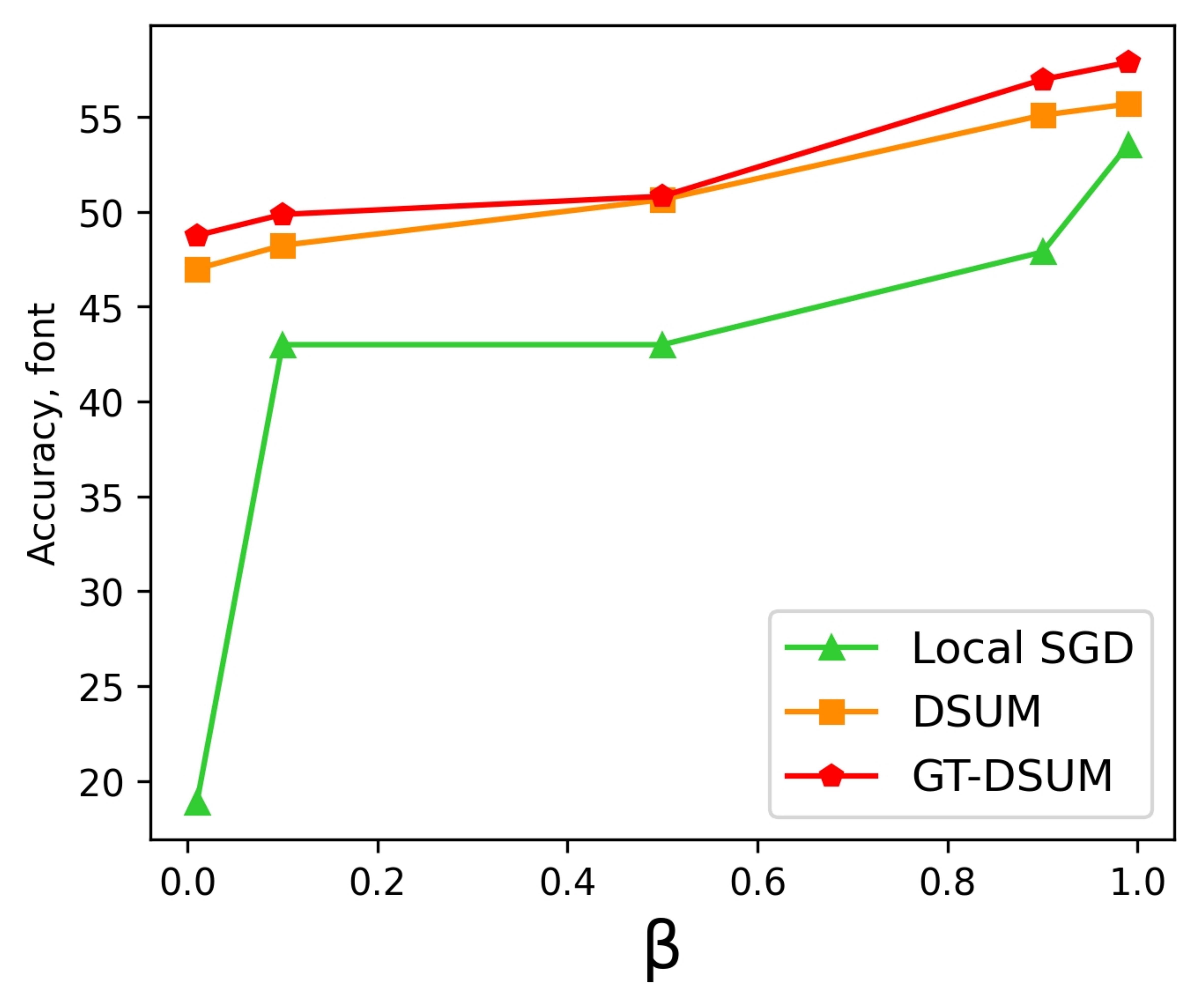}}
\subfigure[\small RNN over AG NEWS]{
\label{fig:impact_momentum_agnews}
\includegraphics[width=0.23\textwidth]{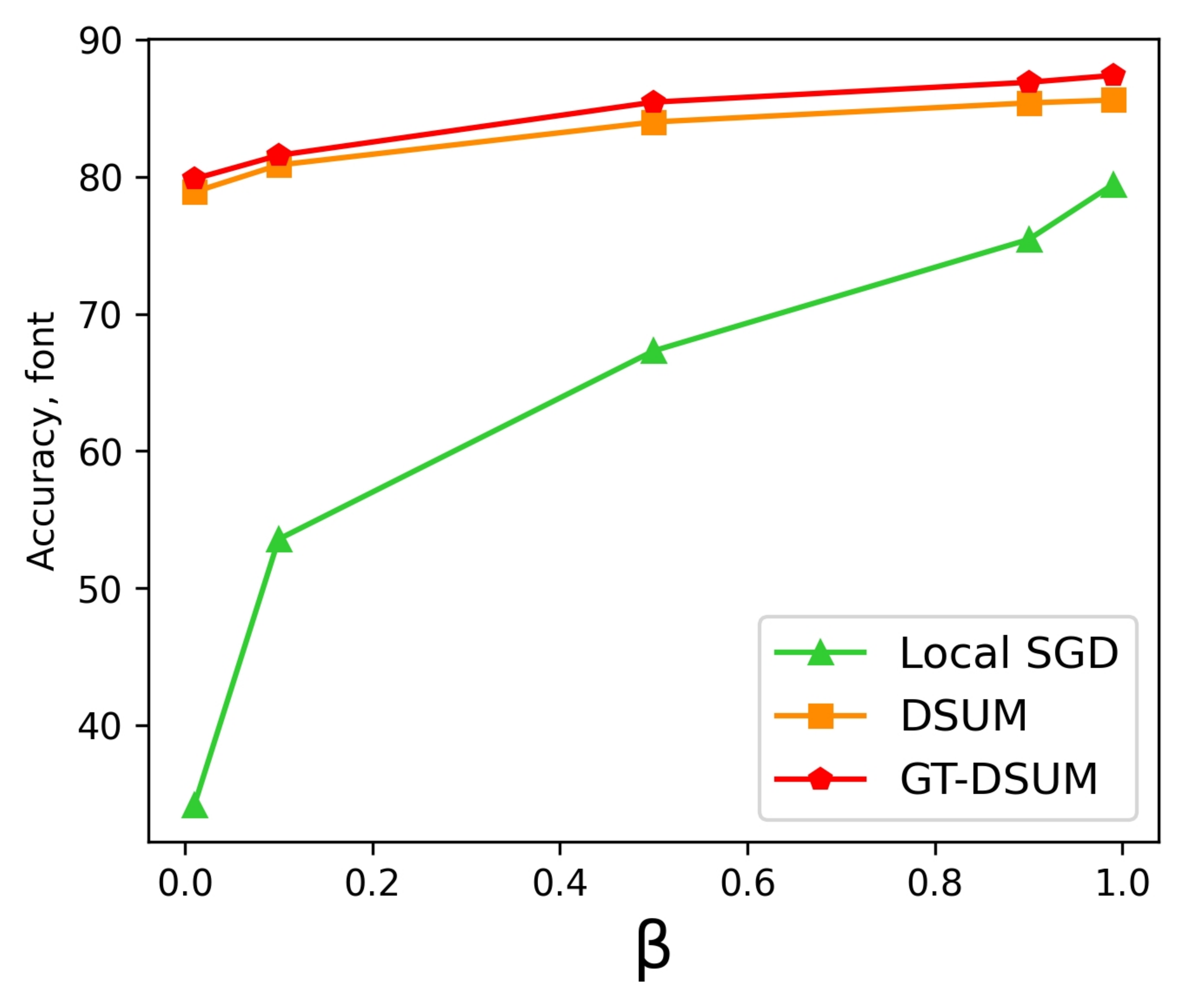}}
\caption{Impact on the momentum $\beta$ on the convergence when the number of local update $K = 10$ under the non-IID $= 1$ case.}
\label{fig:impact_momentum}
\end{figure*}

\begin{figure*}[!htbp]
\centering
\subfigure[LeNet over MNIST]{
\includegraphics[width=0.23\textwidth]{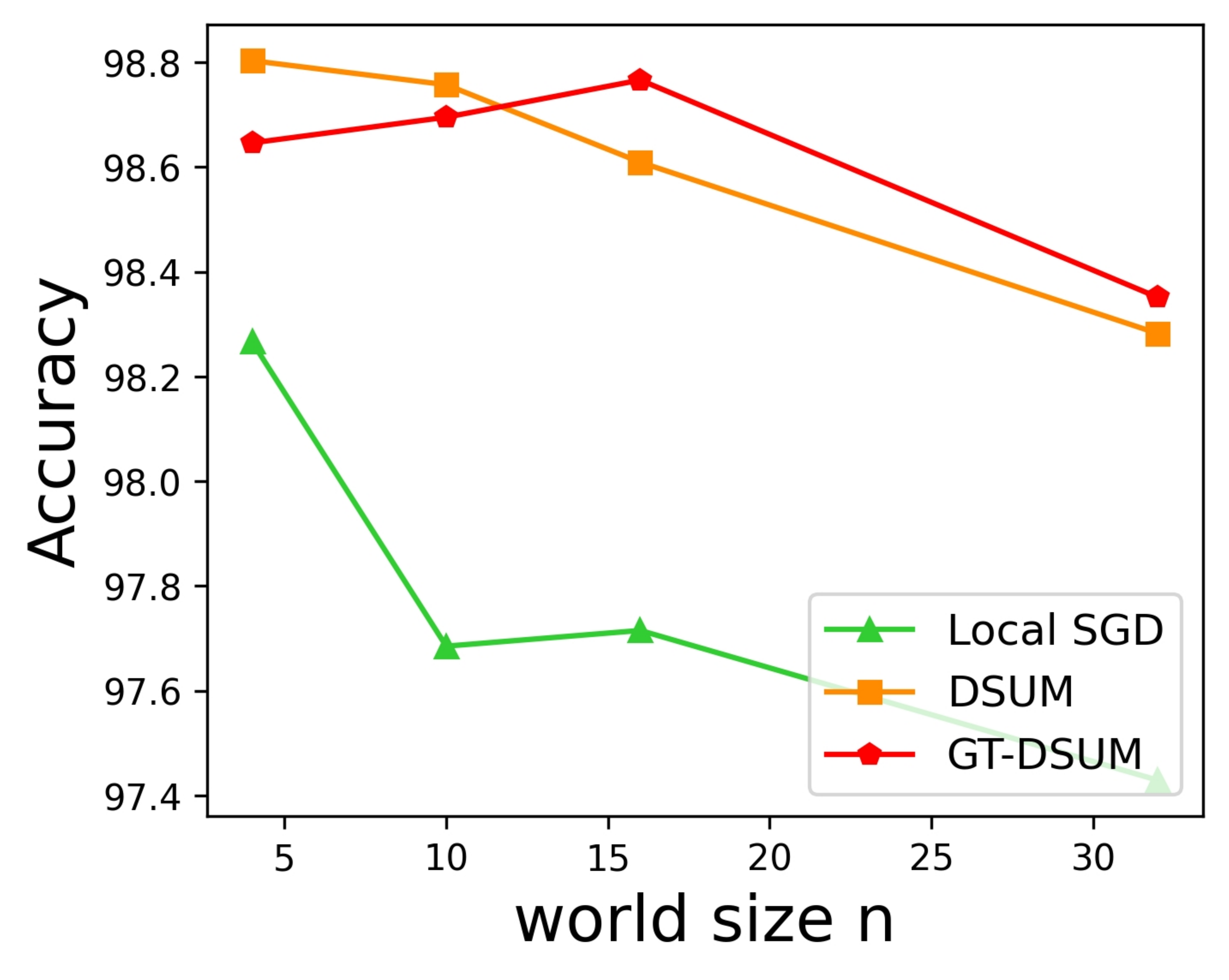}
\label{fig:impact_ws_mnist}}
\subfigure[CNN over EMNIST]{
\includegraphics[width=0.23\textwidth]{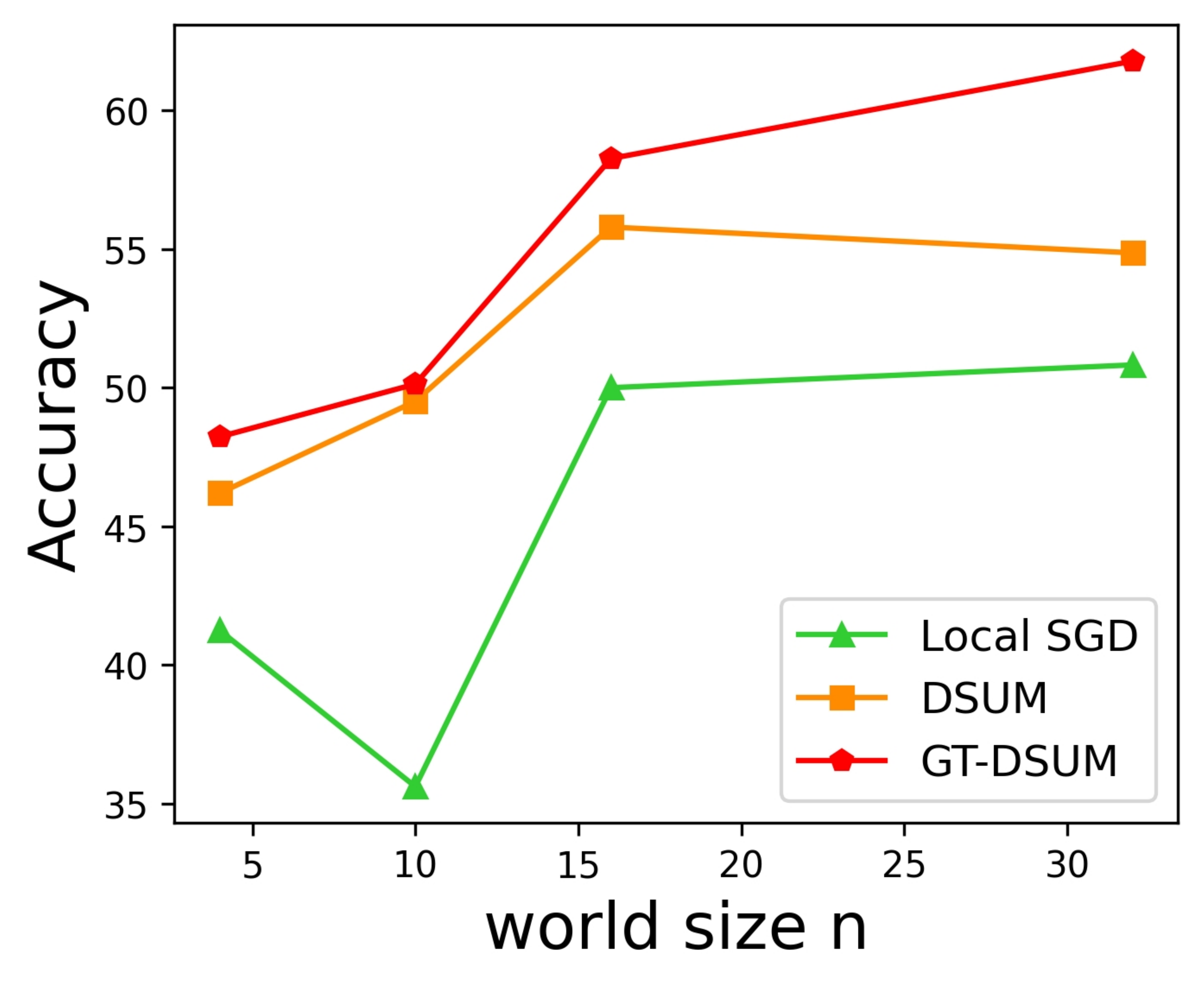}
\label{fig:impact_ws_emnist}}
\subfigure[LeNet over CIFAR10]{
\includegraphics[width=0.23\textwidth]{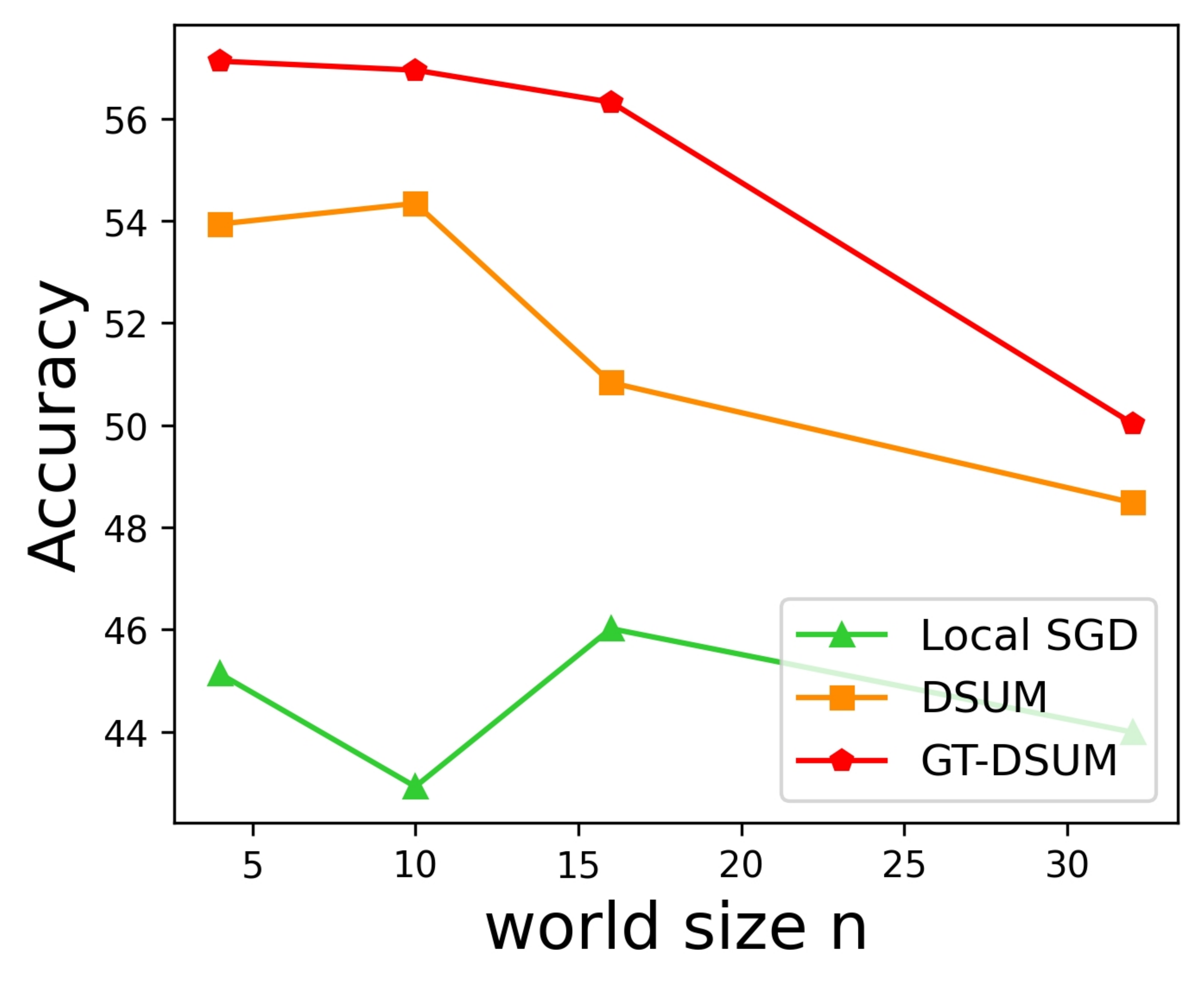}
\label{fig:impact_ws_cifar10}}
\subfigure[RNN over AG NEWS]{
\includegraphics[width=0.23\textwidth]{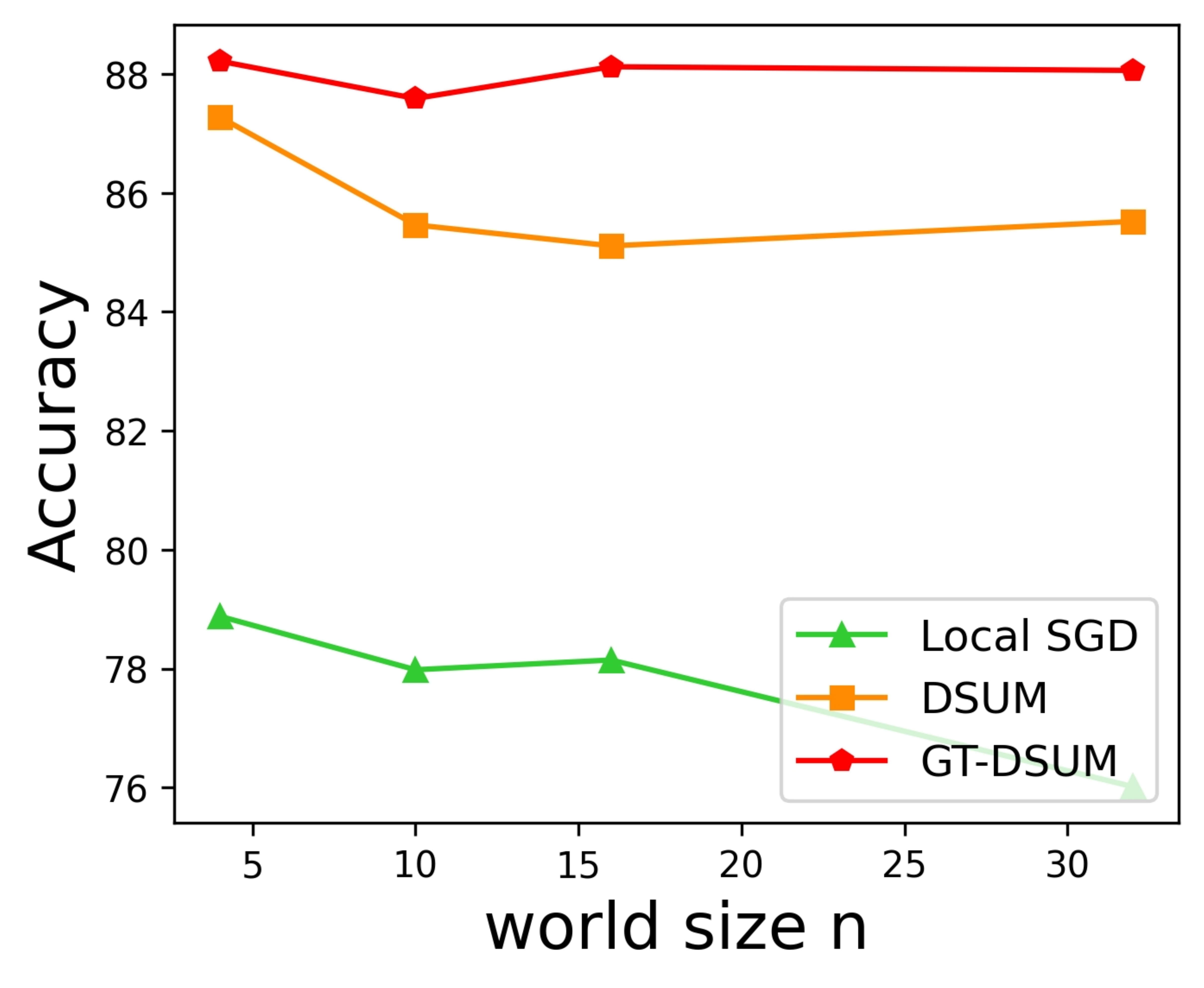}
\label{fig:impact_ws_agnews}}
\caption{Impact on the momentum world size on the convergence when the number of local update $K = 10$ under the non-IID $= 1$ case.}
\label{fig:impact_ws}
\end{figure*}

\para{Effect of local update.}
The number of local updates $K$ is one of the most important parameters since it influences the final model generalization and training time. As usual, the number of $K$ is set less than $20$~\cite{reddi2020adaptive,qu2022generalized}. Hence, we present the comparison with $K \in \left \{ 1, 5, 10, 15, 20 \right \}$.
We make two observations from the results in Figure~\ref{fig:impact_lu}. Firstly, our algorithms have better performance than \textt{Local SGD w/ momentum} regardless $K$. 
For CIFAR10 and AG NEWS, shown in Figure~\ref{fig:impact_lu_cifar10}, \ref{fig:impact_lu_agnews}, we observe that \textt{\algtwo} always keep competitive when the number of local updates increases. Among them, it 
improves accuracy $53.8\%$ than \textt{Local SGD w/ momentum} when $K = 5$.
Secondly, 
we can find see that the workers may not guarantee to improve the model generalization substantially by increasing the number of local updates $K$.
Besides, all benchmarks perform worst when $K = 1$, while when $K$ is too large, performance may be degraded because workers' local models drift too far apart in distributed optimizations~\cite{qu2022generalized}.

\para{Effect of $\beta$.}
The momentum term $\beta$ is further investigated via grid search on the set $\left \{ 0,01, 0.1, 0.5, 0.9, 0.99 \right \}$ as different strategies for handling algorithms.
The third set of simulations evaluates the performance of model accuracy on different $\beta$, which are depicted in Figure~\ref{fig:impact_momentum}. 
We have a key observation from those results. Regardless of datasets or models, evaluation results with a greater $\beta$ (\eg, 0.9, 0.99) trend to outperform with a smaller one (\eg, 0.01, 0.1). In addition, it can be observed that the testing accuracy monotonically increases with $\beta$ on CIFAR10 and AG NEWS. 
We note that \textt{\algtwo} reaches a higher model accuracy compared with \textt{Local SGD w/ momentum} when $\beta$ is increasing. 
Among all tasks, it improves test accuracy up to $56.8\%$
than \textt{Local SGD w/ momentum}.

\para{Sensitivity of $\alpha$.}
One of the most important parameters $\alpha$ is varying from $0$ to $15$ to analyze its influence on the convergence performance.
Two observations are as follows according to Table~\ref{tbl:impact_alpha}. 
Firstly, we note that 
different optimal values of $\alpha$ are always found when \textt{\algone} is evaluated on various datasets with the same level of data heterogeneity (\eg, non-IID $= 1$).
It is hard and time-consuming to determine the optimum $\alpha$ due to different characteristics of datasets and models.
Secondly, as $\alpha$ increases, \textt{\algtwo} makes a significant degradation on model performance. This phenomenon verifies the analysis detailed in Remark~\ref{remark:gt}, which indicates that \textt{\algtwo} requires a stricter constraint on $\alpha$ than \textt{\algone} to ensure model validity. 
Empirically, there is still a gap between the vulnerable property of $\alpha$ to the momentum-based optimizer and the robustness endowing with a superior performance. 

\para{Scalability.}
We finally train on different numbers of workers compared with \textt{Local SGD w/ momentum} when non-IID $= 1$. We evaluate this by extending the scale by adjusting the number of devices $n$ training on $4$, $10$, $16$, and $32$ workers. Results are shown in Figure~\ref{fig:impact_ws}. When the number of participating 
workers increases, the advantage of our schemes is readily apparent since our method \textt{\algtwo} consistently reaches a higher model accuracy compared to the \textt{Local SGD w/ momentum} in this non-IID case.

\section{Conclusion}
\label{sec:conclusion}
In this paper,  we 
propose a unified momentum-based paradigm \framework{} with two algorithms \textt{\algone} and \textt{\algtwo}. The former obtains good model generalization, dealing with the validity under non-convex cases, while the latter is further developed by applying the GT technique to eliminate the negative impact of heterogeneous data. 
By deriving the convergence of general non-convex settings, these algorithms achieve competitive performance closely related to a critical parameter $\alpha$. 
Extensive experimental results show our \framework{} leads to at most 
$57.6\%$ increase in improvement of accuracy. 


\bibliographystyle{named}
\bibliography{reference}

\appendix
\onecolumn{}
\section{Prerequisite}
\label{appendix:prerequisite}
For giving the theoretical analysis of the convergence results of all proposed algorithms, we first present some preliminary facts as follows:
\begin{itemize}
    \item \para{Fact 1: } For any random vector $\mathbf{a}$, it holds for $\mathbb{E} \left \| \mathbf{a} \right \|^2 = \mathbb{E} \left \| \mathbf{a} - \mathbb{E} \left [ \mathbf{a}  \right ]  \right \|^2 + \left \| \mathbb{E} \left [ \mathbf{a}  \right ]   \right \|^2$.
    \item \para{Fact 2: } For any $a >0$, we have $ \pm \left \langle \mathbf{a}, \mathbf{b}   \right \rangle \le \frac{1}{2a} \left \| \mathbf{a}  \right \|^2 + \frac{a}{2} \left \| \mathbf{b}  \right \|^2$.
    \item \para{Fact 3: } $\left \langle \mathbf{a}, \mathbf{b}   \right \rangle = \frac{1}{2} \left \| \mathbf{a}  \right \|^2 + \frac{1}{2} \left \| \mathbf{b}  \right \|^2 - \frac{1}{2} \left \| \mathbf{a} - \mathbf{b}   \right \|^2    $.
    \item \para{Fact 4: } For given two vectors $\mathbf{a}$ and $\mathbf{b}, \forall a > 0$, we have $\left \| \mathbf{a} + \mathbf{b} \right \|^2 \le (1+a) \left \| \mathbf{a} \right \|^2 + (1 + \frac{1}{a}) \left \| \mathbf{b}  \right \|^2$.
    \item \para{Fact 5: } For arbitrary set of $n$ vectors $\left \{ \mathbf{a}_i \right \}_{i=1}^n$, we have $\left \| \sum_{i=1}^n \mathbf{a}_i  \right \|^2 \le n \sum_{i=1}^n \left \| \mathbf{a}_i  \right \|^2$.
    \item \para{Fact 6: } Suppose $\left \{ b_i \right \}_{i=1}^n$, and $\left \{ \mathbf{a} _i \right \}_{i=1}^n $ are a set of non-negative scalars and vectors, respectively. We define $s=\sum_{i=1}^n b_i$. Then according to Jensen's inequality, we have $\left \| \sum_{i=1}^n b_i \mathbf{a}_i  \right \|^2 = s^2 \left \| \sum_{i=1}^n \frac{b_i}{s} \mathbf{a}_i \right \|^2 \le s^2 \cdot \sum_{i=1}^n \frac{b_i}{s}  \left \| \mathbf{a}_i  \right \|^2 = s\cdot \sum_{i=1}^n b_i \left \| \mathbf{a}_i  \right \|^2$. 
\end{itemize}
The inequalities of \para{Fact 4} also hold for the sum of two matrices $\mathbf{A}, \mathbf{B}$ in Frobenius norm.

\begin{proposition}
\label{proposition:consensus_matrix_propoerty}
One step of gossip averaging with the mixing matrix $\mathbf{W}$ defined in the Definition~\ref{def:consensus_matrix} preserves the averaging of the iterates, \ie,
$\mathbf{XW}\frac{\mathbf{11}^{\top}}{n} = \mathbf{X}\frac{\mathbf{11}^{\top}}{n}$.
\end{proposition}

Note our schemes have the following observation on the role of momentum, we now state a basic lemma:
\begin{lemma}
\label{lem:auxiliary_relationship}
Let introduce an auxiliary variable $\mathbf{y}_i^{(t),\tau} = \frac{\beta}{1-\beta} \left ( \mathbf{x}_i^{(t),\tau} - \mathbf{v}_i^{(t),\tau}  \right )$, we define $\mathbf{z}_i^{(t),\tau} \triangleq \mathbf{x}_i^{(t),\tau} + \mathbf{y}_i^{(t),\tau}$ and $\mathbf{c}_i^{(t),\tau} \triangleq \frac{1-\beta}{\beta} \mathbf{y}_i^{(t),\tau} $. Then we have
\begin{equation}
\label{eq:sgd_with_auxiliary}
\mathbf{z}_i^{(t), \tau+1} = \mathbf{z}_i^{(t),\tau} - \frac{\eta}{1-\beta} \mathbf{g}_i^{(t),\tau}
\end{equation}
and
\begin{equation}
\label{eq:recursive_with_auxiliary}
\mathbf{c}_i^{(t),\tau+1} = \beta \mathbf{c}_i^{(t),\tau} + \left ( \alpha - \alpha\beta -1  \right )\eta \mathbf{g}_i^{(t),\tau}.
\end{equation}
\end{lemma}

\begin{proof}
For (\ref{eq:sgd_with_auxiliary}), starting from the definition of $\mathbf{v}_i^{(t),\tau}$ in (\ref{eq:sum}),
\begin{equation}
\begin{aligned}
\mathbf{z}_i^{(t),\tau+1} 
& =  \frac{1}{1-\beta} \mathbf{x}_i^{(t),\tau+1} 
- \frac{\beta}{1-\beta} \left ( \mathbf{x}_i^{(t),\tau} - \alpha\eta \mathbf{g}_i^{(t),\tau} \right ) \\
& = \mathbf{x}_i^{(t),\tau} + \frac{\beta}{1-\beta} 
\left ( \mathbf{x}_i^{(t),\tau} - \mathbf{x}_i^{(t),\tau-1} + \alpha\eta\mathbf{g}_i^{(t),\tau-1}  \right ) - \frac{\eta}{1-\beta} \mathbf{g}_i^{(t),\tau} \\
& = \mathbf{x}_i^{(t),\tau} + \mathbf{y}_i^{(t),\tau} - \frac{\eta}{1-\beta} \mathbf{g}_i^{(t),\tau}. \notag    
\end{aligned}
\end{equation}
Similarly for (\ref{eq:recursive_with_auxiliary}), 
\begin{equation}
\begin{aligned}
\mathbf{c}_i^{(t),\tau+1} 
& = \mathbf{x}^{(t),\tau} - \eta \mathbf{g}_i^{(t),\tau} - \mathbf{x}_i^{(t),\tau} + \alpha \eta \mathbf{g}_i^{(t),\tau} + \beta \left ( \mathbf{x}^{(t),\tau} - \alpha\eta \mathbf{g}_i^{(t),\tau} - \mathbf{x}^{(t),\tau-1} + \alpha\eta \mathbf{g}_i^{(t),\tau-1} \right ) \\
& = \beta \left ( \mathbf{x}_i^{(t),\tau} - \mathbf{v}_i^{(t),\tau} \right ) 
+ \left ( \alpha - \alpha\beta -1\right ) \eta \mathbf{g}_i^{(t),\tau}.
\end{aligned}
\end{equation}
\end{proof}

\section{Proof of \textt{\algone{}}}
\label{appendix:proof_alg1}
Since the smoothness of $f$, it follows that
\begin{equation}
\label{eq:begin_smooth}
\begin{aligned}
\mathbb{E}_{t,\tau} f(\bar{\mathbf{z}}^{(t), \tau+1}) - f(\bar{\mathbf{z}}^{(t), \tau}) 
& \le -\frac{\eta}{1-\beta} \left \langle \nabla f(\bar{\mathbf{z}}^{(t), \tau}), 
\mathbb{E}_{t,\tau} \left [ \bar{\mathbf{g} }^{(t), \tau} \right ]  \right \rangle + \frac{\eta^2 L}{2(1-\beta)^2} \mathbb{E}_{t, \tau} \left \| \bar{\mathbf{g} }^{(t), \tau} \right \|^2,
\end{aligned}
\end{equation}
where $\mathbb{E}_{t,\tau}$ denotes a conditional expectation over the randomness in the $\tau$-th local updates under epoch $t$, conditioned on all past random variables. According to the described factors 
, on the right hand side for (\ref{eq:begin_smooth}):
\begin{equation}
\label{eq:rearrange_one}
\begin{aligned}
& - \left \langle \nabla f(\bar{\mathbf{z}}^{(t), \tau}),
\mathbb{E}_{t,\tau} \left [ \bar{\mathbf{g} }^{(t), \tau} \right ]  \right \rangle \\
& \quad \overset{(a)}{\le} \frac{1}{2a} \left \| \nabla f(\bar{\mathbf{z}}^{(t), \tau}) - \nabla f(\bar{\mathbf{x}}^{(t), \tau}) \right \|^2
- \frac{1-a}{2} \left \| \mathbb{E}_{t,\tau} \left [ \bar{\mathbf{g} }^{(t), \tau} \right ] \right \|^2  - \frac{1}{2} \left \| \nabla f(\bar{\mathbf{x}}^{(t), \tau}) \right \|^2 
+ \frac{1}{2} \left \| \nabla f(\bar{\mathbf{x}}^{(t), \tau}) - \mathbb{E}_{t,\tau} \left [ \bar{\mathbf{g} }^{(t), \tau} \right ] \right \|^2  \\
& \quad \overset{(b)}{\le} \frac{L^2}{2a} \left \| \bar{\mathbf{z}}^{(t),\tau} - \bar{\mathbf{x}}^{(t),\tau} \right \|^2 
- \frac{1-a}{2} \left \| \mathbb{E}_{t,\tau} \left [ \bar{\mathbf{g} }^{(t), \tau} \right ] \right \|^2 - \frac{1}{2} \left \| \nabla f(\bar{\mathbf{x}}^{(t), \tau}) \right \|^2 
+ \frac{1}{2} \left \| \nabla f(\bar{\mathbf{x}}^{(t), \tau}) - \mathbb{E}_{t,\tau} \left [ \bar{\mathbf{g} }^{(t), \tau} \right ] \right \|^2,  \\
\end{aligned}
\end{equation}
where $(a)$ follows from the combination of \textbf{Fact 2} and \textbf{Fact 3}; $(b)$ follows by the smoothness in Assumption~\ref{ass:l_smooth}.
Since we assume that for any $i$ in the initial stage $\mathbf{x}_i^{(0)} = \mathbf{v}_i^{(0)} = \mathbf{0}$, based on the definition of $\mathbf{z}^{(t), \tau}$, and $\mathbf{y}^{(t), \tau}$, it can be shown by averaging
\begin{equation}
\left \| \bar{\mathbf{z}}^{(t), \tau} - \bar{\mathbf{x}}^{(t), \tau} \right \|^2
= \frac{\beta^2}{(1-\beta)^2} 
\| \underset{\bar{\mathbf{c}}^{(t),\tau}}{\underbrace{\bar{\mathbf{x}}^{(t), \tau} - \bar{\mathbf{v}}^{(t), \tau}}} \|^2. 
\end{equation}
Applying the recursion of (\ref{eq:recursive_with_auxiliary}) in Lemma~\ref{lem:auxiliary_relationship},
\begin{equation}
\label{eq:rearrange_two}
\begin{aligned}
\left \| \bar{\mathbf{z}}^{(t), \tau} - \bar{\mathbf{x}}^{(t), \tau} \right \|^2
& \overset{(a)}{=} \left \| \bar{\mathbf{z}}^{l} - \bar{\mathbf{x}}^{l} \right \|^2 \\
& \overset{(b)}{=}  \frac{\beta^2\hat{\eta}^2 s^2}{(1-\beta)^2} \left \| \sum_{j=0}^{l-1} 
\frac{\beta^{l-1-j}}{s} \bar{\mathbf{g}}^{j} \right \|^2 \\
& \overset{(c)}{\le} \frac{\beta^2\hat{\eta}^2 s^2}{(1-\beta)^2} \sum_{j=0}^{l-1} 
\frac{\beta^{l-1-j}}{s} \left \| \bar{\mathbf{g}}^{j} \right \|^2 \\
& \overset{(d)}{\le} \frac{\beta^2\hat{\eta}^2}{(1-\beta)^3} \sum_{j=0}^{l-1} \beta^{l-1-j} 
\left \| \bar{\mathbf{g}}^{j} \right \|^2,
\end{aligned}
\end{equation}
where we omit the aspects of epoch (\ie, $t$), local updates (\ie, $\tau$) and replace them with a more general term: \textbf{iteration} (\ie, $l$) in $(a)$; we define $s=\sum_{j=0}^{l-1} \beta^{l-1-j}$ in $(b)$; $(c)$ follows by the \textbf{Fact 6} 
and Jensen's inequality; $(d)$ follows because $s = \frac{1-\beta^l}{1-\beta} < \frac{1}{1-\beta}$ since $\beta \in [0, 1)$.
Substituting (\ref{eq:rearrange_two}) into (\ref{eq:rearrange_one}), and we set $a = \frac{1}{2}$, which yields
\begin{equation}
\label{eq:rearrange_three}
\begin{aligned}
& - \left \langle \nabla f(\bar{\mathbf{z}}^{(t), \tau}),
\mathbb{E}_{t,\tau} \left [ \bar{\mathbf{g} }^{(t), \tau} \right ]  \right \rangle \\
& \quad \le \frac{1}{2} \left \| \nabla f(\bar{\mathbf{x}}^{(t), \tau}) - \mathbb{E}_{t,\tau} 
\left [ \bar{\mathbf{g} }^{(t), \tau} \right ] \right \|^2 
+ \frac{\beta^2\hat{\eta}^2L^2}{(1-\beta)^3} \sum_{j=0}^{l-1} \beta^{l-1-j} 
\left \| \bar{\mathbf{g}}^{l} \right \|^2 - \frac{1}{2} \left \| \nabla f(\bar{\mathbf{x}}^{(t), \tau}) \right \|^2 - \frac{1}{4} \left \| \mathbb{E}_{t,\tau} \left [ \bar{\mathbf{g} }^{(t), \tau} \right ] \right \|^2.
\end{aligned}
\end{equation}
Moreover, for the second term in (\ref{eq:begin_smooth}), we apply the identified \textbf{Fact 1} for any vector and Assumption~\ref{ass:bound_variance}, then
\begin{equation}
\mathbb{E}_{t,\tau}\left [ \| \bar{\mathbf{g}}^{t,\tau} \|^2  \right ] \le \left \| \mathbb{E}_{t,\tau} [ \bar{\mathbf{g}}^{t,\tau}]  \right \|^2 + \frac{\sigma^2}{n}  
\end{equation}
Then plug (\ref{eq:rearrange_three}) and the above inequality into (\ref{eq:begin_smooth}),
\begin{equation}
\begin{aligned}
\mathbb{E}_{t,\tau} f(\bar{\mathbf{z}}^{(t), \tau+1}) - f(\bar{\mathbf{z}}^{(t), \tau})
& \le  \frac{\tilde{\eta}^2 \sigma^2 L}{2n} + \left ( \frac{\tilde{\eta}^2L}{2} - \frac{\tilde{\eta}}{4}  \right ) 
\left \| \mathbb{E}_{t,\tau} \left [ \bar{\mathbf{g} }^{(t), \tau} \right ] \right \|^2 -\frac{\tilde{\eta}}{2} \left \| \nabla  f(\bar{\mathbf{x}}^{(t), \tau}) \right \|^2 \\ & \quad +\frac{\tilde{\eta}}{2} \left \| \nabla f(\bar{\mathbf{x}}^{(t), \tau}) - \mathbb{E}_{t,\tau} 
\left [ \bar{\mathbf{g} }^{(t), \tau} \right ] \right \|^2 + \frac{\beta^2\hat{\eta}^2L^2}{(1-\beta)^3} \sum_{j=0}^{l-1} \beta^{l-1-j} 
\left \| \bar{\mathbf{g}}^{j} \right \|^2,   
\end{aligned}
\end{equation}
where $\tilde{\eta} = \frac{\eta}{1-\beta}$. Taking the total expectation, and summing from $\tau=0$ to $K-1$, we have
\begin{equation}
\label{eq:sum_local_step}
\begin{aligned}
\mathbb{E} \left [  f(\bar{\mathbf{z}}^{(t), K} ) - f(\bar{\mathbf{z}}^{(t), 0} ) \right ] & = \mathbb{E} \left [ f(\bar{\mathbf{z}}^{(t+1), 0} ) - f(\bar{\mathbf{z}}^{(t), 0} ) \right ] \\
& \le -\frac{\tilde{\eta}}{2} \sum_{\tau=0}^{K-1} \mathbb{E}\left \| \nabla f(\bar{\mathbf{x}}^{(t), \tau}) \right \|^2
+ \left ( \frac{\tilde{\eta}^2L}{2} - \frac{\tilde{\eta}}{4}  \right ) \sum_{\tau=0}^{K-1}
\mathbb{E}\left \| \mathbb{E}_{t,\tau} \left [ \bar{\mathbf{g} }^{(t), \tau} \right ] \right \|^2 \\
& \quad +\frac{\tilde{\eta}}{2} \sum_{\tau=0}^{K-1} \mathbb{E}\left \| \nabla f(\bar{\mathbf{x}}^{(t), \tau}) 
- \mathbb{E}_{t,\tau} \left [ \bar{\mathbf{g} }^{(t), \tau} \right ] \right \|^2 + \frac{\tilde{\eta}^2 K\sigma^2 L}{2n} \\
& \quad + \frac{\beta^2\hat{\eta}^2L^2}{(1-\beta)^3}\sum_{\tau=0}^{K-1} 
\mathbb{E} \left [ \sum_{j=0}^{l-1} \beta^{l-1-j}  \left \| \bar{\mathbf{g}}^{j} \right \|^2
 \right ].
\end{aligned}
\end{equation}
Summing from $t=0$ to $T-1$ and dividing both side by $K T$,
\begin{equation}
\label{eq:sum_epoch}
\begin{aligned}
\frac{1}{KT} \mathbb{E} \left [ f(\bar{\mathbf{z}}^{(T), 0} ) - f(\bar{\mathbf{z}}^{(0), 0} ) \right ]
& \le -\frac{\tilde{\eta}}{2 KT} \sum_{t=0}^{T-1}\sum_{\tau=0}^{K-1} \mathbb{E}\left \| \nabla f(\bar{\mathbf{x}}^{(t), \tau}) \right \|^2
+ \frac{\tilde{\eta}^2\sigma^2 L}{2n} \\
& \quad + \left ( \frac{\tilde{\eta}^2L}{2K T} - \frac{\tilde{\eta}}{4K T}  \right ) \underset{T_1}{\underbrace{\sum_{t=0}^{T-1}\sum_{\tau=0}^{K-1}
\mathbb{E}\left \| \mathbb{E}_{t,\tau} \left [ \bar{\mathbf{g} }^{(t), \tau} \right ] \right \|^2}} \\
& \quad +\frac{\tilde{\eta}}{2K T} \underset{T_2}{\underbrace{\sum_{t=0}^{T-1}\sum_{\tau=0}^{K-1} \mathbb{E}\left \| \nabla f(\bar{\mathbf{x}}^{(t), \tau}) 
- \mathbb{E}_{t,\tau} \left [ \bar{\mathbf{g} }^{(t), \tau} \right ] \right \|^2}} \\
& \quad + \frac{\beta^2\hat{\eta}^2L^2}{(1-\beta)^3KT}\underset{T_3}{\underbrace{\sum_{l=1}^{KT-1}
\mathbb{E} \left [ \sum_{j=0}^{l-1} \beta^{l-1-j}  \left \| \bar{\mathbf{g}}^{j} \right \|^2 \right ]}}.
\end{aligned}
\end{equation}
We now bound the upper bound of $T_1$:
\begin{equation}
\label{eq:sum_t1_bound}
\begin{aligned}
\sum_{t=0}^{T-1}\sum_{\tau=0}^{K-1} \mathbb{E}
\left \| \mathbb{E}_{t,\tau} \left [ \bar{\mathbf{g} }^{(t), \tau} \right ] \right \|^2
& = \sum_{t=0}^{T-1}\sum_{\tau=0}^{K-1} \mathbb{E}
\left \| \mathbb{E}_{t,\tau} \left [ \bar{\mathbf{g} }^{(t), \tau} \right ] \pm \bar{\mathbf{g} }^{(t), \tau} \right \|^2 \\
&  \overset{(a)}{\le} 2\sum_{t=0}^{T-1}\sum_{\tau=0}^{K-1} \mathbb{E} \left \| \mathbb{E}_{t,\tau} \left [ \bar{\mathbf{g} }^{(t), \tau} \right ] - \bar{\mathbf{g} }^{(t), \tau} \right \|^2
+ 2\sum_{t=0}^{T-1} \sum_{\tau=0}^{K-1}\mathbb{E} \left \| \bar{\mathbf{g} }^{(t), \tau} \right \|^2 \\
& \overset{(b)}{\le} \frac{2KT\sigma^2}{n} 
+ 2\sum_{t=0}^{T-1} \sum_{\tau=0}^{K-1}\mathbb{E} \left \| \bar{\mathbf{g} }^{(t), \tau} \right \|^2,
\end{aligned}
\end{equation}
where $(a)$ follows because of the \textbf{Fact 4} by setting $a=1$; $(b)$ follows from the Assumption~\ref{ass:bound_variance}. Then we estimate the $T_2$:
\begin{equation}
\label{eq:sum_t2_bound}
\begin{aligned}
\sum_{t=0}^{T-1}\sum_{\tau=0}^{K-1} \mathbb{E}\left \| \nabla f(\bar{\mathbf{x}}^{(t), \tau}) 
- \mathbb{E}_{t,\tau} \left [ \bar{\mathbf{g} }^{(t), \tau} \right ] \right \|^2
& = \sum_{t=0}^{T-1}\sum_{\tau=0}^{K-1} \mathbb{E}\left \| \frac{1}{n}\sum_{i=1}^n \nabla f(\bar{\mathbf{x}}^{(t), \tau}) 
-  \frac{1}{n}\sum_{i=1}^n \nabla f_i(\mathbf{x}_i^{(t), \tau} ) \right \|^2 \\
& \overset{(a)}{\le} \frac{1}{n^2} \sum_{t=0}^{T-1}\sum_{\tau=0}^{K-1}\sum_{i=1}^n 
\mathbb{E} \left \| \nabla f(\bar{\mathbf{x}}^{(t), \tau}) - \nabla f_i(\mathbf{x}_i^{(t), \tau} )
\pm \nabla f_i(\bar{\mathbf{x}}^{(t), \tau}) \right \|^2 \\
& \overset{(b)}{\le} \frac{2}{n^2} \sum_{t=0}^{T-1}\sum_{\tau=0}^{K-1}\sum_{i=1}^n 
\left ( \mathbb{E}\left \| \nabla f(\bar{\mathbf{x}}^{(t), \tau}) - \nabla f_i(\bar{\mathbf{x}}^{(t), \tau}) \right \|^2 
+ \mathbb{E} \left \| \nabla f_i(\bar{\mathbf{x}}^{(t), \tau})  - \nabla f_i(\mathbf{x}_i^{(t), \tau}) \right \|^2  \right ) \\
& \overset{(c)}{\le} \frac{2KT \zeta^2}{n} + \frac{2L^2}{n^2} \sum_{t=0}^{T-1}\sum_{\tau=0}^{K-1}\sum_{i=1}^n 
\mathbb{E} \left \| \bar{\mathbf{x}}^{(t), \tau} -  \mathbf{x}_i^{(t), \tau} \right \|^2,  
\end{aligned}
\end{equation}
where $(a)$ follows from the complexity of $\left \| \cdot \right \|^2$ and Jensen's inequality; $(b)$ follows by the \textbf{Fact 4} 
; $(c)$ follows by applying $\frac{1}{n} \sum_{i=1}^n\mathbb{E} \left \| \nabla f_i(\mathbf{x} ) - \nabla f(\mathbf{x}) \right \|^2 \le \zeta^2$ in Assumption~\ref{ass:bound_variance}. Finally, we bound the term $T_3$:
\begin{equation}
\label{eq:sum_t3_bound}
\begin{aligned}
\sum_{l=1}^{KT-1}
\mathbb{E} \left [ \sum_{j=0}^{l-1} \beta^{l-1-j}  \left \| \bar{\mathbf{g}}^{j} \right \|^2 \right ]
& \le \sum_{j=0}^{KT-2} \mathbb{E} \left [ \left \| \bar{\mathbf{g}}^j \right \|^2 \sum_{u=j+1}^{KT-1} \beta^{u-1-j} \right ]\\
& \overset{(a)}{\le} \frac{1}{1-\beta} \sum_{j=0}^{KT-2} \mathbb{E} \left \| \bar{\mathbf{g}}^j  \right \|^2 \\
& \le \frac{1}{1-\beta} \sum_{j=0}^{KT-1} \mathbb{E} \left \| \bar{\mathbf{g}}^j  \right \|^2,
\end{aligned}
\end{equation}
where $(a)$ follows by noting that $\sum_{u=j+1}^{TK-1} \beta ^{u-1-j} = \frac{1-\beta^{TK-1-j}}{1-\beta} \le \frac{1}{1-\beta} $.
Substitute (\ref{eq:sum_t1_bound}), (\ref{eq:sum_t2_bound}), and (\ref{eq:sum_t3_bound}) into
(\ref{eq:sum_epoch}), which yields
\begin{equation}
\label{eq:sum_all_bound}
\begin{aligned}
\frac{1}{KT} \mathbb{E} \left [ f(\bar{\mathbf{z}}^{(T), 0} ) - f(\bar{\mathbf{z}}^{(0), 0} ) \right ]
& \le -\frac{\tilde{\eta}}{2 KT} \sum_{t=0}^{T-1}\sum_{\tau=0}^{K-1} \mathbb{E}\left \| \nabla f(\bar{\mathbf{x}}^{(t), \tau}) \right \|^2
+ \frac{\tilde{\eta}^2\sigma^2 L}{2n}
+ \left ( \frac{\tilde{\eta}^2L}{K T} - \frac{\tilde{\eta}}{2K T}  \right ) \sum_{t=0}^{T-1}\sum_{\tau=0}^{K-1}
\mathbb{E}\left \| \bar{\mathbf{g} }^{(t), \tau} \right \|^2 \\
& \quad \quad +\frac{\tilde{\eta}L^2}{K Tn^2} \sum_{t=0}^{T-1}\sum_{\tau=0}^{K-1} \sum_{i=1}^n\mathbb{E}\left \| \bar{\mathbf{x}}^{(t), \tau}
-  \mathbf{x}_i^{(t), \tau} \right \|^2
+ \frac{\beta^2\hat{\eta}^2L^2}{(1-\beta)^4KT}\sum_{j=0}^{KT-1}
\mathbb{E} \left \| \bar{\mathbf{g}}^{j} \right \|^2 \\
& \quad \quad + \frac{\tilde{\eta}\zeta^2}{n} + \frac{\tilde{\eta}\sigma^2}{2n}\left ( 3\tilde{\eta} L-1 \right ).  
\end{aligned}
\end{equation}
Using that $\sum_{i=1}^n\left \| \mathbf{a}_i  \right \|^2 = \left \| \mathbf{A}  \right \|_F^2$ where $\mathbf{A} := \left [ \mathbf{a}_1, \cdots, \mathbf{a}_n \right ]$, we now try to bound the consensus error 
between the nodes' parameters and its averaging. 
We first reiterate the update scheme of (\ref{eq:sum}) in a matrix form regardless of epoch $t$ and local update $\tau$ and denote $l$ as the index of update iteration:
\begin{equation}
\mathbf{X}^{(l+1)} = \mathbf{X}^{(l)} - \eta \mathbf{G}^{(l)} + \beta \left ( \mathbf{X}^{(l)} 
-\alpha\eta\mathbf{G}^{(l)} - \mathbf{X}^{(l-1)} + \alpha\eta \mathbf{G}^{(l-1)}  \right ). 
\end{equation}
For averaged parameters which are performed model averaging across all nodes, we can also simply the updates since $\mathbf{W}$ is doubly stochastic, which is described as follows:
\begin{equation}
\bar{\mathbf{x}}^{(l+1)} = \bar{\mathbf{x}}^{(l)} - \eta \bar{\mathbf{g}}^{(l)} + \beta \left ( \bar{\mathbf{x}}^{(l)}
-\alpha\eta\bar{\mathbf{g}}^{(l)} - \bar{\mathbf{x}}^{(l-1)} + \alpha\eta \bar{\mathbf{g}}^{(l-1)}  \right ).   
\end{equation}

According to the above two equations, we have
\begin{equation}
\begin{aligned}
\frac{1}{n} \mathbb{E} \left \| \mathbf{X}^{(l+1)} - \bar{\mathbf{X}}^{(l+1)} \right \|_F^2  
& = \frac{1}{n} \mathbb{E} \left \| \mathbf{X}^{(l)} - \eta \mathbf{G}^{(l)} + \beta \left ( \mathbf{X}^{(l)} 
-\alpha\eta\mathbf{G}^{(l)} - \mathbf{X}^{(l-1)} + \alpha\eta \mathbf{G}^{(l-1)}  \right ) \right. \\
& \quad \quad \quad \quad \left. - \left ( \bar{\mathbf{X}}^{(l)} - \eta \bar{\mathbf{G}}^{(l)} + \beta \left ( \bar{\mathbf{X}}^{(l)} 
-\alpha\eta\bar{\mathbf{G}}^{(l)} - \bar{\mathbf{X}}^{(l-1)} + \alpha\eta \bar{\mathbf{G}}^{(l-1)}  \right ) \right )  \right \|_F^2 \\
& \overset{(a)}{\le}  \frac{1-\rho}{n} \mathbb{E} \left \| \left ( 1+\beta \right ) \left ( \mathbf{X}^{(l)} - \bar{\mathbf{X}}^{(l)}  \right )  
- \beta \left ( \mathbf{X}^{(l-1)} - \bar{\mathbf{X}}^{(l-1)}  \right ) 
- \left ( 1 + \alpha\beta \right ) \eta \left ( \mathbf{G}^{(l)} - \bar{\mathbf{G}}^{(l)}  \right ) \right .\\ 
& \quad \quad \quad \quad \left. + \alpha\beta\eta \left ( \mathbf{G}^{(l-1)} - \bar{\mathbf{G}}^{(l-1)}  \right )  \right \|_F^2 \\
& \overset{(b)}{\le} \frac{1-\rho}{n} \mathbb{E} \left \| \left ( 1+\beta \right ) \left ( \mathbf{X}^{(l)} - \bar{\mathbf{X}}^{(l)}  \right )  
- \beta \left ( \mathbf{X}^{(l-1)} - \bar{\mathbf{X}}^{(l-1)}  \right )
- \left ( 1 + \alpha\beta \right ) \eta \left ( \mathbb{E} \left [ \mathbf{G}^{(l)} \right ]  - \mathbb{E} \left [ \bar{\mathbf{G}}^{(l)} \right ] \right ) \right. \\  
& \quad \quad \quad \quad \left. + \alpha\beta\eta \left ( \mathbb{E} \left [ \mathbf{G}^{(l-1)} \right ]  - \mathbb{E} \left [ \bar{\mathbf{G}}^{(l-1)} \right ]  \right )  \right \|_F^2 + \underset{:= \Delta}{\underbrace{4\left ( 1 - \rho \right ) \left ( 1 + 2\alpha\beta + 2\alpha^2\beta^2  \right )  \eta^2\sigma^2}}   \\
& \overset{(c)}{\le} \frac{2(1-\rho)}{n} \mathbb{E} \left \| \left ( 1+\beta \right ) \left ( \mathbf{X}^{(l)} - \bar{\mathbf{X}}^{(l)}  \right ) - \left ( 1 + \alpha\beta \right ) \eta \left ( \mathbb{E} \left [ \mathbf{G}^{(l)} \right ]  - \mathbb{E} \left [ \bar{\mathbf{G}}^{(l)} \right ] \right )  \right \|_F^2 \\
& \quad \quad + \frac{2(1-\rho)}{n} \mathbb{E} \left \| \beta \left ( \mathbf{X}^{(l-1)} - \bar{\mathbf{X}}^{(l-1)}  \right ) - \alpha\beta\eta \left ( \mathbb{E} \left [ \mathbf{G}^{(l-1)} \right ] - \mathbb{E} \left [\bar{\mathbf{G}}^{(l-1)} \right ]  \right )  \right \|_F^2  + \Delta ,   
\end{aligned}
\end{equation}
where $(a)$ follows by applying Assumption~\ref{ass:mixing_matrix}; $(b)$ follows because we add the expectation term of $\mathbf{G}$ and $\bar{\mathbf{G}}$, so that $\mathbf{G} - \bar{\mathbf{G}} = \left ( \mathbf{G} - \mathbb{E} \left [ \mathbf{G} \right ]   \right )
-\left (\bar{\mathbf{G}} - \mathbb{E} \left [ \bar{\mathbf{G}} \right ]\right ) + \left ( \mathbb{E} \left [ \mathbf{G} \right ] - \mathbb{E} \left [ \bar{\mathbf{G}} \right ] \right )$, which satisfies the condition of Assumption~\ref{ass:bound_variance} generalizing the constant $\Delta$; $(c)$ follows from the \textbf{Fact 4} by setting $a=1$. Here we use the contractivity of the matrix $\mathbf{W}$ and Young's inequality. 
We can further proceed as 
\begin{equation}
\begin{aligned}
& \frac{1}{n} \mathbb{E} \left \| \mathbf{X}^{(l+1)} - \bar{\mathbf{X}}^{(l+1)} \right \|_F^2 \\ 
& \overset{(a)}{\le} \frac{2(1-\rho)(1+\beta)^2}{n} \left ( 1+\frac{\rho}{2} \right )  \mathbb{E} \left \| \mathbf{X}^{(l)} - \bar{\mathbf{X}}^{(l)} \right \|_F^2
+  \frac{2(1-\rho)(1+\alpha\beta)^2\eta^2}{n} \left ( 1+\frac{2}{\rho}  \right ) \mathbb{E} \left \| \mathbb{E} \left [ \mathbf{G}^{(l)} \right ]  - \mathbb{E} \left [ \bar{\mathbf{G}}^{(l)} \right ] \right \|_F^2 \\
& \quad +  \frac{2(1-\rho)\beta^2}{n} \left ( 1 + \frac{\rho}{2}  \right ) \mathbb{E} \left \| \mathbf{X}^{(l-1)} - \bar{\mathbf{X}}^{(l-1)} \right \|_F^2  
+ \frac{2(1-\rho)}{n} \alpha^2\beta^2\eta^2 \left ( 1+\frac{2}{\rho}  \right ) \mathbb{E} \left \| \mathbb{E} \left [ \mathbf{G}^{(l-1)} \right ]  - \mathbb{E} \left [ \bar{\mathbf{G}}^{(l-1)} \right ] \right \|_F^2 + \Delta \\
& \le \frac{2(1-\rho)(1+\beta)^2}{n} \left ( 1+\frac{\rho}{2} \right )  \mathbb{E} \left \| \mathbf{X}^{(l)} - \bar{\mathbf{X}}^{(l)} \right \|_F^2
+  \frac{2(1-\rho)(1+\alpha\beta)^2\eta^2}{n} \left ( 1+\frac{2}{\rho}  \right ) \mathbb{E} \left \| \mathbb{E} \left [ \mathbf{G}^{(l)} \right ]  - \nabla f (\bar{\mathbf{X}}^{l}) \right \|_F^2 \\
& \quad +  \frac{2(1-\rho)\beta^2}{n} \left ( 1 + \frac{\rho}{2}  \right ) \mathbb{E} \left \| \mathbf{X}^{(l-1)} - \bar{\mathbf{X}}^{(l-1)} \right \|_F^2  
+ \frac{2(1-\rho)}{n} \alpha^2\beta^2\eta^2 \left ( 1+\frac{2}{\rho}  \right ) \mathbb{E} \left \| \mathbb{E} \left [ \mathbf{G}^{(l-1)} \right ]  - \nabla f (\bar{\mathbf{X}}^{l-1}) \right \|_F^2 + \Delta \\
& = \frac{2(1-\rho)(1+\beta)^2}{n} \left ( 1+\frac{\rho}{2} \right )  \mathbb{E} \left \| \mathbf{X}^{(l)} - \bar{\mathbf{X}}^{(l)} \right \|_F^2
+ \frac{2(1-\rho)(1+\alpha\beta)^2\eta^2}{n} \left ( 1+\frac{2}{\rho}  \right ) \sum_{i=1}^n \mathbb{E} \left \| \nabla f_i(\mathbf{x}_i^{(l)}) \pm \nabla f_i(\bar{\mathbf{x}}^{(l)}) - \nabla f(\bar{\mathbf{x}}^{(l)}) \right \|^2 + \Delta \\
& \quad +  \frac{2(1-\rho)\beta^2}{n} \left ( 1 + \frac{\rho}{2}  \right ) \mathbb{E} \left \| \mathbf{X}^{(l-1)} - \bar{\mathbf{X}}^{(l-1)} \right \|_F^2  
+ \frac{2(1-\rho)}{n} \alpha^2\beta^2\eta^2 \left ( 1+\frac{2}{\rho}  \right ) \sum_{i=1}^n \mathbb{E} \left \| \nabla f_i(\mathbf{x}_i^{(l-1)}) \pm \nabla f_i(\bar{\mathbf{x}}^{(l-1)}) - \nabla f(\bar{\mathbf{x}}^{(l-1)}) \right \|^2 \\
& \overset{(b)}{\le} \frac{2(1-\frac{\rho}{2} )}{n} \left ( 1+\beta \right )^2 \left \| \mathbf{X}^{(l)} - \bar{\mathbf{X}}^{(l)} \right \|_F^2 
+ \frac{12\eta^2L^2}{\rho n} \left ( 1+\alpha\beta \right )^2 \sum_{i=1}^n \mathbb{E} \left \| \mathbf{x}_i^{(l)} - \bar{\mathbf{x}}^{(l)} \right \|^2 + \frac{12\eta^2(1+\alpha\beta)^2\zeta^2}{\rho} \\
& \quad +  \frac{2(1-\frac{\rho}{2} )\beta^2}{n} \mathbb{E} \left \| \mathbf{X}^{(l-1)} - \bar{\mathbf{X}}^{(l-1)} \right \|_F^2  
+ \frac{12\eta^2L^2}{\rho n} \alpha^2\beta^2 \sum_{i=1}^n \mathbb{E} \left \| \mathbf{x}_i^{(l-1)} - \bar{\mathbf{x}}^{(l-1)} \right \|^2 + \frac{12\eta^2\alpha^2\beta^2\zeta^2}{\rho} + \Delta ,
\end{aligned}
\end{equation}
where $(a)$ follows by applying the \textbf{Fact 4} and sets $a=\frac{\rho}{2}$; $(b)$ follows because positive $\rho \le 1$, and using \textbf{Fact 4} as well as Assumption~\ref{ass:bound_variance}. By choosing the learning rate $\eta \le \frac{\rho}{5L}$ ensures that $6\eta^2L^2 \le \frac{\rho^2}{4}$, we have two cases since $a \ge 0$
\begin{itemize}
    \item \para{Case one:} $\alpha \in \left [0,1 \right )$, we define $\hat{\beta} = 1+\beta$, then
\begin{equation}
\begin{aligned}
\frac{1}{n} \mathbb{E} \left \| \mathbf{X}^{(l+1)} - \bar{\mathbf{X}}^{(l+1)} \right \|_F^2
& \le \frac{2\hat{\beta}^2 \left ( 1-\frac{\rho}{4}  \right ) }{n} \mathbb{E} \left \| \mathbf{X}^{(l)} - \bar{\mathbf{X}}^{(l)} \right \|_F^2 + \frac{2\hat{\beta}^2 \left ( 1-\frac{\rho}{4}  \right ) }{n} \mathbb{E} \left \| \mathbf{X}^{(l-1)} - \bar{\mathbf{X}}^{(l-1)} \right \|_F^2 + \frac{24\eta^2 \hat{\beta}^2\zeta^2}{\rho} + \Delta. \notag
\end{aligned}
\end{equation}
    \item \para{Case two:} $\alpha \in \left [1, \infty \right )$, we define $\tilde{\beta} = 1+\alpha\beta$, just replace $\hat{\beta}$ term by $\tilde{\beta}$.
\end{itemize}
Here we denote $\beta_0 \triangleq \max \left \{ \hat{\beta}, \tilde{\beta} \right \}$. Since $\mathbf{x}_i^{(0)} = \bar{\mathbf{x}}^{(0)} = \mathbf{x}_0$, we get $\frac{1}{n} \mathbb{E} \| \mathbf{X}^{(0)} - \bar{\mathbf{X}}^{(0)} \|_F^2 = \mathbf{0}$. Furthermore, we can easily obtain $\frac{1}{n} \mathbb{E}  \| \mathbf{X}^{(1)} - \bar{\mathbf{X}}^{(1)} \|_F^2 = 24\eta^2\beta_0^2\zeta^2/\rho + \Delta$. We observe that when $l \ge 1$, 
$\frac{1}{n} \mathbb{E} \| \mathbf{X}^{(l)} - \bar{\mathbf{X}}^{(l)}  \|_F^2
\le \frac{C_1l}{n\left ( 1-Q_1 \right ) }, $
where we denote $Q_1 = 2\beta_0^2 (1-\frac{\rho}{4})$ ensuring that $Q_1 < 1$, and $C_1 =  24\eta^2\beta_0^2\zeta^2/\rho + \Delta$. Substituting the above inequality into (\ref{eq:sum_all_bound}), for any $\bar{\mathbf{g}}$, $\mathbb{E} \left \| \bar{\mathbf{g}} \right \|^2 \le \frac{G^2}{n}$ by Assumption~\ref{ass:bound_sto_grad}, and $\bar{\mathbf{z}}^{(0)} = \bar{\mathbf{x}}^{(0)} = \mathbf{x}_0$ by definition, rearranging terms yields the Theorem~\ref{thm:dsgd_sum}.

\section{Proof of \textt{\algtwo{}}}
\label{appendix:proof_alg2}
Next, we provide a rigorous proof of \textt{\algtwo{}} under non-convexity. Here we consider a special case where $K=1$ with a fixed consensus matrix. Then we construct the matrix form of Algorithm~\ref{alg:gt_sum_gt} as follows:
\begin{equation}
\label{eq:upper_relationship_eq}
\begin{aligned}
& \mathbf{M}^{(t)} = \lambda \mathbf{G}^{(t)} + (1-\lambda) \mathbf{Y}^{(t)} \\
& \mathbf{X}^{(t+1)} = \mathbf{W}\left ( \left ( 1+\beta \right ) \mathbf{X}^{(t)} - \beta \mathbf{W} \mathbf{X}^{(t-1)} 
- \left ( 1+\alpha\beta \right )\eta \mathbf{M}^{(t)} + \alpha \beta \eta \mathbf{W}\mathbf{M}^{(t-1)} \right )  \\
& \mathbf{Y}^{(t+1)} = \mathbf{W}\left (\mathbf{Y}^{(t)} + \frac{2\mathbf{X}^{(t)} - \mathbf{X}^{(t+1)} - \mathbf{X}^{(t-1)}}{\eta}  \right ).     
\end{aligned}
\end{equation}

\para{Proof sketch.} 
we try to bound the consensus distance (Lemma~\ref{lem:dist_x_gt}) between the worker's parameters and its averaging. During this step, we perform a propagation 
step which brings the parameters of the workers closer to each other. Moreover, we also perform additional gradient tracking (Lemma~\ref{lem:dist_y_gt}) and their accumulation steps (Lemma~\ref{lem:dist_xy_gt}) which move the distance away from each other. After that, we could immediately apply Lemma~\ref{lem:dist_xy_gt} into the the single-step update 
progress 
in (\ref{eq:gt_close}).

\begin{lemma}[Consensus distance change]
\label{lem:dist_x_gt}
Given above assumptions in Section~\ref{sec:design}, let $\beta_0 = \max \left \{ 1+\alpha\beta, 1+\beta \right \}$, and the update rule generated by (\ref{eq:upper_relationship_eq}) using learning rate satisfy $\eta \le \frac{\rho}{12\lambda L}$ when $t \ge 1$,
\begin{equation}
\begin{aligned}
\frac{1}{n} \mathbb{E} \left \| \mathbf{X}^{(t+1)} - \bar{\mathbf{X}}^{(t+1)} \right \|_F^2
& \le \frac{2\left ( 1-\frac{\rho}{4} \right ) \beta_0^2}{n} \mathbb{E} \left \| \mathbf{X}^{(t)} - \bar{\mathbf{X}}^{(t)} \right \|_F^2
+  \frac{2\left ( 1-\frac{\rho}{4} \right )\left ( \beta_0 - 1 \right )^2 }{n} \mathbb{E} \left \| \mathbf{X}^{(t-1)} - \bar{\mathbf{X}}^{(t-1)} \right \|_F^2 \\
& \quad+ \frac{\rho\left(1-\frac{\rho}{2}\right) \left ( 1+\alpha\beta \right )^2\left ( 1-\lambda \right )^2  }{12\lambda^2L^2n} \mathbb{E} \left \| \mathbf{Y}^{(t)} - \bar{\mathbf{Y}}^{(t)} \right \|_F^2 
+ \frac{\rho\left (1-\frac{\rho}{2}\right ) \alpha^2\beta^2\left ( 1-\lambda \right )^2  }{12\lambda^2L^2n} \mathbb{E} \left \| \mathbf{Y}^{(t-1)} - \bar{\mathbf{Y}}^{(t-1)} \right \|_F^2 \\
& \quad + \frac{\rho\left(1-\frac{\rho}{2}\right)}{nL^2} \left ( 1+\alpha\beta \right )^2\left ( \sigma^2+\zeta^2 \right ). \notag
\end{aligned}
\end{equation}
\end{lemma}

\begin{proof}
Starting from (\ref{eq:upper_relationship_eq}) and all consensus matrices satisfy Proposition~\ref{proposition:consensus_matrix_propoerty}, we have
\begin{equation}
\label{eq:dist_x}
\begin{aligned}
\frac{1}{n} \mathbb{E}  \left \| \mathbf{X}^{(t+1)} - \bar{\mathbf{X}}^{(t+1)}  \right \|_F^2
& \overset{(a)}{\le} \frac{1-\rho}{n} \mathbb{E} \left \| \left ( 1+\beta \right ) \left ( \mathbf{X}^{(t)} - \bar{\mathbf{X}}^{(t)} \right ) 
-\beta \left ( \mathbf{WX}^{(t-1)} - \bar{\mathbf{X}}^{(t-1)} \right ) 
-(1+\alpha\beta)\eta \left ( \mathbf{M}^{(t)} - \bar{\mathbf{M}}^{(t)} \right ) \right. \\
& \quad \quad \left. +\alpha\beta\eta \left ( \mathbf{WM}^{(t-1)} - \bar{\mathbf{M}}^{(t-1)} \right )  \right \|_F^2 \\
& \overset{(b)}{\le} \frac{1-\rho}{n} \left ( 1+\frac{\rho}{2}  \right ) \mathbb{E} \left \| \left ( 1+\beta \right ) \left ( \mathbf{X}^{(t)} - \bar{\mathbf{X}}^{(t)} \right ) 
- \beta \left ( \mathbf{WX}^{(t-1)} - \bar{\mathbf{X}}^{(t-1)} \right ) \right \|_F^2 \\ 
& \quad + \frac{1-\rho}{n} \left ( 1+ \frac{2}{\rho}  \right ) \mathbb{E} \left \| (1+\alpha\beta)\eta \left ( \mathbf{M}^{(t)} - \bar{\mathbf{M}}^{(t)} \right ) 
- \alpha\beta\eta \left ( \mathbf{WX}^{(t-1)} - \bar{\mathbf{X}}^{(t-1)} \right ) \right \|_F^2 \\
& \overset{(c)}{\le} \frac{2(1-\rho) (1+\beta)^2}{n} \left ( 1+\frac{\rho}{2}  \right ) \mathbb{E} \left \| \mathbf{X}^{(t)} - \bar{\mathbf{X}}^{(t)} \right \|_F^2 
+ \frac{2(1-\rho)^2 \beta^2}{n} \left ( 1+\frac{2}{\rho}  \right ) \mathbb{E} \left \| \mathbf{X}^{(t-1)} - \bar{\mathbf{X}}^{(t-1)} \right \|_F^2 \\ 
& \quad + \frac{2(1-\rho)(1+\alpha\beta)^2\eta^2}{n} \left ( 1+\frac{\rho}{2}  \right )  \mathbb{E} \left \| \mathbf{M}^{(t)} - \bar{\mathbf{M}}^{(t)} \right \|_F^2 \\
& \quad + \frac{2(1-\rho)^2 \alpha^2\beta^2\eta^2}{n} \left ( 1+\frac{2}{\rho}  \right ) \mathbb{E} \left \| \mathbf{M}^{(t-1)} - \bar{\mathbf{M}}^{(t-1)} \right \|_F^2, 
\end{aligned}
\end{equation}
where $(a)$ follows by applying Assumption~\ref{ass:mixing_matrix}; $(b)$, $(c)$ follows by the \textbf{Fact 4} by choosing $\rho/2$, and $1$ respectively.
Then we try to bound the distance between $\mathbf{M}^{(t)}$ and $\bar{\mathbf{M}}^{(t)}$,
\begin{equation}
\label{eq:dist_m}
\begin{aligned}
\frac{1}{n} \mathbb{E} \left \| \mathbf{M}^{(t)} - \bar{\mathbf{M}}^{(t)} \right \|_F^2
& \overset{(a)}{=} \frac{1}{n} \mathbb{E} \left \| \lambda \left ( \mathbf{G}^{(t)} - \bar{\mathbf{G}}^{(t)} \right ) 
+ \left ( 1-\lambda \right ) \left ( \mathbf{Y}^{(t)} - \bar{\mathbf{Y}}^{(t)} \right )  \right \|_F^2 \\
& \overset{(b)}{\le}  \frac{2\lambda^2}{n} \mathbb{E} \left \| \mathbf{G}^{(t)} - \bar{\mathbf{G}}^{(t)} \right \|_F^2
+ \frac{2(1-\lambda)^2}{n} \mathbb{E} \left \| \mathbf{Y}^{(t)} - \bar{\mathbf{Y}}^{(t)} \right \|_F^2 \\
& =  \frac{2\lambda^2}{n} \mathbb{E} \left \| \mathbf{G}^{(t)} \pm \mathbb{E}_t\left [ \mathbf{G}^{(t)} \right ] 
- \bar{\mathbf{G}}^{(t)}  \pm \mathbb{E}_t \left [ \bar{\mathbf{G}}^{(t)} \right ] \right \|_F^2
+ \frac{2(1-\lambda)^2}{n} \mathbb{E} \left \| \mathbf{Y}^{(t)} - \bar{\mathbf{Y}}^{(t)} \right \|_F^2 \\
& \overset{(c)}{\le} 12\lambda^2\sigma^2 + \frac{6\lambda^2}{n} \mathbb{E} \left \| \mathbb{E}_t\left [ \mathbf{G}^{(t)} \right ] 
- \mathbb{E}_t\left [ \bar{\mathbf{G}}^{(t)} \right ]  \right \|_F^2 
+ \frac{2(1-\lambda)^2}{n} \mathbb{E} \left \| \mathbf{Y}^{(t)} - \bar{\mathbf{Y}}^{(t)} \right \|_F^2 \\
& \le \frac{6\lambda^2}{n} \sum_{i=1}^n \mathbb{E} 
\left \| \nabla f_i(\mathbf{x}_i^{(t)}) \pm \nabla f_i(\bar{\mathbf{x}}^{(t)}) - \nabla f(\bar{\mathbf{x}}^{(t)}) \right \|^2
+ 12\lambda^2\sigma^2 + \frac{2(1-\lambda)^2}{n} \mathbb{E} \left \| \mathbf{Y}^{(t)} - \bar{\mathbf{Y}}^{(t)} \right \|_F^2 \\
& \overset{(d)}{\le} \frac{12 \lambda^2L^2}{n} \sum_{i=1}^n \mathbb{E} \left \| \mathbf{x}_i^{(t)} - \bar{\mathbf{x}}^{(t)}  \right \|^2 
+ 12\lambda^2\left ( \sigma^2 + \zeta^2 \right )  
+ \frac{2(1-\lambda)^2}{n} \mathbb{E} \left \| \mathbf{Y}^{(t)} - \bar{\mathbf{Y}}^{(t)} \right \|_F^2,
\end{aligned}
\end{equation}
where $(a)$ follows from (\ref{eq:upper_relationship_eq}); $(b)$ follows by applying the \textbf{Fact 4}; $(c)$ follows by \textbf{Fact 5} with the vector set $\left \{ \mathbf{G} - \mathbb{E} \left [ \mathbf{G} \right ], \mathbb{E} \left [ \bar{\mathbf{G}} \right ] - \bar{\mathbf{G}},  \mathbb{E} \left [ \mathbf{G} \right ] - \mathbb{E} \left [ \bar{\mathbf{G}} \right ]   \right \} $; $(d)$ follows from the \textbf{Fact 4} and Assumption~\ref{ass:bound_variance}.
Since the positive scalar $\rho \le 1$, substitute (\ref{eq:dist_m}) into (\ref{eq:dist_x}) on the condition that $\eta \le \frac{\rho}{12 \lambda L}$ ensures that $36\lambda^2L^2\eta^2 \le \rho^2/4$, completing the proof.
\end{proof}

\begin{lemma}[Gradient tracker distance change]
\label{lem:dist_y_gt}
Given the assumptions in Section~\ref{sec:design}, when $t\ge 1$, let the learning rate satisfy $\eta \le \min \left \{ \frac{3+\beta}{2\sqrt{3} \lambda L (1+\alpha\beta)} , \frac{1+\beta}{2\sqrt{3}\lambda L \alpha\beta}  \right \}$ and follows from the assumption of hyperparameter that $(1+\alpha\beta) (1-\lambda) \le \frac{1}{2\sqrt{2}}$, which yields
\begin{equation}
\begin{aligned}
\frac{1}{n}\mathbb{E} \left \| \mathbf{Y}^{(t+1)} - \bar{\mathbf{Y}}^{(t+1)} \right \|_F^2 
& \le \frac{16 (1-\rho) (3+\beta)^2}{n\eta^2} \mathbb{E} \left \| \mathbf{X}^{(t)} - \bar{\mathbf{X}}^{(t)} \right \|_F^2
+ \frac{16 (1-\rho) (1+\beta)^2}{n\eta^2} \mathbb{E} \left \| \mathbf{X}^{(t-1)} - \bar{\mathbf{X}}^{(t-1)} \right \|_F^2 \\
& \quad \quad + \frac{4(1-\rho)}{n} \mathbb{E} \left \| \mathbf{Y}^{(t)} - \bar{\mathbf{Y}}^{(t)} \right \|_F^2
+ \frac{2(1-\rho)}{n} \mathbb{E} \left \| \mathbf{Y}^{(t-1)} - \bar{\mathbf{Y}}^{(t-1)} \right \|_F^2 \\
& \quad \quad + \frac{192 \lambda^2}{n} (1-\rho) (1+\alpha\beta)^2(\sigma^2+\zeta^2). \notag
\end{aligned}
\end{equation}
\end{lemma}

\begin{proof}
According to the update scheme in (\ref{eq:upper_relationship_eq}), we can get 
\begin{equation}
\mathbf{X}^{(t)} - \mathbf{X}^{(t+1)} 
= \left ( \mathbf{I} - \left ( 1+ \beta \right ) \mathbf{W} \right ) \mathbf{X}^{(t)}
+ \beta \mathbf{W}^2 \mathbf{X}^{(t-1)} + \left ( 1+\alpha\beta \right ) \eta \mathbf{WM}^{(t)}
- \alpha\beta\eta \mathbf{W}^2\mathbf{M}^{(t-1)},
\end{equation}
and
\begin{equation}
\label{eq:global_est_matrix}
2\mathbf{X}^{(t)} - \mathbf{X}^{(t+1)}  - \mathbf{X}^{(t-1)} 
= \left ( 2\mathbf{I} - \left ( 1+ \beta \right ) \mathbf{W} \right ) \mathbf{X}^{(t)}
+ \left ( \beta \mathbf{W}^2 - \mathbf{I} \right )  \mathbf{X}^{(t-1)} + \left ( 1+\alpha\beta \right ) \eta \mathbf{WM}^{(t)}
- \alpha\beta\eta \mathbf{W}^2\mathbf{M}^{(t-1)}.
\end{equation}
Then, based on (\ref{eq:global_est_matrix}) and (\ref{eq:upper_relationship_eq}), we have
\begin{equation}
\begin{aligned}
\frac{1}{n} \mathbb{E} \left \| \mathbf{Y}^{(t+1)} - \bar{\mathbf{Y}}^{(t+1)} \right \|_F^2 
& \le \frac{2(1-\rho)}{n} \mathbb{E} \left \| \frac{2\mathbf{I} - (1+\beta)\mathbf{W}}{\eta} \left ( \mathbf{X}^{(t)} - \bar{\mathbf{X}}^{(t)} \right ) 
+ \frac{\beta \mathbf{W}^2-\mathbf{I}}{\eta} \left ( \mathbf{X}^{(t-1)} - \bar{\mathbf{X}}^{(t-1)} \right ) \right. \\
& \left. \quad + (1+\alpha\beta)\mathbf{W} \left ( \mathbf{M}^{(t)} - \bar{\mathbf{M}}^{(t)} \right ) 
-\alpha\beta\mathbf{W}^2 \left ( \mathbf{M}^{(t-1)} - \bar{\mathbf{M}}^{(t-1)} \right ) \right \|_F^2 + \frac{2(1-\rho)}{n} \mathbb{E} \left \|\mathbf{Y}^{(t)} - \bar{\mathbf{Y}}^{(t)} \right \|_F^2,   
\end{aligned}
\end{equation}
The inequality holds for Assumption~\ref{ass:mixing_matrix} and \textbf{Fact 4}. Since $\mathbf{W} \prec \mathbf{I}$, $-\mathbf{I} \prec \mathbf{W}$ as well as the product of two doubly stochastic matrices is still doubly stochastic, we have $ 2\mathbf{I} - (1+\beta)\mathbf{W} \prec (3+\beta)\mathbf{I}$; $\beta\mathbf{W}^2 - \mathbf{I} \prec (\beta+1)\mathbf{W}^2 \prec (\beta+1)\mathbf{I}$, we can continue
\begin{equation}
\label{eq:dist_y}
\begin{aligned}
\frac{1}{n} \mathbb{E} \left \| \mathbf{Y}^{(t+1)} - \bar{\mathbf{Y}}^{(t+1)} \right \|_F^2 
& \le \frac{8(1-\rho) (3+\beta)^2}{n\eta^2} \mathbb{E} \left \| \mathbf{X}^{(t)} - \bar{\mathbf{X}}^{(t)} \right \|_F^2
+ \frac{8(1-\rho)(1+\beta)^2}{n\eta^2} \mathbf{E} \left \| \mathbf{X}^{(t-1)} - \bar{\mathbf{X}}^{(t-1)} \right \|_F^2 \\
& \quad \quad + \frac{8(1-\rho)(1+\alpha\beta)^2}{n} \mathbb{E} \left \| \mathbf{M}^{(t)} - \bar{\mathbf{M}}^{(t)} \right \|_F^2
+ \frac{8(1-\rho)\alpha^2\beta^2}{n} \mathbf{E} \left \| \mathbf{M}^{(t-1)} - \bar{\mathbf{M}}^{(t-1)} \right \|_F^2 \\
& \quad \quad + \frac{2(1-\rho)}{n} \mathbb{E} \left \| \mathbf{Y}^{(t)} - \bar{\mathbf{Y}}^{(t)} \right \|_F^2,
\end{aligned}
\end{equation}
where the inequality follows by applying the basic inequality $\left \| \sum_{i=1}^j \mathbf{A}_i \right \|_F^2 \le j \sum_{i=1}^j \left \| \mathbf{A}_i  \right \|_F^2$ for matrices of the same dimension with $j=4$. Plug (\ref{eq:dist_m}) into(\ref{eq:dist_y}) under the condition that $(\mathit{i})$ $\eta \le \frac{3+\beta}{2\sqrt{3} \lambda L (1+\alpha\beta)} $; $(\mathit{ii})$ $\eta \le \frac{1+\beta}{2\sqrt{3} \lambda L \alpha\beta}$; $(\mathit{iii})$ $(1+\alpha\beta) (1 - \lambda) \le \frac{1}{2\sqrt{2}}$ for ease of presentation.
\end{proof}

\begin{lemma}[Distance step improvement]
\label{lem:dist_xy_gt}
When $t \ge 1$, using learning rate $\eta \ge \max \left \{ \frac{2\sqrt{2}(3+\beta)}{\beta_0}, \frac{2\sqrt{2}(1+\beta)}{\beta_0-1}   \right \}$ and $\rho \le \frac{48 \lambda^2L^2}{(1-\lambda)^2}$, satisfy
\begin{equation}
 \begin{aligned}
& \frac{1}{n} \mathbb{E} \left \| \mathbf{X}^{(t+1)} - \bar{\mathbf{X}}^{(t+1)} \right \|_F^2 
+ \frac{1}{n} \mathbb{E} \left \| \mathbf{Y}^{(t+1)} - \bar{\mathbf{Y}}^{(t+1)} \right \|_F^2  \\
& \quad \quad \le \frac{4\beta_0^2 \left ( 1-\frac{\rho}{4}  \right ) }{n} 
\mathbb{E} \left \| \mathbf{X}^{(t)} - \bar{\mathbf{X}}^{(t)} \right \|_F^2 
+ \frac{4\beta_0^2 \left ( 1-\frac{\rho}{4}  \right ) }{n}
\mathbb{E} \left \| \mathbf{X}^{(t-1)} - \bar{\mathbf{X}}^{(t-1)} \right \|_F^2
+ \frac{8(1-\rho)(1+\alpha\beta)^2}{n} \mathbb{E} \left \| \mathbf{Y}^{(t)} - \bar{\mathbf{Y}}^{(t)} \right \|_F^2 \\
& \quad \quad \quad + \frac{8(1-\rho)(1+\alpha\beta)^2}{n} \mathbb{E} \left \| \mathbf{Y}^{(t-1)} - \bar{\mathbf{Y}}^{(t-1)} \right \|_F^2
+ \frac{\beta_0^2\left ( 1-\frac{\rho}{2}  \right )}{nL^2} (\sigma^2+\zeta^2) (192\lambda^2L^2 + \rho ). \notag
\end{aligned}  
\end{equation}
\end{lemma}
\begin{proof}
Adding the results of Lemma~\ref{lem:dist_x_gt} and \ref{lem:dist_y_gt} gives
the result following from our assumption of learning rate $\eta$ and the hyperparameters $\rho, \lambda, L$.
\end{proof}

We now state our convergence results for \textt{\algtwo} in Algorithm~\ref{alg:gt_sum_gt}. Similar to proof process of Theorem~\ref{thm:dsgd_sum}, with the smoothness of $f$,
\begin{equation}
\label{eq:begin_smooth_gt}
\begin{aligned}
\mathbb{E}f(\bar{\mathbf{z}}^{(t+1)}) 
& \le \mathbb{E} f(\bar{\mathbf{z}}^{(t)}) 
+ \tilde{\eta}L^2 \underset{T_1}{\underbrace{\mathbb{E}\left \| \bar{\mathbf{z}}^{(t)} - \bar{\mathbf{x}}^{(t)} \right \|^2}} 
+ \underset{T_2}{\underbrace{\left ( \frac{\tilde{\eta}^2L}{2} - \frac{\tilde{\eta}}{4}\right )   \mathbb{E} \left \| \bar{\mathbf{m}}^{(t)} \right \|^2 }}
- \frac{\tilde{\eta}}{2} \mathbb{E} \left \| \nabla f(\bar{\mathbf{x}}^{(t)}) \right \|^2 
+ \frac{\tilde{\eta}}{2} \underset{T_3}{\underbrace{\mathbb{E} \left \| \nabla f(\bar{\mathbf{x}}^{(t)}) - \bar{\mathbf{m}}^{(t)}  \right \|^2}},
\end{aligned}
\end{equation}
where $\tilde{\eta} = \frac{\eta}{1-\beta}$. For term $T_1$, we adopt the same derivation process as (\ref{eq:rearrange_two}), which is also suitable for the circumstance with no multiple local updates, indicating that 
\begin{equation}
\label{eq:gt_smooth_t1}
\left \| \bar{\mathbf{z}}^{(t)} - \bar{\mathbf{x}}^{(t)} \right \|^2 \le \frac{\beta^2\hat{\eta}^2}{\left ( 1-\beta \right )^3 } 
\sum_{l=0}^{t-1} \beta^{t-1-l} \left \| \bar{\mathbf{m}}^{(l)} \right \|^2,   
\end{equation}
where $\hat{\eta} = ((1-\beta)\alpha-1)\eta$. We can further obtain that 
\begin{equation}
\label{eq:gt_smooth_t1_part}
\begin{aligned}
\mathbb{E} \left \| \bar{\mathbf{m}}^{(t)}  \right \|^2 
& \overset{(a)}{=} \mathbb{E} \left \| \lambda \bar{\mathbf{g}}^{(t)} +  \left ( 1-\lambda  \right ) \bar{\mathbf{y}}^{(t)} \right \|^2 \\
& \overset{(b)}{\le} \mathbb{E} \left \| \bar{\mathbf{g}}^{(t)} + \bar{\mathbf{y}}^{(t)} \right \|^2  \\ 
& = \mathbb{E} \left \| \bar{\mathbf{g}}^{(t)} 
\pm \frac{1}{n} \sum_{i=1}^n \nabla f_i(\mathbf{x}_i^{(t)}) 
+ \bar{\mathbf{y}}^{(t)} \right \|^2 \\
& \overset{(c)}{\le} 2 \underset{T_4}{\underbrace{\mathbb{E} \left \| \bar{\mathbf{g}}^{(t)} + \frac{1}{n} \sum_{i=1}^n \nabla f_i(\mathbf{x}_i^{(t)}) \right \|^2}}
+ 2  \underset{T_5}{\underbrace{\mathbb{E} \left \| \bar{\mathbf{y}}^{(t)} - \frac{1}{n} \sum_{i=1}^n \nabla f_i(\mathbf{x}_i^{(t)}) \right \|^2}},
\end{aligned}
\end{equation}
where $(a)$ follows from (\ref{eq:upper_relationship_eq}) by model averaging; we omit the coefficients in $(b)$; $(c)$ follows by applying the \textbf{Fact 4}. For $T_4$,
\begin{equation}
\label{eq:gt_smooth_t4}
\begin{aligned}
T_4 & \overset{(a)}{=} \mathbb{E} \left \| \frac{2}{n} \sum_{i=1}^n \nabla F_i(\mathbf{x}_i^{(t)}, \xi_i^{(t)})
+ \frac{1}{n} \sum_{i=1}^n \nabla f_i(\mathbf{x}_i^{(t)} ) - \frac{1}{n} \sum_{i=1}^n \nabla F_i(\mathbf{x}_i^{(t)}, \xi_i^{(t)}) \right \|^2 \\
& \overset{(b)}{\le} \frac{8}{n^2} \sum_{i=1}^n \mathbb{E} \left \| \nabla F_i(\mathbf{x}_i^{(t)}, \xi_i^{(t)}) \right \|^2 
+ \frac{2}{n^2} \sum_{i=1}^n \mathbb{E} \left \| \nabla f_i(\mathbf{x}_i^{(t)} ) - F_i(\mathbf{x}_i^{(t)}, \xi_i^{(t)}) \right \|^2 \\
& \overset{(c)}{\le}  \frac{8G^2}{n} + \frac{2\sigma^2}{n},
\end{aligned}
\end{equation}
where $(a)$ follows the definition of $\bar{\mathbf{g}}_i$; $(b)$ follows by the Jensen's inequality and the \textbf{Fact 4}; and $(c)$ follows because Assumption~\ref{ass:bound_variance} and \ref{ass:bound_sto_grad}. Estimate $T_5$,
\begin{equation}
\label{eq:gt_smooth_t5}
\begin{aligned}
T_5 & \overset{(a)}{=} \mathbb{E} \left \| \frac{1}{n} \sum_{i=1}^n \mathbf{y}_i^{(t)} - \frac{1}{n} \sum_{i=1}^n \nabla f_i(\mathbf{x}_i^{(t)} )  \right \|^2 \\
& = \mathbb{E} \left \| \frac{1}{n} \sum_{i=1}^n \left ( \mathbf{y}_i^{(t)} 
- \frac{1}{n} \sum_{j=1}^n \nabla f_j(\mathbf{x}_i^{(t)})    \right ) 
+ \frac{1}{n^2} \sum_{i=1}^n \sum_{j=1}^n \left ( \nabla f_j(\mathbf{x}_i^{(t)}) 
- \nabla f(\mathbf{x}_i^{(t)} ) \right ) + \frac{1}{n^2} \sum_{i=1}^n \sum_{j=1}^n
\left ( \nabla f(\mathbf{x}_i^{(t)}) - \nabla f_i(\mathbf{x}_i^{(t)} ) \right )   \right \|^2 \\
& \overset{(b)}{\le} \frac{3}{n}\sum_{i=1}^n \mathbb{E} \left \| \mathbf{y}_i^{(t)} - \frac{1}{n} \sum_{j=1}^n \nabla f_j(\mathbf{x}_i^{(t)}) \right \|^2  
+ \frac{3}{n^4} \sum_{i=1}^n \sum_{j=1}^n \mathbb{E} \left \| \nabla f_j(\mathbf{x}_i^{(t)}) - \nabla f(\mathbf{x}_i^{(t)} ) \right \|^2
+ \frac{3}{n^4} \sum_{i=1}^n \sum_{j=1}^n \mathbb{E} \left \| \nabla f(\mathbf{x}_i^{(t)}) - \nabla f_i(\mathbf{x}_i^{(t)} ) \right \|^2 \\
& \overset{(c)}{\le} 3\epsilon^2 + \frac{6\zeta^2}{n^2}, 
\end{aligned}
\end{equation}
where $(a)$ follows because $\bar{\mathbf{y}}^{(t)} = \frac{1}{n} \sum_{i=1}^n \mathbf{y}_i^{(t)}$; $(b)$ follows by applying the \textbf{Fact 5}; $(c)$ follows by the Proposition~\ref{prop:appro_grad_tracking} and Assumption~\ref{ass:bound_variance}. Substituting (\ref{eq:gt_smooth_t1}), (\ref{eq:gt_smooth_t1_part}), (\ref{eq:gt_smooth_t4}), and (\ref{eq:gt_smooth_t5}) into the $T_1$ in (\ref{eq:begin_smooth_gt}), which yields
\begin{equation}
\label{eq:gt_smooth_t1_bound}
T_1 \le \frac{\beta^2\hat{\eta}^2}{(1-\beta)^3} \sum_{l=0}^{t-1} \beta^{t-1-l} 
\left ( \frac{16G^2}{n} + \frac{4\sigma^2}{n} + \frac{12\zeta^2}{n^2} + 6 \epsilon^2 \right ).  
\end{equation}
We can omit $T_2$ in (\ref{eq:begin_smooth_gt}) on the assumption that $\eta \le \frac{1-\beta}{2L}$ ensures that $\frac{\tilde{\eta}^2L}{2} - \frac{\tilde{\eta}}{4} < 0$. Then we estimate the bound of $T_3$ in (\ref{eq:begin_smooth_gt}),
\begin{equation}
\label{eq:gt_smooth_t3}
\begin{aligned}
\mathbb{E} \left \| \nabla f(\bar{\mathbf{x}}^{(t)}) - \bar{\mathbf{m}}^{(t)} \right \|^2 
& \overset{(a)}{=}  \mathbb{E} \left \| \lambda \left ( \nabla f(\bar{\mathbf{x}}^{(t)}) - \bar{\mathbf{g}}^{(t)}  \right )  
+ \left ( 1 - \lambda \right ) \left ( \nabla f(\bar{\mathbf{x}}^{(t)}) - \bar{\mathbf{y}}^{(t)} \right )   \right \|^2 \\
& \overset{(b)}{\le} 2\lambda^2 \mathbb{E} \left \| \nabla f(\bar{\mathbf{x}}^{(t)}) - \bar{\mathbf{g}}^{(t)} \right \|^2
+ 2 \left ( 1 - \lambda \right )^2 \mathbb{E} \left \| \nabla f(\bar{\mathbf{x}}^{(t)}) - \bar{\mathbf{y}}^{(t)} \right \|^2 \\
& = 2\lambda^2 \mathbb{E}\left \| \frac{1}{n} \sum_{i=1}^n 
\left ( \nabla f(\bar{\mathbf{x}}^{(t)}) \pm \nabla f_i(\mathbf{x}_i^{(t)}) 
- \nabla F_i(\mathbf{x}_i^{(t)}, \xi_i^{(t)}) \right )   \right \|^2 
+ 2 \left ( 1 - \lambda \right )^2 \mathbb{E} \left \| \nabla f(\bar{\mathbf{x}}^{(t)}) - \bar{\mathbf{y}}^{(t)} \right \|^2 \\
& \overset{(c)}{\le} \frac{4\lambda^2}{n^2} \sum_{i=1}^n \mathbb{E} \left \|  \nabla f(\bar{\mathbf{x}}^{(t)}) - \nabla f_i(\mathbf{x}_i^{(t)}) \right \|^2
+  \frac{4\lambda^2\sigma^2}{n} + 2 \left ( 1 - \lambda \right )^2 \mathbb{E} \left \| \nabla f(\bar{\mathbf{x}}^{(t)}) - \bar{\mathbf{y}}^{(t)} \right \|^2 \\
& = \frac{4\lambda^2}{n^2} \sum_{i=1}^n \mathbb{E} 
\left \| \nabla f(\bar{\mathbf{x}}^{(t)}) \pm \nabla f(\mathbf{x}_i^{(t)} )- \nabla f_i(\mathbf{x}_i^{(t)}) \right \|^2
+  \frac{4\lambda^2\sigma^2}{n} + 2 \left ( 1 - \lambda \right )^2 \mathbb{E} \left \| \nabla f(\bar{\mathbf{x}}^{(t)}) - \bar{\mathbf{y}}^{(t)} \right \|^2 \\
& \overset{(d)}{\le} \frac{8L^2\lambda^2}{n^2} \sum_{i=1}^n \left \| \bar{\mathbf{x}}^{(t)} - \mathbf{\mathbf{x}}_i^{(t)} \right \|^2 
+ \frac{4\lambda^2}{n} \left ( \sigma^2 + 2\zeta^2 \right ) 
+ 2 \left ( 1 - \lambda \right )^2 \underset{T_6}{\underbrace{\mathbb{E} \left \| \nabla f(\bar{\mathbf{x}}^{(t)}) - \bar{\mathbf{y}}^{(t)} \right \|^2}},
\end{aligned}
\end{equation}
where $(a)$ is follows by the definition of $\bar{\mathbf{m}}^{(t)}$; $(b)$ follows because of the \textbf{Fact 4}; $(c)$ and $(d)$ follow by applying the \textbf{Fact 4}, Jensen's inequality and Assumption~\ref{ass:bound_variance}. For $T_6$, we can estimate as follows

\begin{equation}
\label{eq:gt_smooth_t6}
\begin{aligned}
T_6 & = \mathbb{E} \left \| \frac{1}{n} \sum_{i=1}^n \nabla f(\bar{\mathbf{x}}^{(t)})
\pm \frac{1}{n} \sum_{i=1}^n\nabla f(\mathbf{x}_i^{(t)}) - \bar{\mathbf{y}}^{(t)}  \right \|^2 \\
& \overset{(a)}{\le} \frac{2}{n^2} \sum_{i=1}^n\mathbb{E} \left \| \nabla f(\bar{\mathbf{x}}^{(t)}) 
- \nabla f(\mathbf{x}_i^{(t)}) \right \|^2 
+ 2 \mathbb{E} \left \| \frac{1}{n} \sum_{i=1}^n \nabla f(\mathbf{x}_i^{(t)}) 
- \frac{1}{n} \sum_{i=1}^n\mathbf{y}_i^{(t)}  \right \|^2 \\
& \overset{(b)}{\le} \frac{2L^2}{n^2} \sum_{i=1}^n \mathbb{E} \left \| \bar{\mathbf{x}}^{(t)} - \mathbf{x}_i^{(t)} \right \|^2
+ 2 \mathbb{E} \left \| \frac{1}{n^2} \sum_{j=1}^n\sum_{i=1}^n \nabla f(\mathbf{x}_i^{(t)})
\pm \frac{1}{n^2}\sum_{j=1}^n \sum_{i=1}^n \nabla f_j(\mathbf{x}_i^{(t)} ) -\frac{1}{n}\sum_{i=1}^n \mathbf{y}_i^{(t)}   \right \|^2 \\
& \overset{(c)}{\le} \frac{2L^2}{n^2} \sum_{i=1}^n \mathbb{E} \left \| \bar{\mathbf{x}}^{(t)} - \mathbf{x}_i^{(t)} \right \|^2
+ \frac{2}{n^2}\sum_{i=1}^n \left ( \frac{2}{n^2} \sum_{j=1}^n 
\mathbb{E} \left \| \nabla f(\mathbf{x}_i^{(t)}) - \nabla f_j(\mathbf{x}_i^{(t)} ) \right \|^2 
+ 2\mathbb{E} \left \| \frac{1}{n}\sum_{j=1}^n \nabla f_j(\mathbf{x}_i^{(t)}) - \mathbf{y}_i^{(t)}  \right \|^2    \right ) \\
& \overset{(d)}{\le} \frac{2L^2}{n^2} \sum_{i=1}^n \mathbb{E} \left \| \bar{\mathbf{x}}^{(t)} - \mathbf{x}_i^{(t)} \right \|^2+ \frac{4\zeta^2}{n^2} + 4\epsilon^2,
\end{aligned}
\end{equation}
where $(a)$ follows from the \textbf{Fact 4}; $(b)$ follows because the Assumption~\ref{ass:l_smooth} for the first term on the right hand of the inequality; $(c)$ follows by applying the \textbf{Fact} 4 with Jensen's inequality; $(d)$ follows because the Assumption~\ref{ass:bound_variance} and Proposition~\ref{prop:appro_grad_tracking}. Plugging (\ref{eq:gt_smooth_t6}) into (\ref{eq:gt_smooth_t3}) yields
\begin{equation}
\label{eq:gt_smooth_t3_bound}
T_3 \le \left ( \frac{8L^2 \lambda^2}{n^2} + \frac{4L^2 \left ( 1 - \lambda \right )^2 }{n^2}    \right ) 
\mathbb{E} \left \| \bar{\mathbf{X}}^{(t)} - \mathbf{X}^{(t)} \right \|_F^2
+ \frac{4\lambda^2\sigma^2 + 16\zeta^2}{n} + 8(1-\lambda)^2\epsilon^2.  
\end{equation}
Plugging (\ref{eq:gt_smooth_t1_bound}) and (\ref{eq:gt_smooth_t3_bound}) into (\ref{eq:begin_smooth_gt}), which yields
\begin{equation}
\label{eq:gt_close}
\begin{aligned}
\mathbb{E} f(\bar{\mathbf{z}}^{(t+1)}) & \le \mathbb{E} f(\bar{\mathbf{z}}^{(t)}) 
+ \frac{\beta^2\hat{\eta}^2\tilde{\eta}L^2}{(1-\beta)^3} \sum_{l=0}^{t-1} \beta^{t-1-l}
\left ( \frac{16G^2}{n} + \frac{4\sigma^2}{n} + \frac{12\zeta^2}{n^2} + 6 \epsilon^2 \right )
- \frac{\tilde{\eta}}{2} \mathbb{E} \left \| \nabla f(\bar{\mathbf{x}}^{(t)} ) \right \|^2 + \frac{6L^2\tilde{\eta}}{n^2} \mathbb{E} \left \| \bar{\mathbf{X}}^{(t)} - \mathbf{X}^{(t)} \right \|_F^2 \\
& \quad + \frac{2\tilde{\eta}\lambda^2\sigma^2 + 8\tilde{\eta}\zeta^2}{n} + 4\tilde{\eta}(1-\lambda)^2 \epsilon^2 \\
& \le \mathbb{E} f(\bar{\mathbf{z}}^{(t)}) 
+ \frac{\beta^2\hat{\eta}^2\tilde{\eta}L^2}{(1-\beta)^4} \left ( \frac{16G^2}{n} + \frac{4\sigma^2}{n} + \frac{12\zeta^2}{n} + 6 \epsilon^2 \right ) 
- \frac{\tilde{\eta}}{2} \mathbb{E} \left \| \nabla f(\bar{\mathbf{x}}^{(t)} ) \right \|^2
+ \frac{2\tilde{\eta}\lambda^2\sigma^2 + 8\tilde{\eta}\zeta^2}{n} + 4\tilde{\eta}(1-\lambda)^2 \epsilon^2 \\
& \quad + \frac{6L^2\tilde{\eta}}{n^2} \left ( \mathbb{E} \left \| \bar{\mathbf{X}}^{(t)} - \mathbf{X}^{(t)} \right \|_F^2 + \mathbb{E} \left \| \bar{\mathbf{Y}}^{(t)} - \mathbf{Y}^{(t)} \right \|_F^2\right ),
\end{aligned}
\end{equation}
the last inequality holds because $\sum_{l=0}^{t-1}\beta^{t-1-l} \le \frac{1}{1-\beta}$, and we add the non-negative term $\mathbb{E}  \left \| \bar{\mathbf{Y}}^{(t)} - \mathbf{Y}^{(t)} \right \|_F^2$.

\begin{proposition}
\label{prop:recursive_abs}
Let $\left \{ V_t \right \}_{t \ge 0}$ be a non-negative sequence and $C_2 \ge 0$ be some constant such that $\forall t \ge 1, V_{t+1} \le Q_2V_t + Q_2V_{t-1} + C_2$, where $Q_2 \in (0, 1)$. Then the following inequality holds if $T \ge 1$, 
\begin{equation}
\begin{aligned}
V_{t+1} & \le Q_2V_t +Q_2V_{t-1}+C_2 \\
& \le Q_2V_t +V_{t-1}+C_2 \\
& \le Q_2^2V_{t-1} + (Q_2V_{t-2} + V_{t-1}) + (Q_2+1)C_2 \\
& \dots \\
& \le Q_2^tV_1 + \sum_{l=0}^{t-1} Q_2^{t-1-l} V_l + C_2 \sum_{l=0}^{t-1}Q_2^l .
\end{aligned}
\end{equation}
Summing $t$ over from $0$ to $T-1$,
\begin{equation}
\begin{aligned}
\sum_{t=0}^{T-1} V_t & = V_0 + V_1+ \sum_{t=2}^{T-1}V_t \\
& \le V_0 + V_1+ V_1 \sum_{t=2}^{T-1} Q_2^{t-1} + \sum_{t=2}^{T-1} \sum_{l=0}^{t-2} Q_2^{t-2-l} V_l + C_2\sum_{t=2}^{T-1} \sum_{l=0}^{t-2} Q_2^l \\
& \le V_0 + V_1 + V_1 \sum_{t=0}^{\infty} Q_2^t + \sum_{t=0}^{T-1} \sum_{l=0}^{\infty} Q_2^l V_l + C_2\sum_{t=0}^{T-1} \sum_{l=0}^{\infty}  Q_2^l \\
& \le V_0 + \frac{(2-Q_2)V_1}{1-Q_2} + \frac{TV_{\max}}{1-Q_2} + \frac{C_2T}{1-Q_2},
\end{aligned}
\end{equation}
where $V_{\max} = \max_{0 \le t \le T-1} \left \{ V_t\right \}$.
\end{proposition}
Summing (\ref{eq:gt_close}) over $t$ from $0$ to $T-1$, then dividing both sides by $2/\tilde{\eta}$ and rearranging terms. Finally, applying Lemma~\ref{lem:dist_xy_gt} to Proposition~\ref{prop:recursive_abs}. Concretely, we consider $C_2 = \frac{\beta_0^2\left ( 1-\frac{\rho}{2}  \right )}{L^2} (\sigma^2+\zeta^2) (192\lambda^2L^2 + \rho )$, $V_t \triangleq \mathbb{E} \left \| \mathbf{X}^{(t)} - \bar{\mathbf{X}}^{(t)} \right \|_F^2 + \mathbb{E} \left \| \mathbf{Y}^{(t)} - \bar{\mathbf{Y}}^{(t)} \right \|_F^2$, $Q_2 \triangleq \min \left \{ 4\beta_0^2\left ( 1-\frac{\rho}{4}  \right ), 8(1-\rho) (1+\alpha\beta)^2  \right \}$ ensuring that $\left\{\begin{matrix}
\begin{aligned} 
& 4\beta_0^2\left ( 1-\frac{\rho}{4}  \right ) < 1 \\
& 8(1-\rho) (1+\alpha\beta)^2 < 1 
\end{aligned}
\end{matrix}\right.$. Since $\mathbf{x}_i^{(0)} = \bar{\mathbf{x}}^{(0)} = \mathbf{0}$, and 
$\mathbb{E} \left \| \mathbf{Y}^{(0)} - \bar{\mathbf{Y}}^{(0)}  \right \|_F^2
= \mathbb{E} \left \| \mathbf{G}^{(0)} - \bar{\mathbf{G}}^{(0)}  \right \|_F^2
\le \frac{6L^2}{n} \mathbb{E} \left \| \mathbf{X}^{(0)} - \bar{\mathbf{X}}^{(0)} \right \| 
+ 6(\zeta^2 + \sigma^2) = 6(\zeta^2 + \sigma^2)$, thus $V_0 \le Q_3 = 6(\zeta^2 + \sigma^2)$. Furthermore, we assume that $\forall i, \mathbf{x}_i^{(-1)} = \mathbf{y}_i^{(-1)} = \mathbf{0}$, which is extended at an initial stage (\ie, $t=0$) for Lemma~\ref{lem:dist_x_gt}, (\ref{lem:dist_y_gt}), and (\ref{lem:dist_xy_gt}), thus $V_1 \le Q_4 = \beta_0^2 \left ( 1-\frac{\rho}{2}  \right ) (\sigma^2+\zeta^2) \left ( 48+ \frac{\rho}{L^2} + 192\lambda^2   \right )$. The proof is completed. 

\section{Experimental Setup}
\label{sec:appendix:ex_setup}

\subsection{Dataset and Model Description}
MNIST is a 10-class handwritten digits image classification dataset with $70,000$ $28$ $\times$ $28$ examples, $60,000$ of which are training datasets, the remaining $10,000$ are test datasets. 
Its extended version, EMNIST consists of images of digits and upper and lower case English characters, which includes $62$ total classes. CIFAR10 is labeled subsets of the 80 million images dataset, sharing the same $60,000$ input images with $10$ unique labels. For NLP, AG NEWS is a 4-class classification dataset on categorized news articles, containing $120,000$ training samples and $7,600$ testing samples. An overall description is given in Table~\ref{tab:models_datasets}.

\begin{table}[!htbp]
\centering
\caption{Datasets and Models
}
\resizebox{1.\textwidth}{!}{
\begin{tabular}{ccccccc}
\toprule
Dataset & Task & Training samples & Testing samples & Classes & Model  \\ \midrule
MNIST~\cite{lecun1998gradient}  & Handwritten character recognition (CV) & $60,000$ & $10,000$ & $10$ &LeNet described in Table~\ref{tbl:lenet_architecture} \\
EMNIST~\cite{cohen2017emnist}  & Handwritten character recognition (CV) & $731,668$ & $82,587$ & $62$ & CNN described in Table~\ref{tbl:cnn_architecture} \\
CIFAR10~\cite{krizhevsky2009learning} 
& Image classification (CV) & $50,000$ & $10,000$ & $10$ & LeNet described in Table~\ref{tbl:lenet_architecture} \\
AG NEWS \cite{zhang2015character} 
& Text classification (NLP) & $120,000$ & $7,600$ & $4$ & RNN described in Table~\ref{tbl:lstm_architecture} \\ \bottomrule
\end{tabular}
}
\label{tab:models_datasets}
\end{table}

\begin{table}[!htbp]
\centering
\caption{LeNet model on MNIST and CIFAR10.}
\begin{tabular}{cccc}
\toprule
Layer     & Output Shape & Hyperparameters & Activation \\
\midrule
\textsc{Conv2d}    & $(28, 28, 6)$ & kernel size $= 5$ & ReLU       \\
\textsc{MaxPool2d} & $(14, 14, 6)$  & pool size $= 2$   &            \\
\textsc{Con2d}     & $(10, 10, 16)$ & kernel size $= 5$ & ReLU       \\
\textsc{MaxPool2d} & $(5, 5, 16)$   & pool size $= 2$   &            \\
\textsc{Flatten}   & $400$          &                 &            \\
\textsc{Dense}     & $120$          &                 &            \\
\textsc{Dense}     & $84$           &                 &            \\
\textsc{Dropout}   & $84$           & $p = 0.5$         & \\
\textsc{Dense}     & $10$           &                 &  \\ 
\bottomrule
\end{tabular}
\label{tbl:lenet_architecture}
\end{table}

\begin{table}[!htbp]
\centering
\caption{CNN model on EMNIST.}
\begin{tabular}{cccc}
\toprule
Layer & Output Shape & Hyperparameters & Activation \\
\midrule
\textsc{Conv2d} & $(26, 26, 32)$ & kernel size $= 3$, strides $= 1$ \\ 
\textsc{Conv2d} & $(24, 24, 64)$ & kernel size $= 3$, strides $= 1$ & ReLU \\
\textsc{MaxPool2d} & $(12, 12, 64) $ & pool size $= 2$ & \\
\textsc{Dropout} & $(12, 12, 64)$ & $p = 0.25$ & \\
\textsc{Flatten} & $9,216$ & & \\
\textsc{Dense} & $128$ & & \\
\textsc{Dropout} & $128$ & $p = 0.5$ & \\
\textsc{Dense} & $62$ & & softmax \\
\bottomrule
\end{tabular}
\label{tbl:cnn_architecture}
\end{table}

\begin{table}[!htbp]
\centering
\caption{RNN model on AG NEWS.}
\begin{tabular}{cc}
\toprule
Layer        & Hyperparameters                      \\
\midrule
\textsc{EmbeddingBag} & embeddings $= 95,812$, dimension $= 64$   \\
\textsc{Dense}        & in\_features $= 64$, out\_features $= 4$ \\
\textsc{Dropout}      & $p = 0.5$ \\
\bottomrule
\end{tabular}
\label{tbl:lstm_architecture}
\end{table}

\section{Addtional Evaluations}
\label{appendix:addi_eval}
\subsection{Model Generalization}
Figure~\ref{fig:mnist_lenet}, \ref{fig:emnist_cnn}, \ref{fig:cifar10_lenet} and \ref{fig:agnews_lstm} present the experimental results on the model training process for different benchmarks. We note that the superiority of our proposed methods is better reflected in the convergence acceleration. For example, both \textt{\algone} and \textt{\algtwo} require about $50$ epochs to reach convergence for LeNet over MNIST, which reduces the number of training epochs by $40\%$. Switching to large datasets, \ie, CIFAR10 and AG NEWS, our proposed algorithms converge faster than other baselines with respect to the training epochs.

\begin{figure}[!htbp]
\centering
\subfigure[non-IID $= 0.1$]{
\label{fig:mnist_lenet_noniid0.1}
\includegraphics[width=0.32\linewidth]{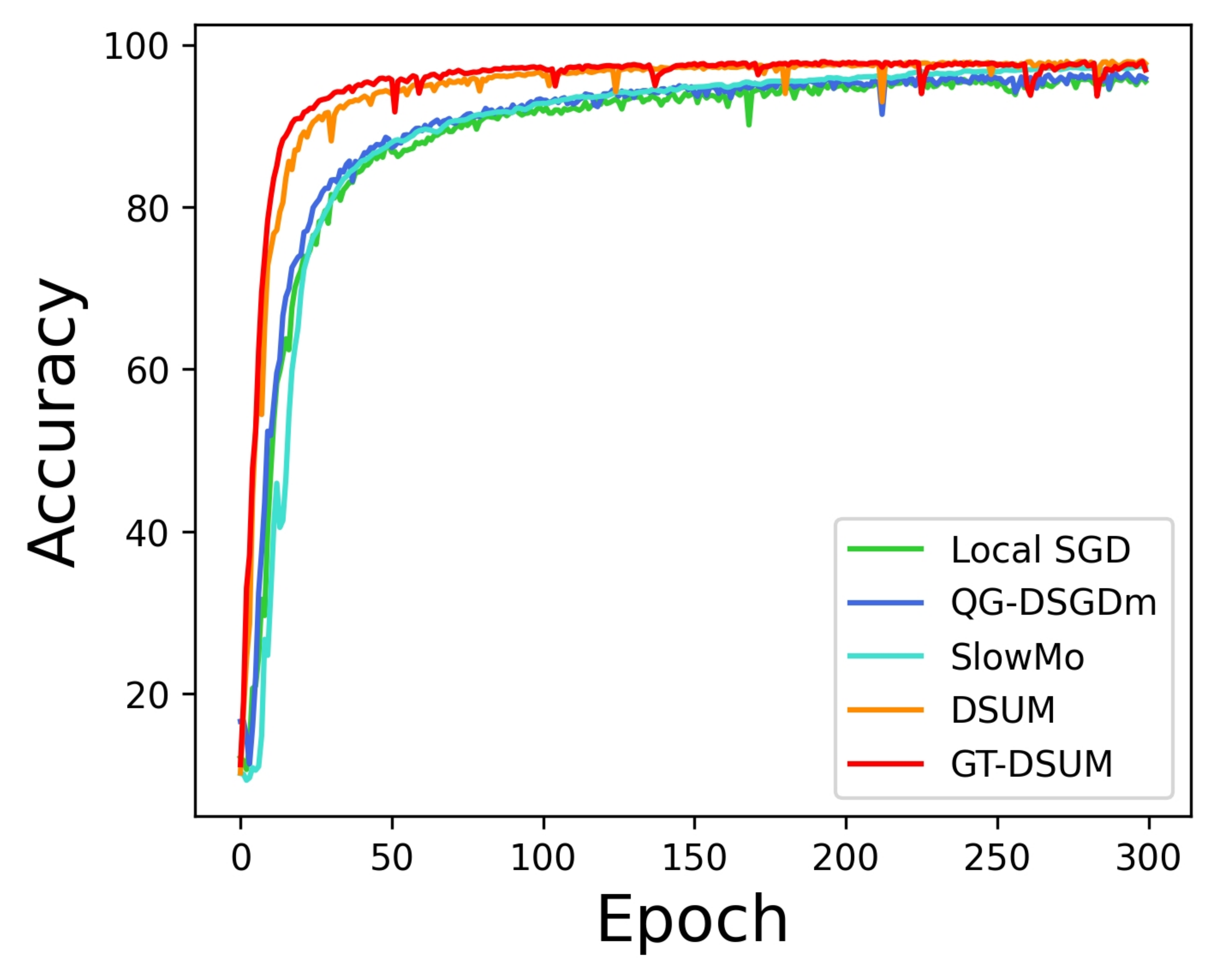}}
\subfigure[non-IID $= 1$]{
\label{fig:mnist_lenet_noniid1}
\includegraphics[width=0.32\linewidth]{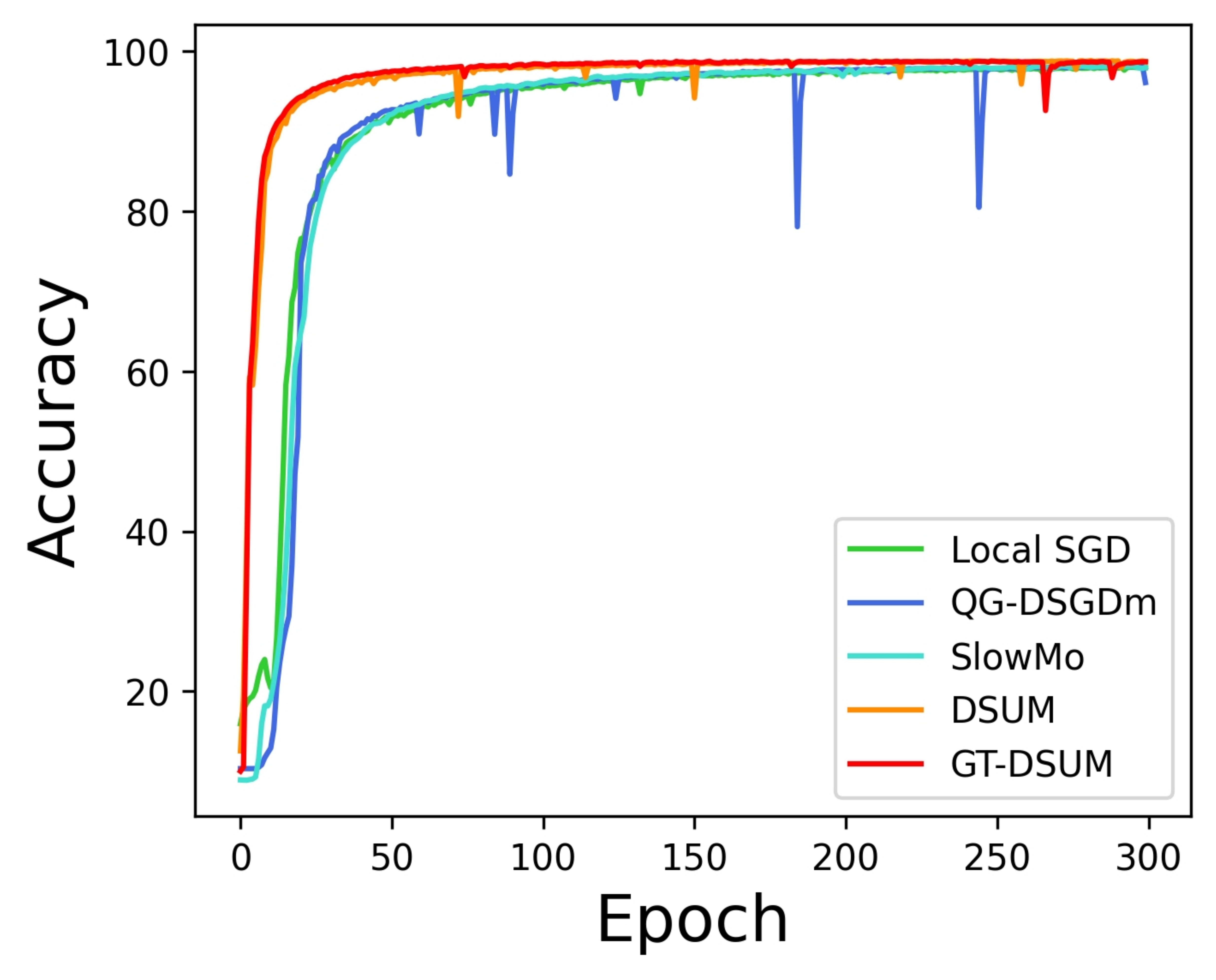}}
\subfigure[non-IID $= 10$]{
\label{fig:mnist_lenet_noniid10}
\includegraphics[width=0.32\linewidth]{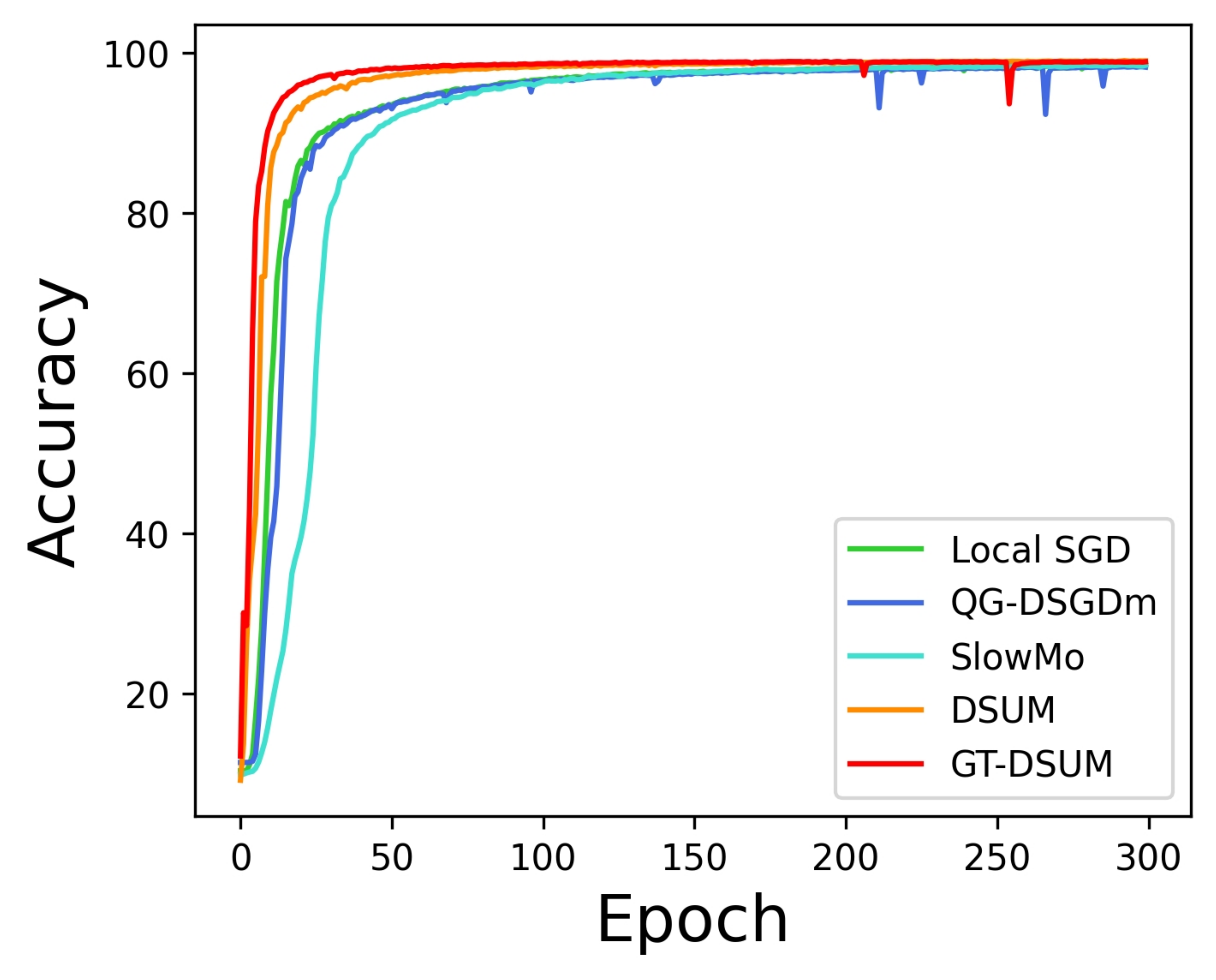}}
\label{fig:mnist_lenet_noniid10}
\caption{Testing accuracy for various tasks training on \textbf{LeNet} over \textbf{MNIST}.}
\label{fig:mnist_lenet}
\end{figure}

\begin{figure}[!htbp]
\centering
\subfigure[non-IID $= 0.1$]{
\includegraphics[width=0.32\linewidth]{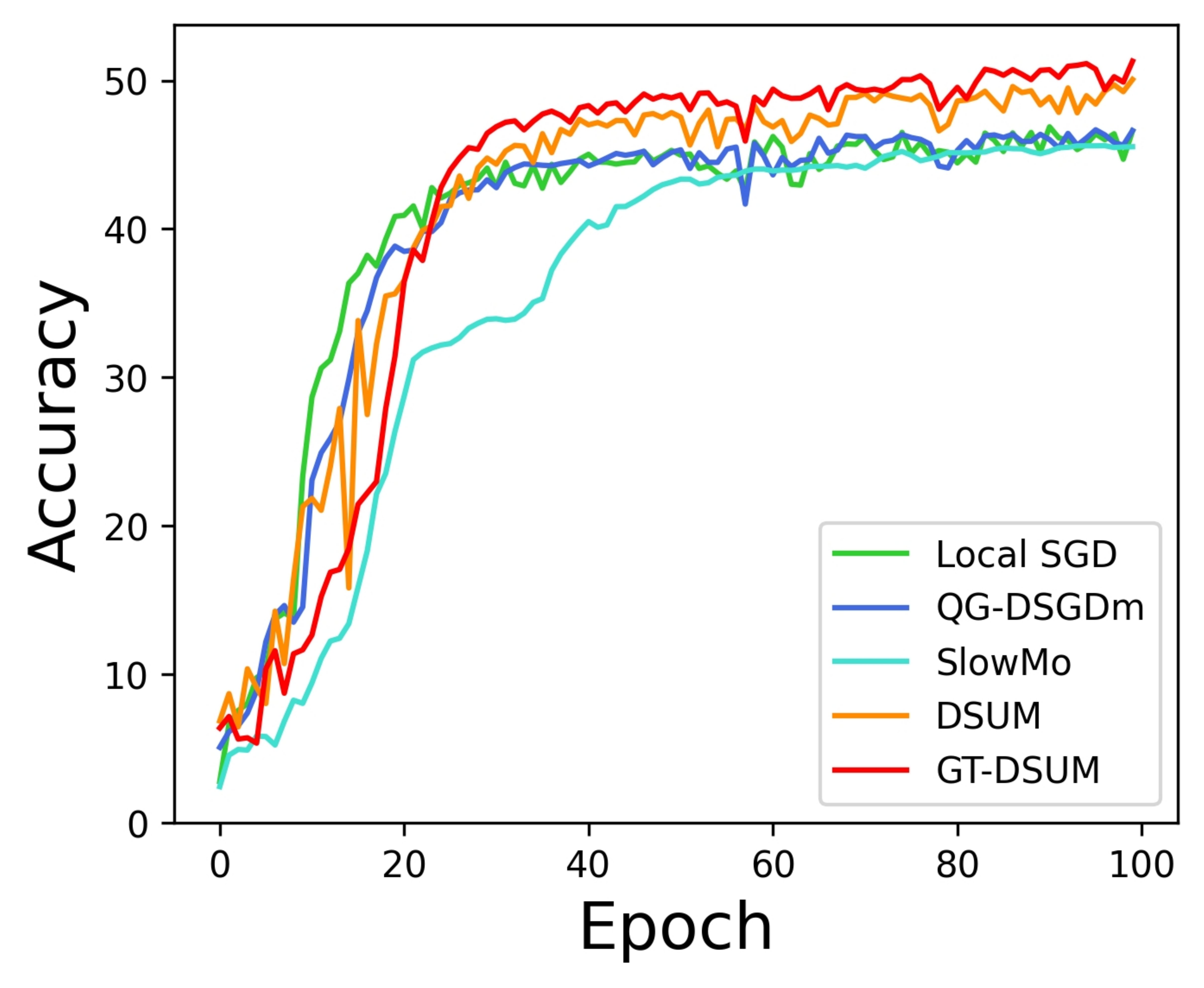}
\label{fig:emnist_cnn_noniid0.1}}
\subfigure[non-IID $= 1$]{
\includegraphics[width=0.32\linewidth]{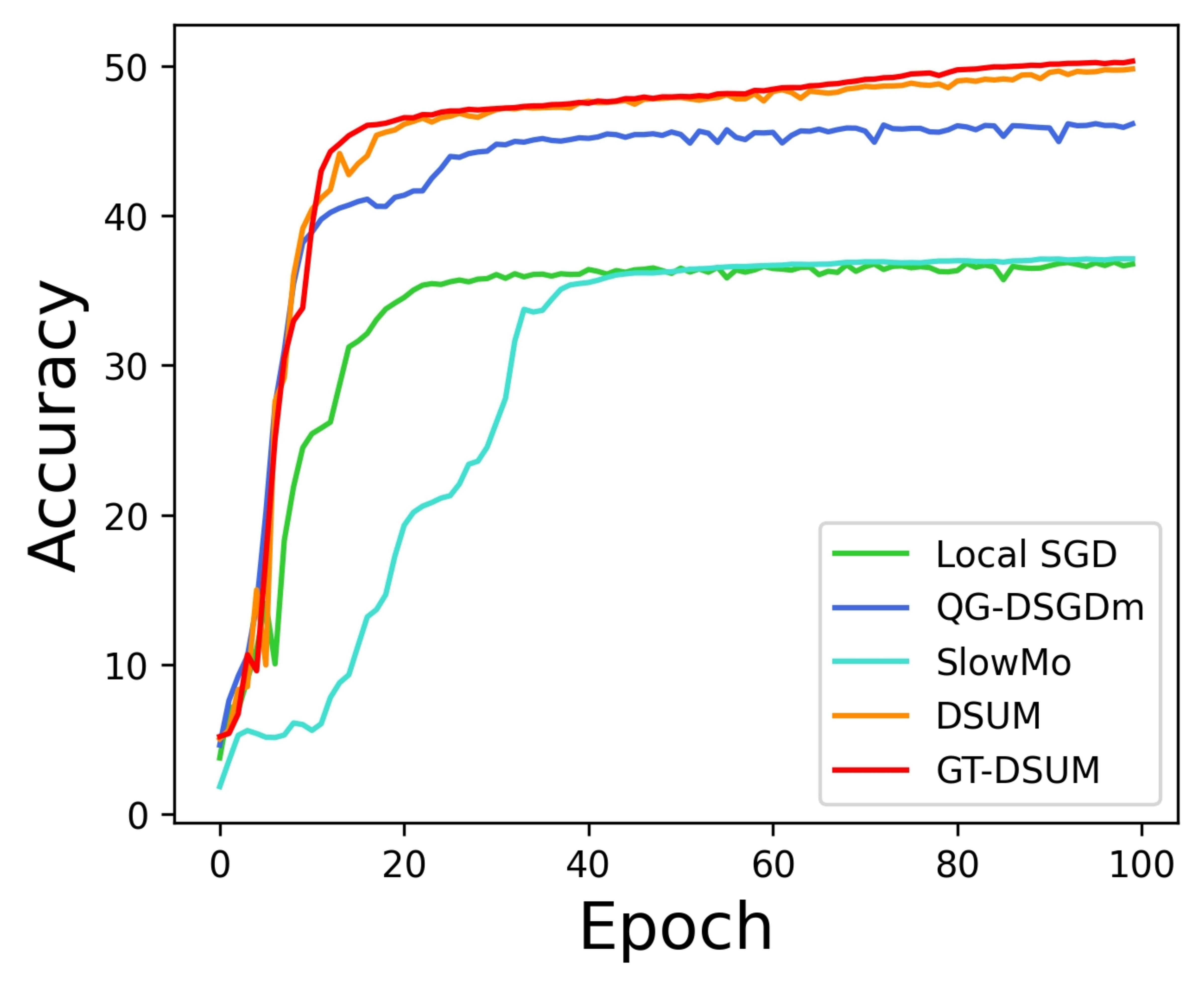}
\label{fig:emnist_cnn_noniid1}}
\subfigure[non-IID $= 10$]{
\includegraphics[width=0.32\linewidth]{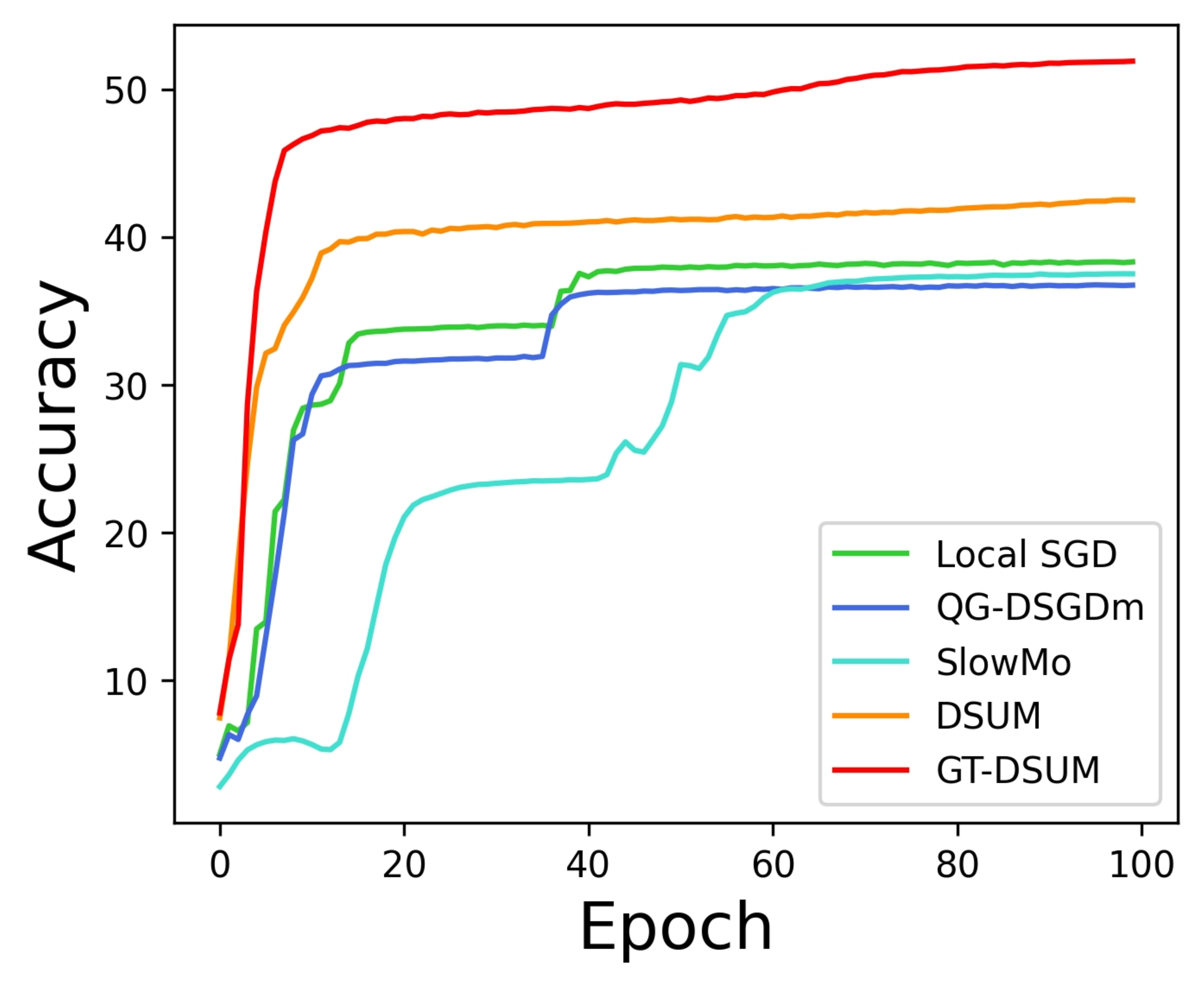}
\label{fig:emnist_cnn_noniid10}}
\caption{Testing accuracy for various tasks training on \textbf{CNN} over \textbf{EMNIST}.}
\label{fig:emnist_cnn}
\end{figure}

\begin{figure}[!htbp]
\centering
\subfigure[non-IID $= 0.1$]{
\includegraphics[width=0.32\linewidth]{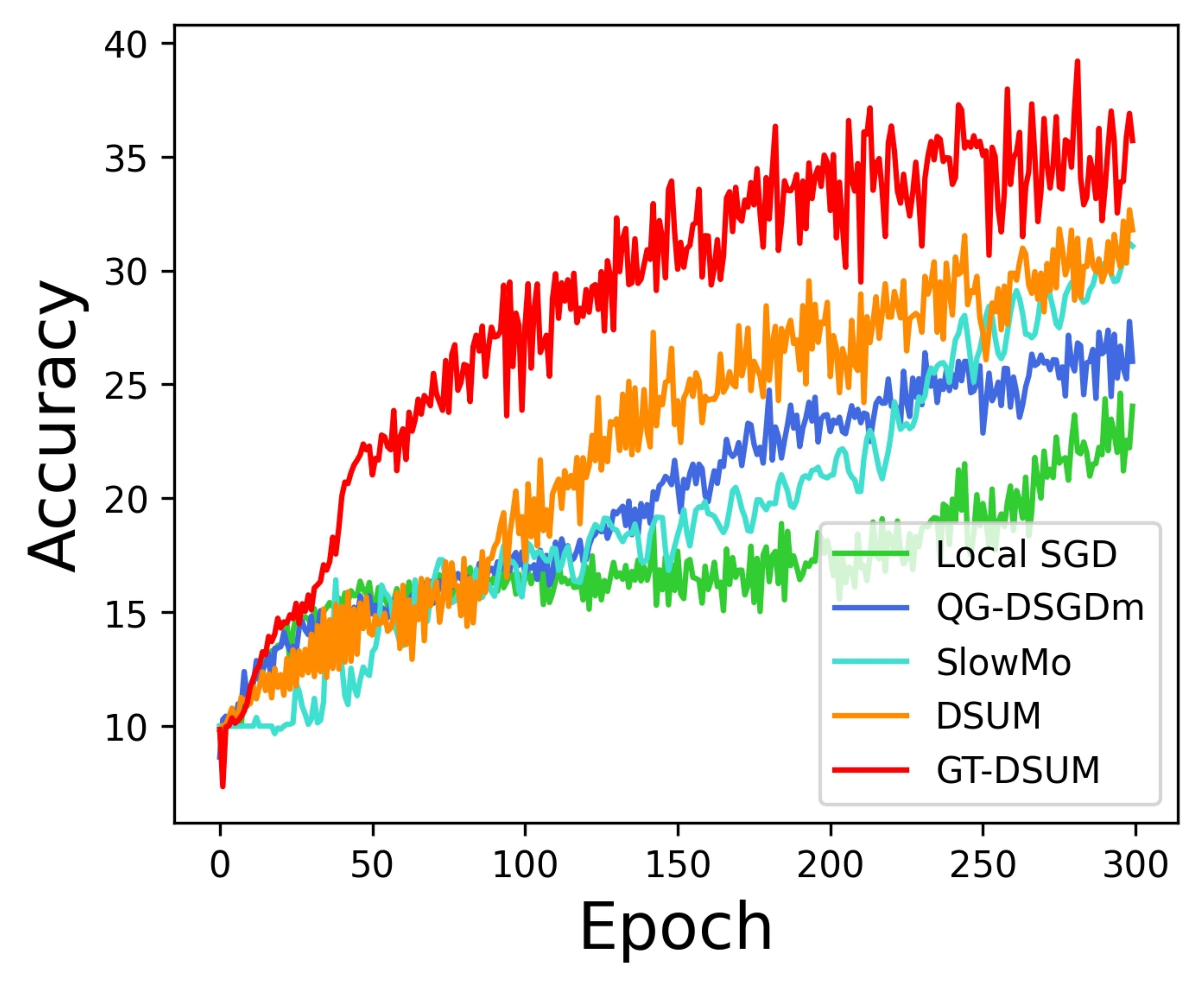}
\label{fig:cifar10_lenet_noniid0.1}}
\subfigure[non-IID $= 1$]{
\includegraphics[width=0.32\linewidth]{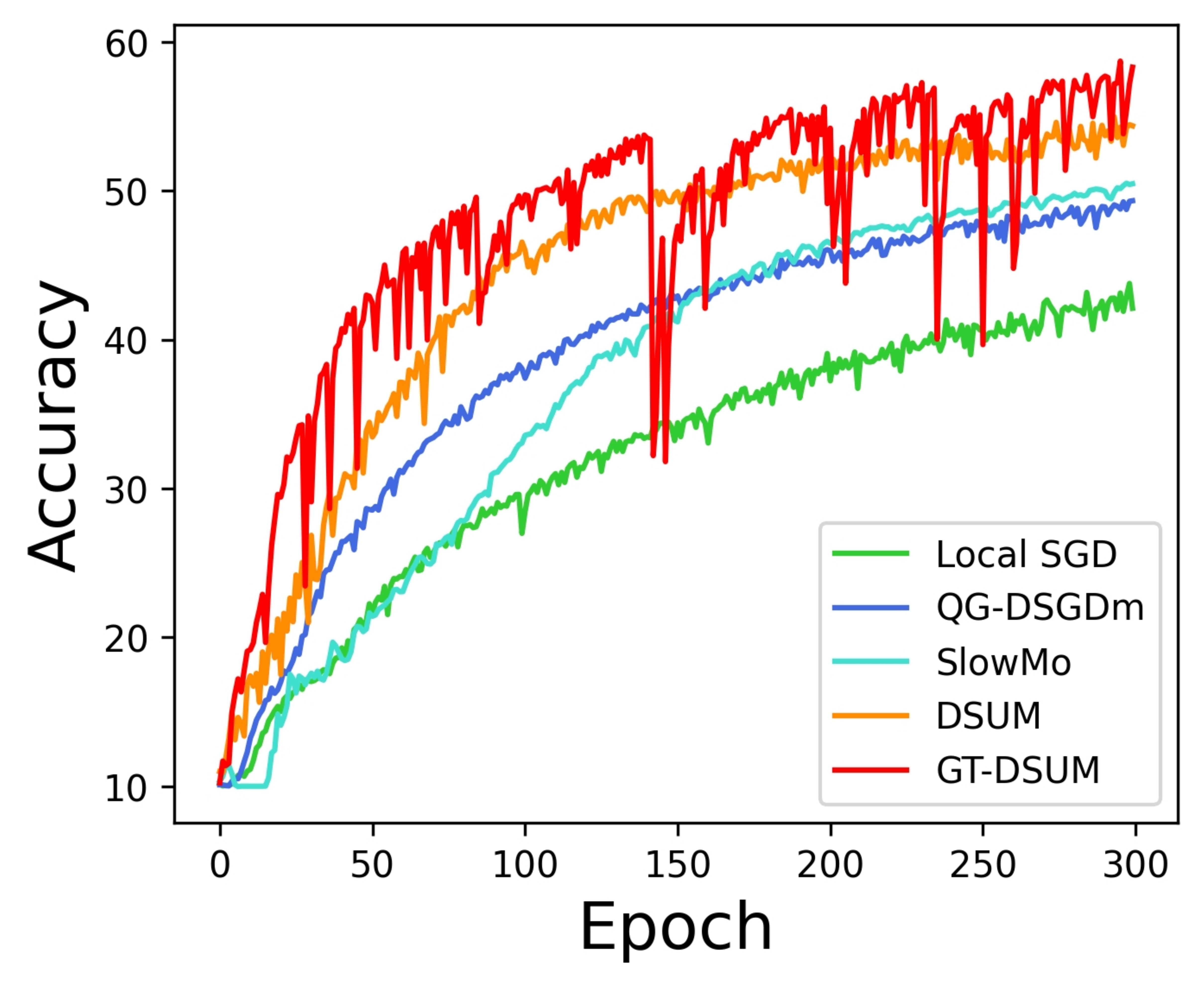}
\label{fig:cifar10_lenet_noniid1}}
\subfigure[non-IID $= 10$]{
\includegraphics[width=0.32\linewidth]{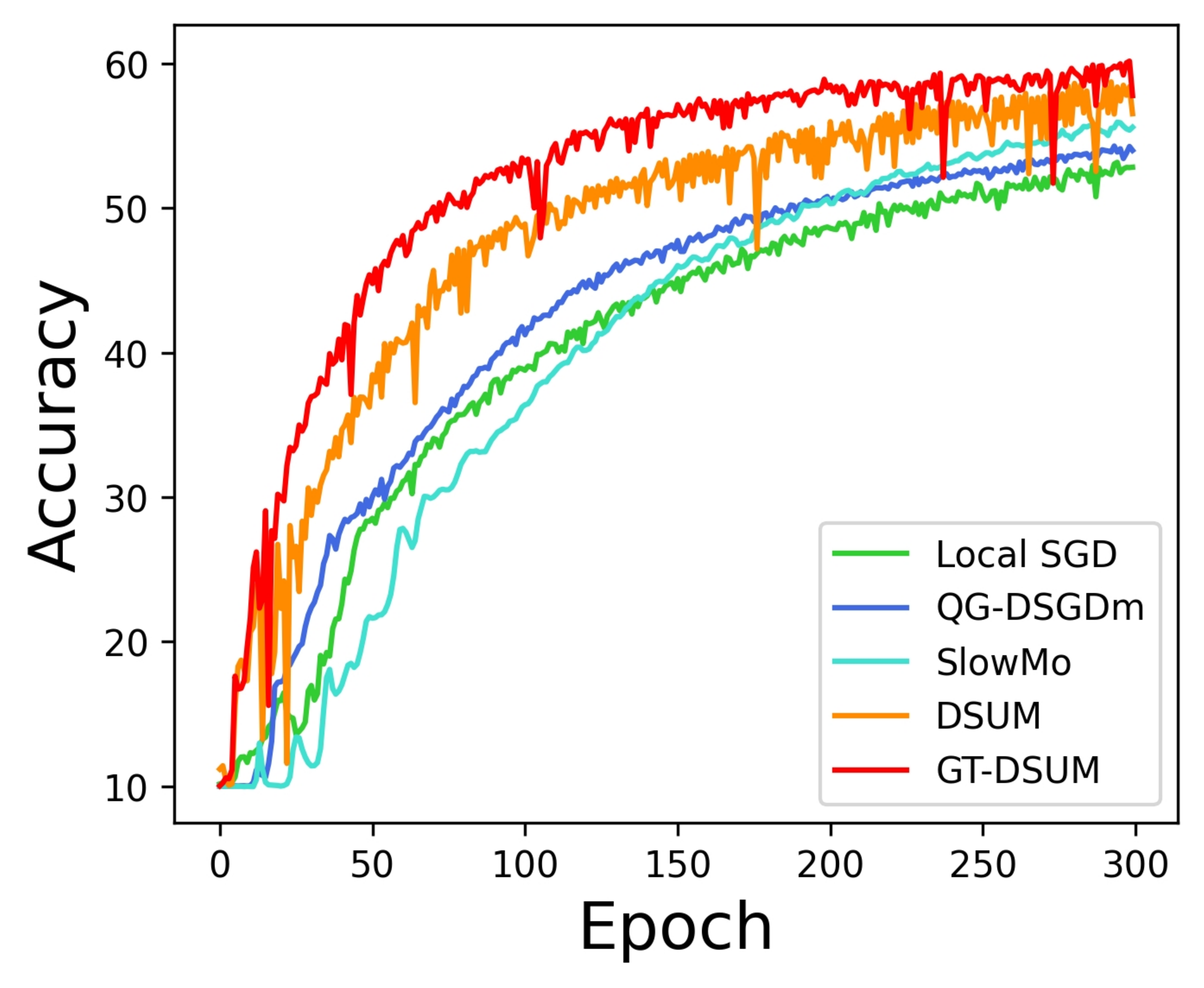}
\label{fig:cifar10_lenet_noniid10}}
\caption{Testing accuracy for various tasks training on \textbf{LeNet} over \textbf{CIFAR10}.}
\label{fig:cifar10_lenet}
\end{figure}

\begin{figure}[!htbp]
\centering
\subfigure[non-IID $= 0.1$]{
\includegraphics[width=0.32\linewidth]{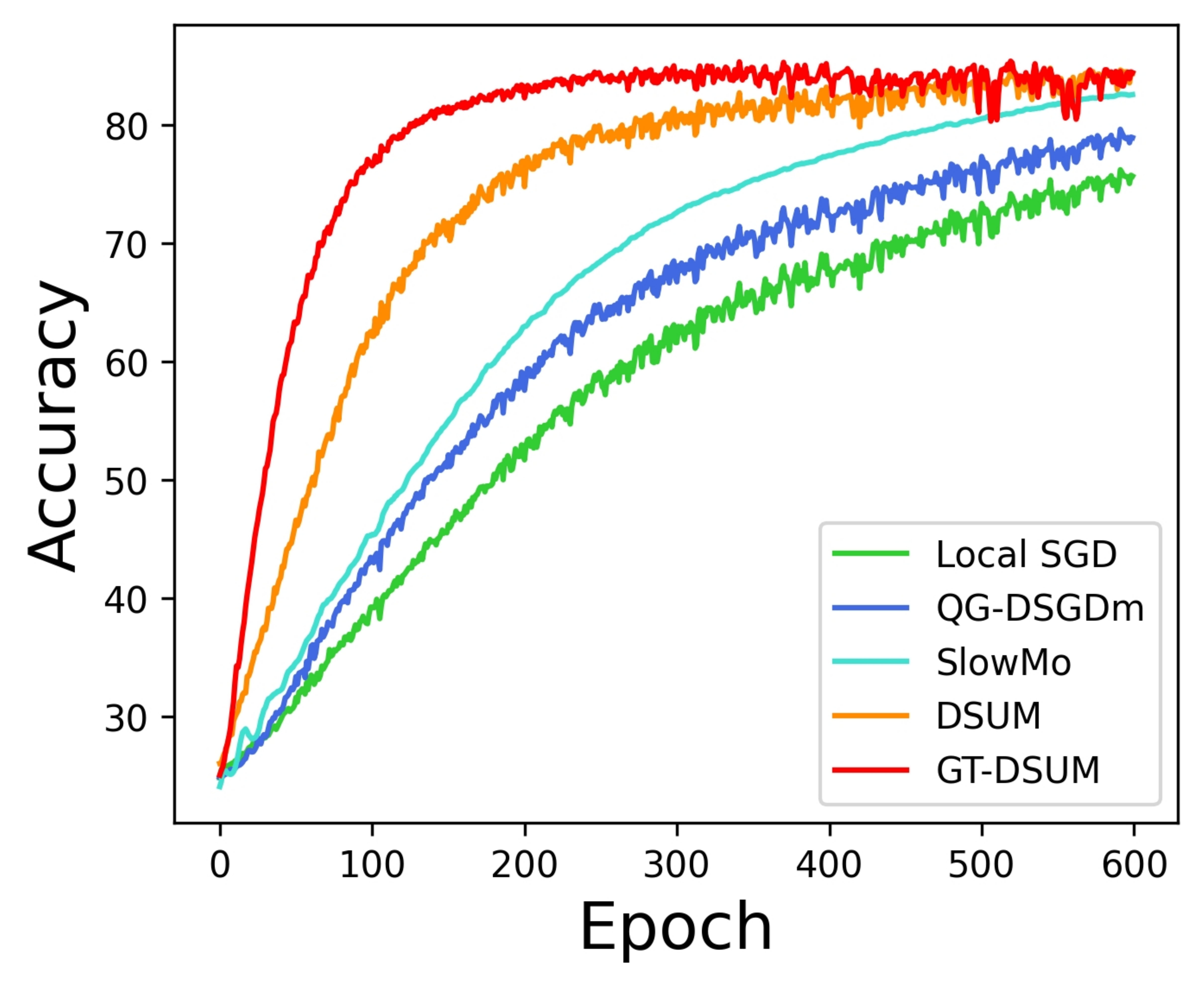}
\label{fig:agnews_lstm_noniid0.1}}
\subfigure[non-IID $= 1$]{
\includegraphics[width=0.32\linewidth]{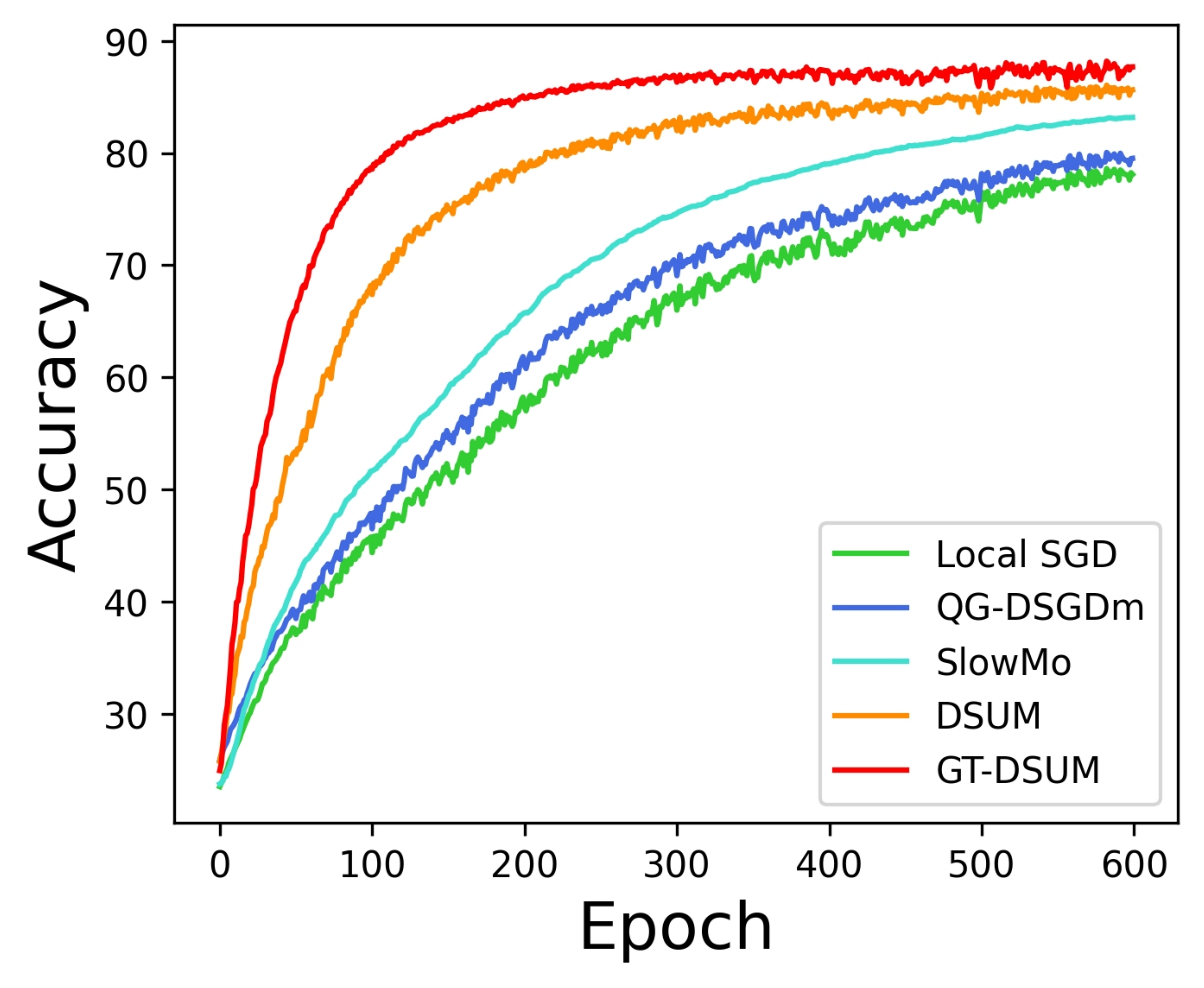}
\label{fig:agnews_lstm_noniid1}}
\subfigure[non-IID $= 10$]{
\includegraphics[width=0.32\linewidth]{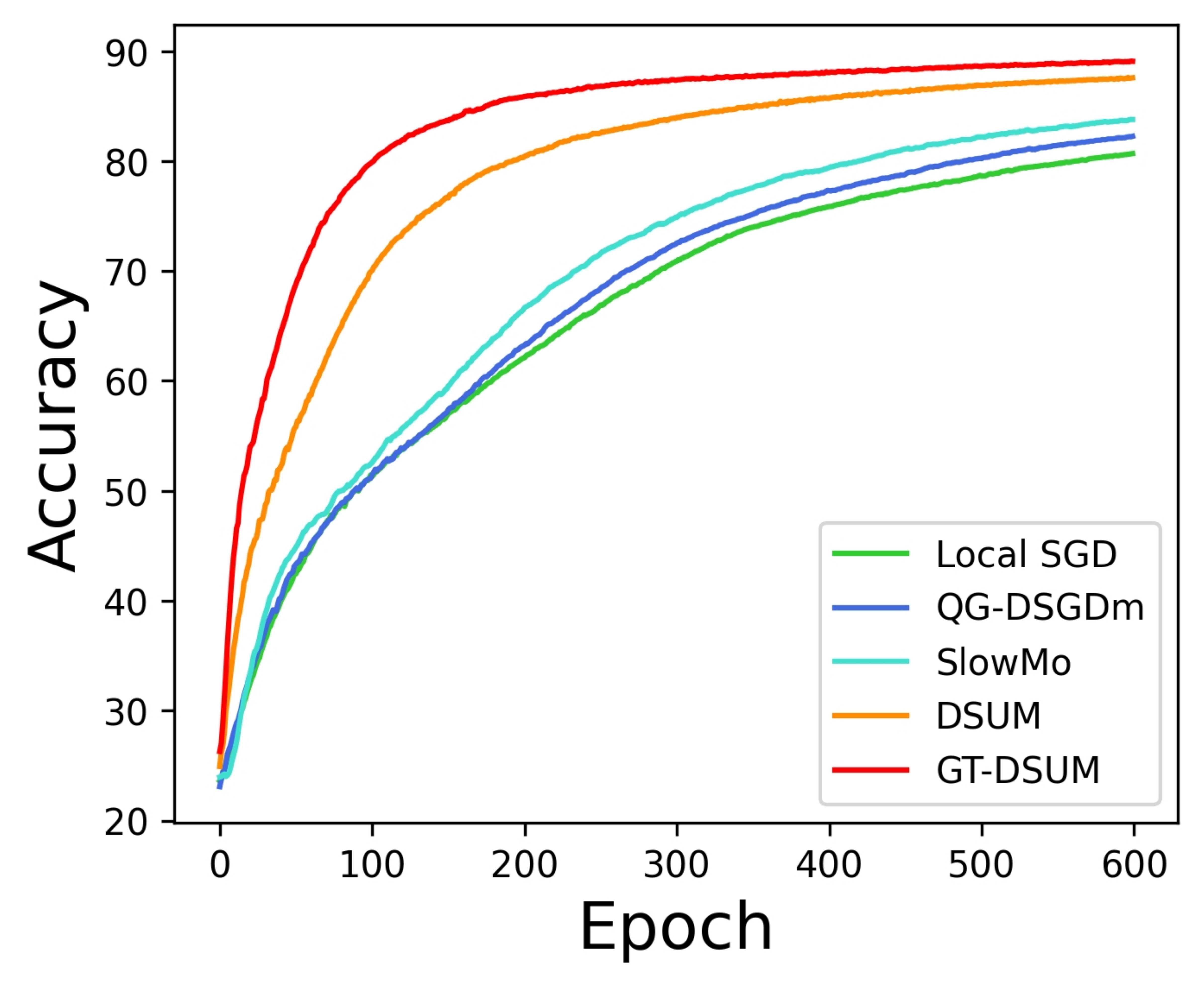}
\label{fig:agnews_lstm_noniid10}}
\caption{Testing accuracy for various tasks training on \textbf{RNN} over \textbf{AG NEWS}.}
\label{fig:agnews_lstm}
\end{figure}


\end{document}